\documentclass[11pt]{article}

\usepackage{amsthm,amsmath,bbm,amssymb,natbib,graphicx,booktabs,array,fullpage,url,mathtools,wrapfig,lipsum,mathrsfs,dsfont,titling,epstopdf,bm,relsize,caption,xcolor,algorithm,enumitem,multirow,soul,makecell}

\usepackage[noend]{algpseudocode}
\usepackage[colorlinks=true,
linkcolor=blue,
urlcolor=blue,
citecolor=blue]{hyperref}

\usepackage{xr} 

\newtheorem{conj}{Conjecture}
\newtheorem{thm}[conj]{Theorem}
\newtheorem{cor}[conj]{Corollary}
\newtheorem{prop}[conj]{Proposition}
\newtheorem{lemma}[conj]{Lemma}
\newtheorem{ass}{Assumption}

\newtheorem{definition}{Definition}

\newtheorem{innercustomgeneric}{\customgenericname}
\providecommand{\customgenericname}{}
\newcommand{\newcustomtheorem}[2]{%
	\newenvironment{#1}[1]
	{%
		\renewcommand\customgenericname{#2}%
		\renewcommand\theinnercustomgeneric{##1}%
		\innercustomgeneric
	}
	{\endinnercustomgeneric}
}

\newcustomtheorem{customAss}{Assumption}

\theoremstyle{definition}\newtheorem{remark}{Remark}

\def\PP{\mathbb{P}}
\def\EE{\mathbb{E}}
\def\RR{\mathbb{R}}

\def\E{\mathcal{E}}
\def\I{\mathcal{I}}

\def\S{\mathcal{S}}
\def\N{\mathcal{N}}
\def\R{\mathcal{R}}

\def\lp{\lambda_p}
\def\sxy{\Sigma_{XY}}
\def\Y{\bm{Y}}
\def\X{\bm{X}}

\def\Z{\bm{Z}}
\def\W{\bm{W}}
\def\Eps{\bm{\eps}}
\def\U{\bm{U}}
\def\V{\bm{V}}
\def\D{\bm{D}}
\def\G{\bm{G}}

\def\Gamma{\Sigma_W}
\def\Cov{\text{Cov}}   
\def\wh{\widehat}
\def\eps{\varepsilon}
\def\e{\gamma_{\eps}}
\def\w{\gamma_w}
\def\z{\gamma_z}

\def\T{\top}
\def\op{{\rm op}}

\def\tr{{\rm tr}}

\def\i{\infty}
\def\bI{\bm{I}}

\def\wt{\widetilde}
\let\emptyset\varnothing

\def\rs{\!\!\!}
\def\C{\Sigma_Z}
\def\whC{\wh \Sigma_Z}

\def\og{\|\Gamma\|_{\op}}

\def\1{\bm{1}}
\def\pn{(p\vee n)}

\def\rank{\textrm{rank}}

\def\lk{\lambda_K}
\def\sza{A\Sigma_Z A^\top}

\def\sx{\Sigma_X}
\def\sw{\Sigma_W}
\def\sz{\Sigma_Z}

\def\sep{\sigma}
\def\errw{\delta_W}

\newcommand{\sbt}{\,\begin{picture}(-1,1)(-0.5,-2)\circle*{2.3}\end{picture}\ }

\DeclareMathOperator*{\argmin}{arg\,min}

\begin{document}
	\title{
	Prediction under latent factor regression: adaptive PCR, interpolating predictors  and beyond}
	\author{Xin Bing\thanks{Department of Statistics and Data Science, Cornell University, Ithaca, NY. E-mail: \texttt{xb43@cornell.edu}.}~~~~~Florentina Bunea\thanks{Department of Statistics and Data Science, Cornell University, Ithaca, NY. E-mail: \texttt{fb238@cornell.edu}.}~~~~~ Seth Strimas-Mackey\thanks{Department of Statistics and Data Science, Cornell University, Ithaca, NY. E-mail: \texttt{scs324@cornell.edu}.}~~~~~Marten Wegkamp\thanks{Departments of Mathematics, and   of Statistics and Data Science, Cornell University, Ithaca, NY. E-mail: \texttt{mhw73@cornell.edu}.} }
	\date{}
	\maketitle
	\vspace{-0.5in}
	
	\begin{abstract}
	    This work is devoted to the finite sample prediction risk analysis of a class of linear predictors of a response $Y\in \RR$ from a high-dimensional random vector $X\in \RR^p$ when $(X,Y)$ follows a latent factor regression model generated by a unobservable latent vector $Z$ of dimension less than $p$. Our primary contribution is in establishing finite sample risk bounds for prediction with the ubiquitous Principal Component Regression (PCR) method, under the factor regression model, with the number of principal components adaptively selected from the data -- a form of theoretical guarantee that is surprisingly lacking from the PCR literature. To accomplish this, we prove a master theorem that establishes a risk bound for a large class of predictors, including the PCR predictor as a special case. This approach has the benefit of providing a unified framework for the analysis of a wide range of linear prediction methods, under the factor regression setting. In particular, we use our main theorem to recover known risk bounds for the minimum-norm interpolating predictor, which has received renewed attention in the past two years, and a prediction method tailored to a subclass of factor regression models with identifiable parameters. This model-tailored method can be interpreted as prediction via clusters with latent centers. 
	    
	    To address the problem of selecting among a set of candidate predictors, we analyze a simple model selection procedure based on data-splitting, providing an oracle inequality under the factor model to prove that the performance of the selected predictor is close to the optimal candidate. We conclude with a detailed simulation study to support and complement our theoretical results.
	\end{abstract}
	

\section{Introduction}\label{sec_pred}

This work is devoted to the derivation and analysis of finite sample prediction   risk bounds for a class of linear predictors of a random response $Y \in \RR$ from a high-dimensional, and possibly highly correlated random vector $X \in \RR^p$, when the vector $(X, Y)$ follows a latent factor regression model, generated by a latent vector of dimension lower than $p$. We assume that  there exist  a random, unobservable, latent vector $Z \in \RR^K$, 
a deterministic matrix $A \in \RR^{p \times K}$, and a coefficient vector $\beta \in \RR^K$ such that 
\begin{equation}\label{main_model}
    \begin{aligned}
	Y &= Z^\T \beta + \eps,\\ 
	X &=  AZ + W,
	\end{aligned}
\end{equation}
with some unknown $ K <  p$. 
The random noise $\eps\in \RR$ and $W\in \RR^p$ have mean zero and second moments $\sigma^2 \coloneqq \EE[\eps^2]$ and $\Gamma \coloneqq \EE[WW^\T]$, respectively. The random variable $\eps$ and random vectors $W$ and $Z$ are mutually independent. Throughout the paper,  both $\C := \EE[ZZ^\T]$ and $A$ have rank equal to $K$.

Independently  of this model formulation, but based on the  belief that $Y$ depends chiefly on a lower-dimensional approximation of $X$,  prediction of $Y$ via principal components (PCR) is perhaps the most utilized  scheme, with a history dating back many decades \citep{Kendall,Hotelling}. Given the data $\X = (X_1, \ldots, X_n)^\T$ and $\Y = (Y_1,\ldots, Y_n)$ consisting of $n$ independent copies of $(X,Y) \in \RR^p \times \RR$, PCR-$k$  predicts $Y_*\in\RR$ after observing a  new data point $X_*\in\RR^p$  by
	\begin{eqnarray}\label{def_U}
	\wh Y^*_{\U_{k}}&=& 
	X_{*}^\T \U_{k}\left[\U_{k}^{\T}\X^\T\X\U_{k}\, \right]^+\U_{k}^{\T}\X^\T \Y \nonumber\\
	&=& 	X_{*}^\T \U_{k}\left[ \X \U_{k} \right]^{+} \Y,
	\end{eqnarray} 
where 	${\U_{k}}$ is the $p \times k$ matrix of the top eigenvectors of the sample covariance matrix $\X^\T \X/n$, relative to the largest $k$ eigenvalues, where $k$ is ideally determined in a data-dependent fashion and $M^{+}$ denotes the Moore-Penrose inverse of a matrix $M$.
	
	

Model (\ref{main_model}) provides a natural context for the theoretical analysis of PCR-$k$ prediction. It is perhaps surprising that its theoretical study so far is limited to asymptotic analyses of the out-of-sample prediction risk for PCR-$K$ as $p, n \rightarrow \infty$
\citep{SW2002_JASA,Bai-Ng-CI}, and finite sample / asymptotic risk bounds on the in-sample prediction accuracy of PCR-$K$ \citep{Bai-factor-model-03,Bair_JASA,fan2013large,Kelly-2015,fan2017} in identifiable factor models with known and fixed $K$.  

To the best of our knowledge, finite sample prediction risk bounds for $	\wh Y^*_{\U_{k}}$, corresponding to data-dependent choices of $k$,  are lacking in the literature, and their study  under factor models of unknown $K$, possibly varying with $n$, provides motivation for this work.

To  obtain risk bounds for PCR, we prove a master theorem, Theorem \ref{thm_pred}, that establishes a finite sample prediction risk bound for linear predictors of the general form 
\begin{equation}\label{def_pred_B_intro}
\wh Y^*_{\wh B}=X_{*}^\T \wh B\left(\wh B^\T\X^\T\X\wh B\, \right)^+\wh B^\T\X^\T \Y, 
\end{equation} 
where $\wh B\in \RR^{p\times q}$ is an appropriate matrix that may be deterministic or depend on the data $\X$, with dimension $q$ allowed to be random.

This approach has the benefit of not only covering the special case of PCR, corresponding to choice $\wh B = \U_k$, but of offering a unifying analysis of other prediction schemes of the form (\ref{def_pred_B_intro}). One important example 
corresponds to $\wh B = \bI_p$, which leads to another model agnostic predictor, the generalized least squares estimator (also known as the  minimum norm interpolating predictor), which has enjoyed revamped popularity in the last two years \citep{montanari2019generalization, bunea2020interpolation, muthukumar2019harmless,muthukumar2020classification,hastie2019surprises, feldman2019does, Belkin15849, belkin2019models,belkin2018overfitting,belkin2018understand, 
 belkin2018does, bartlett2019,liang2019}. Using the full data matrix $\X$ for prediction -- instead of just the first $k$ 
 principal components as in PCR -- leads to additional bias compared to PCR prediction. However, in the high-dimensional regime $p\gg n$, this bias can become small and choosing $\wh B = \bI_p$ can become a viable alternative to PCR that requires no tuning parameters.

In addition to these two model-agnostic prediction methods, Theorem \ref{thm_pred} can be used to analyze predictors directly tailored to model (\ref{main_model}), which are shown  formally to be of type  (\ref{def_pred_B_intro}) in Section \ref{sec_pred_er}. We give a  particular expression of $\wh B$, as well as the corresponding  prediction analysis,  under further modelling restrictions that render parameters $K$, $A$ and $\beta$ identifiable.  The model specifications given in Section \ref{sec_pred_er}  allow us to view $A$ as a cluster membership matrix, making it possible  to address a third, understudied,  class of examples pertaining to prediction from low-dimensional feature representation, that of prediction of $Y$ via latent cluster centers, for features that exhibit an overlapping clustering structure corresponding to $A$. 

\subsection{Our contributions and organization of the paper} 

Our main theoretical goal is to offer sufficient conditions on $\wh B$ under which the prediction risk $\R(\wh B)$, defined as  \begin{equation}\label{def_risk}
    \R(\wh B)\coloneqq\EE[(Y_* - \wh Y_{\wh B}^*)^2],
\end{equation}
provably approaches an optimal risk benchmark, as $n$ and $p$ grow, with particular attention given to the case $p > n$. The expectation in (\ref{def_risk}) is taken with respect to the new data point $(X_*, Y_*)$.
Our main applications will be to the finite sample risk bounds of the three classes of predictors discussed in the previous section. \\

     {\bf 1. General finite sample risk bounds for linear predictors, under factor regression models.} To meet our main theoretical  goal, in Section \ref{sec_pred_fr}, we state the risk benchmark in Lemma \ref{lem_pop_X} and prove a master theorem, and  our main theoretical result, Theorem \ref{thm_pred}.  It provides a finite sample bound on  $\R(\wh B)$, for generic $\wh B$, when $(X,Y)$ follow a factor regression model (\ref{main_model}) that is fully introduced in Section \ref{sec_prelim}. 
    
    The risk bound  (\ref{bd_R_B_star}) of Theorem \ref{thm_pred} depends on random quantities $\wh r = \rank(\X P_{\wh B})$, $\wh \eta = n^{-1}\sigma_{\wh r}^2(\X P_{\wh B})$, and $\wh\psi = n^{-1}\sigma_1^2(\X P_{\wh B}^{\perp})$, where we use $\sigma_k(M)$ to denote the $k$th largest singular value for any matrix $M$.
To interpret these, note that $\wh Y_{\wh B}^* = \wh Y_{P_{\wh B}}^*$  (see Lemma \ref{lem_invariant} in Appendix \ref{sec:main proofs} for the proof), where $P_{\wh B}$ is the projection onto the range of $\wh B$. We then see that $\wh r$ is the rank of the projected data matrix $\X P_{\wh B}$ used for constructing $\wh Y_{\wh B}^*$, $\wh \eta$ captures the size of the signal that is retained in $\X$ after projection onto the range of $\wh B$, and $\wh \psi$ captures the bias introduced by using only the component of $\X$ in the range of $\wh B$ for prediction. 

The utility of Theorem \ref{thm_pred}, as a general result,  is in reducing the difficult task of bounding $\R(\wh B)$ to the relatively easier one of controlling $\wh r$, $\wh \eta$, and $\wh \psi$ corresponding to any matrix $\wh B$ of interest. \\



{\bf 2. Finite sample risk bounds for PCR-$\wh s$, with data-adaptive $\wh s$ principal components.}
We use Theorem \ref{thm_pred} to analyze the prediction risk of PCR-$\wh s$ under the factor regression model, for two choices of the number of principal components $\wh s$. We first consider the \textit{theoretical elbow method}, which selects $\wh s$ corresponding to the smallest eigenvalue of $\X^\T \X/n$ above the noise level of order $\delta_W \coloneqq c(\|\sw\|_{\op} + \tr(\sw)/n)$, for an absolute constant $c > 0$.  Corollary \ref{cor_PCR_delta_w} provides the rate
    	\begin{equation}
        \R(\U_{\wh s})-\sigma^2
		~ \lesssim  ~ (K+\log n) {\sigma^2\over n}+ \delta_W \beta^\T(A^\T A)^{-1}\beta.
     \end{equation}
 The first term on the right hand side is the standard variance term of linear regression in $K$ dimensions. The second term is a bias term that arises from the fact that we predict using $X$ instead of $Z$; we show that such a term is unavoidable in Lemma \ref{lem_pop_X} of Section \ref{sec_benchmark} below.

We termed this procedure {\it theoretical} as $\delta_W$ depends on unknown quantities of the data distribution. We address this by introducing a novel method in Section \ref{pcr_newmethod}, which we show in Corollary \ref{cor_PCR_s_tilde} achieves the same rate as PCR with the theoretical elbow method, under mild additional assumptions, and is fully data-adaptive, only requiring the choice of one scale-free tuning parameter. \\

{\bf 3. Minimum-norm interpolating predictors.} In Section \ref{sec_pred_GLS} we use the master theorem to recover risk bounds for the Generalized Least Squares predictor (GLS), independently derived in \cite{bunea2020interpolation}. This predictor is also known as the minimum-norm interpolating predictor when $p > n$.\\

{\bf 4. Prediction under identifiable factor regression models: Essential regression.}   In Section \ref{sec_pred_er} we consider a particular identifiable factor regression model, the Essential Regression model introduced in \cite{ER}. The identifiability assumptions employ a type of errors-in-variables parametrization of $A$, described in Section  \ref{sec_pred_er}, that allows the components of $Z$ to be respectively matched with distinct groups of components of $X$. The latter property, combined with a further sparsity assumption on $A$, can be used to define overlapping clusters of $X$ with latent centers $Z_k$, $ 1 \leq k \leq K$ \citep{LOVE}.  Thus, of independent interest, prediction in Essential Regression is prediction via latent cluster centers. We show  formally in Section  \ref{sec_pred_er} that this model specification leads to predictors  of type (\ref{def_pred_B_intro}), with $\wh B = \wh A$, for an appropriate estimator $\wh A$ of $A$. We provide a finite sample prediction bound in Theorem \ref{thm_pred_A}, as an application of  Theorem \ref{thm_pred}. We use the derived bound  as an example that illustrates the possible benefits of sparsity in the predictor's coefficient matrix, as our matrix $\wh A$ is allowed to be sparse. \\

{\bf 5. Data-splitting under factor regression models.}
To allow for model selection among the diverse set of prediction methods in this setting, we offer a simple model selection approach in Section \ref{sec_data_split} based on data splitting. We provide an oracle inequality showing that the selected predictor performs nearly as well as the predictor with the lowest risk. \\

A preview of the results in Sections \ref{sec_pred_pcr} -- \ref{sec_pred_GLS_ER} is given in  Table \ref{table} below, which focuses on the high-dimensional regime where $p > Cn$ for a large enough constant $C>0$, and is stated under the simplifying assumptions $\lambda_K(A^\T A) \gtrsim p / K$ and $r_e(\Gamma) \asymp p$, where $r_e(\Gamma)  \coloneqq \tr(\sw)/\|\sw\|_{\op}$ is the reduced effective rank of $\sw$, the covariance matrix of $W$ from model (\ref{main_model}). The bound for Essential Regression contains the quantity $\|A_J\|_0$, which is the sparsity level of the sub-matrix $A_J$ of $A$ corresponding to \textit{non-pure} variables in the Essential Regression model, namely the variables associated with more than one latent factor $Z_k$ (see Section \ref{sec_pred_er} for a formal definition). The full set of conditions under which these bounds hold, as well as their general form is given, respectively, in each of the sections in which these methods are analyzed. For now we mention that we do not make specific distributional assumption on the data, but we do derive the rates given in the table below under the assumption that  $\eps  \in \RR$, $Z \in \RR^K$, and $W \in \RR^p$ are sub-Gaussian. 

The term $\sigma^2 K/n$  is common to all three risk bounds, and shows that all methods have the potential to adapt to the unknown, latent, $K$-dimensional model structure, provided that the remaining terms are small. Relative to PCR and ER, the GLS method has an additional variance term $\sigma^2n/p$, that arises from the fact that GLS uses the full data matrix $\X$, as opposed to a lower-dimensional projection of it; this demonstrates that GLS  has competitive performance only when $p\gg n$. The relative performance of the PCR and ER methods depends on the sparsity of the matrix $A_J$: when $\|A_J\|_0=o(p)$, for example, the ER method can outperform PCR.

We further discuss the relative merits of these predictors, in terms of their respective risk bounds and assumptions under which they hold, in Section \ref{sec_compare}.

\begin{table}[h!] 
\centering
    \begin{tabular}{|c|c|c|}
    \hline
   \rule{0pt}{0.5cm}  Prediction Method & $\wh B$ &Excess risk bound\\
    \hline
    \rule{0pt}{0.65cm} PCR & $ \U_K$ &${K\over n}\sigma^2 + {K\over p}\og \|\beta\|^2+ {K\over n}\og \|\beta\|^2$\rule[-0.5cm]{0pt}{0.6cm}\\
            \hline
    \rule{0pt}{0.6cm} GLS & $\bI_p$ &  ${K\over n}\sigma^2 +
			{n \over p}\sigma^2+  {K \over n} \og \|\beta\|^2$\rule[-0.5cm]{0pt}{0.6cm}\\
	\hline
    	\rule{0pt}{0.6cm} ER & $\wh A$ & ${K\over n}\sigma^2+ {K\over p}\og \|\beta\|^2+ {\|A_J\|_0 \over p} \times {K \over n}\og \|\beta\|^2$\rule[-0.5cm]{0pt}{0.6cm}\\
    \hline
    \end{tabular}
    \caption{Summary of bounds on $\R(\wh B) - \sigma^2$, where $\R(\wh B)$ is defined in (\ref{def_risk}), for Principal Component Regression (PCR), Generalized Least Squares (GLS), and Essential Regression (ER), stated under simplifying assumptions described in Section \ref{sec_compare}. The second column gives the choice of $\wh B$ corresponding to each method. All three bounds follow from
    the main Theorem \ref{thm_pred}. }\label{table}
\end{table}

We conclude the paper with Section \ref{sec_sims}, in which we present a detailed simulation study of the PCR-type predictors, the minimum-norm interpolating predictor, and predictors under Essential Regression, as well as the proposed model selection method. All proofs are deferred to the Appendix.

\paragraph{Notation} We use the following notation throughout the paper. For any vector $v$, we use $\|v\|_q$ denote its $\ell_q$ norm for $0\le q\le \i$. We write $\|v\| = \|v\|_2$. For an arbitrary real-valued matrix $M\in \RR^{r\times q}$, we use $M^+$ to denote the Moore-Penrose inverse of $M$, and $\sigma_1(M)\ge \sigma_2(M)\ge \cdots \ge \sigma_{\min(r,q)}(M)$ to denote the singular values of $M$ in non-increasing order. We define the operator norm $\|M\|_{\op}=\sigma_1(M)$, the Frobenius norm $\|M\|_F^2 = \sum_{i,j}M_{ij}^2$, the elementwise sup-norm $\|M\|_\i = \max_{i,j}|M_{ij}|$ and the cardinality of non-zero entries $\|M\|_0 = \sum_{i,j}1_{M_{ij}\ne 0}$. For a symmetric positive semi-definite matrix $Q\in \RR^{p\times p}$, we use $\lambda_1(Q)\ge \lambda_2(Q)\ge \cdots \ge \lambda_p(Q)$ to denote the eigenvalues of $Q$ in non-increasing order, and $\kappa(Q) = \lambda_1(Q)/\lambda_p(Q)$ to denote its condition number. 

For any two sequences $a_n$ and $b_n$, we write $a_n\lesssim b_n$ if there exists some constant $C$ such that $a_n \le Cb_n$. 
The notation $a_n\asymp b_n$ stands for $a_n \lesssim b_n$ and $b_n \lesssim a_n$.

We use $\bI_d$ to denote the $d\times d$ identity matrix.
For $m\ge 1$, we let $[m] = \{1,2,\ldots,m\}$. Lastly, we use $c,c',C,C'$ to denote positive and finite absolute constants that unless otherwise indicated can change from line to line.

	\section{Bounding the risk $\R(\wh B)$}\label{sec_pred_fr}

    In this section we derive and discuss bounds on the risk $\R(\wh B)$ defined in  (\ref{def_risk}), corresponding to  the  predictor $\wh Y_{\wh B}^*$. Our results are valid for any $\wh B\in \RR^{p\times q}$ that can be either random depending on $\X$ or fixed, where $q\le p$ but is allowed to be random.

    \subsection{Preliminaries}\label{sec_prelim}

    As the risk $\R(\wh B)$ is defined relative to  the first two moments of $(X,Y)$, which are further linked to quantities $(A, \beta, \C, \Gamma, \sigma^2)$ under model (\ref{main_model}), our risk bounds are written in terms of the components of $\theta := (K,\beta,A,\C,\Gamma,\sigma^2)$. We thus start by formally defining model (\ref{main_model}) with respect to $\theta$. 

    \begin{definition}[(Sub-Gaussian) Factor Regression Model]\label{frm}
    We say the pair $(X,Y)$ follows the model FRM$(\theta)$ with $\theta=(K,\beta,A,\C,\Gamma,\sigma^2)$, and write $(X,Y)\sim \PP_\theta$ or $(X,Y)\sim \textrm{FRM}(\theta)$, when
    \begin{enumerate}
        \item[(1)] Equation (\ref{main_model}) holds with matrix $A\in \RR^{p\times K}$, vector $\beta\in \RR^K$, and random quantities $(Z, W, \eps)\in (\RR^K, \RR^p,\RR)$ that are mutually independent;
        \item[(2)] $W$ and $\eps$ are mean zero with $\EE_\theta[WW^\T] = \Gamma$ and $\EE_\theta[\eps^2] = \sigma^2$, and $Z$ is also mean zero without loss of generality, with $\EE_\theta[ZZ^\T] = \C$.
        \item[(3)] Both $A$ and $\sz$ have rank equal to $K$.
    \end{enumerate} 
    We further say $(X,Y)\sim \textrm{sG-FRM}(\theta)$ if the following holds in addition to (1) -- (3)
    \begin{enumerate} 
        \item[(4)] There exist finite, absolute positive constants $\e, \w$ and $\z$ such that
		\begin{enumerate}
			\item $\eps$ is $\sigma \e$ sub-Gaussian\footnotemark\footnotetext{A mean zero random variable $x$ is called $\gamma$ sub-Gaussian if $\EE[\exp(tx)]\le \exp(t^2\gamma^2/2)$ for all $t\in \RR$.};
			\item $Z = \C^{1/2}\wt Z$ where $\wt Z$ is $\z$ sub-Gaussian\footnotemark\footnotetext{A mean zero random vector $x$ is called $\gamma$ sub-Gaussian if $\langle x, v\rangle$ is $\gamma$ sub-Gaussian for any unit vector $v$.} with $\EE_{\theta}[\wt Z\wt Z^\T] = \bI_K$;
			\item $W = \sw^{1/2}\wt W$ where $\wt W$ is $\gamma_w$ sub-Gaussian with $\EE_\theta [\wt W\wt W^\T] = \bI_p$.
    \end{enumerate}
    \end{enumerate} 
    \end{definition}
    Since there exist multiple parameters $\theta$ for which  $(X,Y)$ has the same joint distribution, the model  is not identifiable without further restrictions on the parameter space.  As this work is devoted to  the prediction of  $Y$, and not to the estimation of $\theta$, this is not problematic. We thus allow for this lack of identifiability and our subsequent analysis of $\R(\wh B)\coloneqq\EE_\theta[(Y_* - \wh Y_{\wh B}^*)^2]$ is valid for any $\theta$ such that $(X,Y)\sim $ sG-FRM($\theta$). In particular, the analysis is applicable to  any identifiable sG-FRM$(\theta)$,  whenever further structure on $\theta$ is  added to Definition \ref{frm}. We note that $\R(\wh B)$ depends on $\theta$, but we suppress this dependence in the notation for simplicity.

    \subsection{Benchmark of $\R(\wh B)$}\label{sec_benchmark}
    
    To provide a benchmark for $\R(\wh B)$, we let \begin{equation}\label{def_blp}
    \alpha^* := \arg\min_{\alpha}\EE\left[(Y_* - X_*^\T \alpha)^2\right] = [\Cov(X)]^{+}\Cov(X,Y)
    \end{equation}
    denote the coefficient of the best linear predictor (BLP) of $Y_*$ from $X_*$, where $[\Cov(X)]^+$ is the Moore-Penrose pseudoinverse of $\Cov(X)$. For any $\theta = (K,A,\beta, \C,\Gamma, \sigma^2)$ such that $(X_*, Y_*)\sim \text{FRM}(\theta)$ with corresponding latent vector $Z_*$, we have the following chain of simple equalities from our independence assumptions
    \begin{eqnarray}\label{decomp_full}
     \R(\wh B)  &=& \EE_\theta\left[(Y_* - X_*^\T \alpha^*)^2\right] + \EE_\theta\left[(X_*^\T \alpha^* - \wh Y^*_{\wh B})^2\right]  \nonumber \\
     &= & \sigma^2 + \EE_\theta\left[
            (Z_*^\T\beta - X_*^\T \alpha^*)^2
        \right] + \EE_\theta\left[(X_*^\T \alpha^* - \wh Y^*_{\wh B})^2\right] \\\nonumber
        & = & \sigma^2 +  \EE_\theta\left[
            (Z_*^\T\beta -  \wh Y^*_{\wh B})^2
        \right].
    \end{eqnarray}
    We interpret  the term $\sigma^2 = \EE_\theta[\eps^2]$ as an oracle risk value because it is the minimal risk of predicting $Y_*$ from $Z_*$, had $Z_*$ been observable. We thus focus on bounding the difference $\R(\wh B) - \sigma^2$ and refer to it as {\em excess risk}, with the tacit understanding that the excess is relative to oracle prediction.

We further note that the term $\EE_\theta[(Z_*^\T\beta - X_*^\T \alpha^*)^2]$ in (\ref{decomp_full}) is the  minimal risk incurred by predicting $Z_*^\T\beta$ by  $X_*^\T \alpha^*$, with an observable $X_*$. Display (\ref{decomp_full}) shows that it is a population level cost that is incurred in any risk analysis of a predictor of type (\ref{def_pred_B_intro}) performed under FRM$(\theta)$. 
Lemma \ref{lem_pop_X} below quantifies its size, and makes use of the signal-to-noise ratio given by 
\begin{equation}\label{def_snr}
\xi \coloneqq {\lambda_K(A\C A^\T )/ \|\Gamma\|_{{\rm op}}}.
\end{equation}
Its proof can be found in Appendix \ref{app_proof_pred}. 
\begin{lemma} \label{lem_pop_X}
    For any $\theta=(K,A,\beta, \C,\Gamma, \sigma^2)$ with invertible $\Gamma$ such that $(X,Y)\sim \text{FRM}(\theta)$,
	\begin{equation}\label{ineq_R_star}
	    \frac{\xi}{1+\xi}\beta^\top (A^\top \sw^{-1}A)^{-1}\beta\le \EE_\theta \left[
        (Z_*^\T\beta - X_*^\T \alpha^*)^2\right] \le \beta^\top (A^\top \sw^{-1}A)^{-1}\beta.
	\end{equation}
\end{lemma}
The inequalities above become asymptotically tight when the signal retained in $K$ dimensions by $X$ dominates the ambient noise, that is, when $\xi \rightarrow \infty$ as $p\to \i$. In general, as soon as $\xi > c$, for some $c > 0$ and $\Sigma_W$ is well conditioned such that $\kappa(\sw) = \lambda_1(\sw)/\lambda_p(\sw) < C$, we further obtain,   using (\ref{decomp_full}), for any $\wh B$, that 
\begin{equation}\label{risk_lb}
    \R(\wh B) - \sigma^2  \ge \EE_\theta \left[
        (Z_*^\T\beta - X_*^\T \alpha^*)^2\right]  \gtrsim \og \beta^\top(A^\top A)^{-1}\beta.
\end{equation}
Therefore a risk analysis of linear predictors under factor regression models, which consists in upper bounding  $\R(\wh B) - \sigma^2$, will necessarily   include terms  larger than  $\og \beta^\top(A^\top A)^{-1}\beta$  in the risk bounds, irrespective of the construction of the linear predictor.
 If, in addition, $A\sz A^\T  $ is well-conditioned with $\lambda_1(A\sz A^\T) / \lambda_K(A\sz A^\T)\le C$,
then
	\begin{align*}
	      \beta^\top ( A^\top \sw^{-1} A )^{-1} \beta
	    \asymp
	    	   \|\sw\|_{\op} \beta^\T\sz^{1/2} \left(\sz^{1/2}A^\T A\sz^{1/2}\right)^{-1}\sz^{1/2}\beta
	    \asymp {\beta^\T \sz \beta \over \xi}
	\end{align*}
and	Lemma \ref{lem_pop_X} in turn implies
	\[ {\beta^\T \sz \beta \over 1+ \xi} \lesssim \EE_\theta \left[
        (Z_*^\T\beta - X_*^\T \alpha^*)^2\right]   \lesssim {\beta^\T \sz \beta \over \xi}.
	\]
This demonstrates that the signal-to-noise ratio $\xi$ must necessarily dominate $\beta^\T \sz \beta $ for
 the excess risk $\R(\wh B)-\sigma^2$ to vanish as $p\to\infty$.\\

\subsection{Upper bound of the risk $\R(\wh B)$}

To motivate our main result, we first introduce some key quantities 
that appear in the risk bound derivation for any generic  $\wh B$ leading to the predictors of type  (\ref{def_pred_B_intro}).

      The prediction risk bound depends on  $W$ in Definition \ref{frm}, specifically on the  noise level of $n^{-1}\|\W^\T \W\|_{\op}$. To quantify this noise level, we use the following deviation bound from Lemma \ref{lem_op_norm} in Appendix \ref{sec_proof_aux}. For any $\theta$ such that $(X,Y)\sim \text{sG-FRM}(\theta)$, one has
	\begin{equation}\label{bd_W_op}
	\PP_{\theta} \left\{{1\over n}\| \W^\T\W \|_{{\rm op}} \le \delta_W\right\} \ge  1 -  e^{-n}
	\end{equation}
	where $\errw$ is defined as 
	\begin{equation}\label{def_ErrW}
	\delta_W := \errw(\theta)  =
	c \left[\|\Gamma(\theta)\|_{{\rm op}} + {\tr(\Gamma(\theta))\over n}\right],
	\end{equation}
	with $c = c(\gamma_w)$ being some positive constant. The quantity $\errw$
	will play a role in the risk bound and it could take any non-negative value in general. When $\lambda_1(\Gamma)\le C$ for some constant $C>0$, one has 
	$
	    \errw \lesssim 1 + {p/n}.
	$
	When $\lambda_p(\Gamma) \ge c$ for some constant $c>0$, we have $\errw \gtrsim 1 + p/n$. In particular, if $c\le \lambda_p(\Gamma) \le \lambda_1(\Gamma) \le C$, we have $\errw \asymp 1 + p/n$. This holds for instance when $\sw$ is diagonal with entries bounded away from 0 and $\i$, independent of $n$.

	We write the projection onto the column space of $\wh B$ as \[
    P_{\wh B} = \wh B [ \wh B^\T \wh B]^{+} \wh B^\T =\wh B \wh B^+,\]
    its complement as $P_{\wh B}^{\perp} = \bI_p - P_{\wh B}$ and $\wh r = \rank(\X P_{\wh B})$.
    Since $\wh B [ \X \wh B]^{+} = P_{\wh B} [ \X P_{\wh B}]^{+}$, as proved in Lemma \ref{lem_invariant} in Appendix \ref{sec:main proofs}, we find that $\wh Y_{\wh B}^* = X_*^\T \wh B [ \X\wh B]^{+} \Y=  \wh Y_{P_{\wh B}}^*$ making clear that 
    the component of the data matrix orthogonal to the range of $\wh B$, $\X P^\bot_{\wh B}$, is not used for prediction.  It is  natural therefore that the size of this component, as measured by its largest singular value, $ \sigma_1^2(\X P_{\wh B}^{\perp})$,  will affect the risk bound, and needs to be contrasted with  the size 
     of the retained signal, $\X P_{\wh B}$,  as measured by its  smallest non-zero singular value $ \sigma_{\wh r}^2(\X P_{\wh B} )$. These two quantities appear in the risk bound below.
    

	We now state our main theorem; its proof is deferred to 	 Appendix \ref{sec_proof_thm_pred}. Recall that $\R(\wh B)$ is the risk defined in (\ref{def_risk}). Write $a\wedge b = \min\{a, b\}$.

\begin{thm}\label{thm_pred}
    	Let $\wh B=\wh B(\X)\in \RR^{p\times q}$ for some $q\ge 1$, and set
		\begin{equation}\label{def_eta_psi}
		   \wh r := \rank\left(\X P_{\wh B}\right),\qquad  \wh \eta := {1\over n}\sigma_{\wh r}^2\left(\X P_{\wh B} \right),\qquad 
		    \wh \psi := {1\over n}\sigma_1^2\left(\X P_{\wh B}^{\perp}\right).
		\end{equation}
	For any $\theta = (K,A,\beta,\C,\Gamma, \sigma^2)$ with $K\le Cn/\log n$ for some positive constant $C=C(\z)$ such that $(X,Y)\sim \textrm{sG-FRM}(\theta)$, there exists some absolute constant $c>0$ such that
		\begin{align}\label{bd_R_B_star}
		\PP_{\theta} \bigg\{  \R(\wh B) -  \sigma^2
		&~\lesssim ~ \left[{\|\Gamma\|_{{\rm op}}\over \wh \eta}\wh r+ \left(1+ {\errw \over \wh \eta}\right)(K\wedge \wh r+\log n) \right] {\sigma^2\over n}\\ \nonumber
		&\quad + \left[ \left(1 +  {\|\Gamma\|_{\op}\over \wh \eta }\right) \errw+ \left(1+ {\errw \over \wh \eta}\right)\wh \psi \right]\beta^\T(A^\T A)^{-1}\beta\bigg\} \ge 1-c/n.
		\end{align}
	Here the symbol $\lesssim$ means the inequality holds up to a multiplicative constant possibly depending on the sub-Gaussian constants $\gamma_\eps$, $\gamma_{z}$ and $\gamma_w$.
\end{thm}

    	Since we aim to provide a unified analysis of the risk for a general $\wh B$, the bound (\ref{bd_R_B_star}) itself depends on the random quantities $\wh r$, $\wh \eta$ and $\wh \psi$. To make it informative, one needs to further control these random quantities for specific choices of $\wh B$. The main usage of Theorem \ref{thm_pred} is thus to reduce the task of bounding $\R(\wh B)$ to the relatively easier one of controlling $\wh r$, $\wh \eta$ and $\wh \psi$.  We will demonstrate this for several choices of $\wh B$ in the following sections.
	


	Theorem \ref{thm_pred} holds for any estimator $\wh B \in \RR^{p\times q}$ that is constructed from $\X$ with any $q\ge 1$. 
	We now explain the various terms in the bound (\ref{bd_R_B_star}). Recall that $\wh Y^*_{\wh B} = X_*^\T \wh B(\X \wh B)^+\Y$ and $\Y = \Z\beta + \Eps$. To aid intuition, by adding and subtracting terms, we have
	\begin{align}\nonumber
	   \wh Y^*_{\wh B} - Z_*^\T \beta
	   &= X_*^\T\wh B(\X \wh B)^+\Eps + X_*^\T\alpha^* - Z_*^\T \beta+ X_*^\T\left[ \wh B(\X \wh B)^+\Z\beta - \alpha^*\right]\\\nonumber
	   &=  X_*^\T\wh B(\X \wh B)^+\Eps + \left(X_*^\T\alpha^* - Z_*^\T \beta\right) + X_*^\T\wh B(\X \wh B)^+(\Z\beta - \X \alpha^*)\\\label{decomp_error}
	   &\quad +X_*^\T\left[ \wh B(\X \wh B)^+\X -\bI_p\right]\alpha^*.
	\end{align}
	We discuss the four terms above one by one. 
	\begin{itemize}
	    \item The first term leads to the following variance term in (\ref{bd_R_B_star}):
	\[
	     \left[{\|\Gamma\|_{{\rm op}}\over \wh\eta}\ \wh r+ \left(1+ {\errw \over \wh\eta}\right)(K\wedge \wh r+\log n) \right] {\sigma^2\over n}.
	\]
	We see that the random variable $\wh \eta$ quantifies the retained signal in $\wh B(\X \wh B)^+$
	by noting that $\|\wh B(\X \wh B)^+\|_{\op}^2 = \|P_{\wh B}(\X P_{\wh B})^+\|_{\op}^2 \le (n\wh\eta)^{-1}$. The two factors $\|\Gamma\|_{\op}/\wh \eta$ and $(1+\errw / \wh\eta)$ come from bounding the second moments of $W_*$ and $AZ_*$ from $X_*=AZ_*+W_*$, respectively, relative to the retained signal $\wh\eta$. The dimension $\wh r$ reflects the complexity of $\X P_{\wh B}$ and the integer $K$ is the intrinsic dimension of the latent factor, thus only appearing in the term containing $(1+\errw/\wh\eta)$.
	
	\item The second and third terms in (\ref{decomp_error}) lead to the following term in (\ref{bd_R_B_star}), which can be interpreted as arising from the fact that $Z_*$ and $\Z$ are not observed:
	\[
	    \left(1 +  {\|\Gamma\|_{\op}\over \wh\eta }\right) {\errw  \over \|\Gamma\|_{\op}}\cdot  \|\Gamma\|_{\op} \beta^\T(A^\T A)^{-1}\beta.
	\]
	With slight abuse of terminology, we refer to this as a bias term. The factor $\|\Gamma\|_{\op} \beta^\T(A^\T A)^{-1}\beta$ is irreducible, as  argued in (\ref{risk_lb}), the term $\|\Gamma\|_{\op}/ \wh\eta$ has been explained in the first term, and the inflation factor $\errw /\|\Gamma\|_{\op}$ is due to the inflated noise level of $n^{-1}\|\W^\T\W\|_{\op}$ compared to $\|\Gamma\|_{\op}$. 
	
	\item The fourth term in (\ref{decomp_error}) quantifies the error of estimating the best linear predictor $\alpha^*$ under the factor regression  model. In this model, we note that   $\alpha^* = \Sigma^+ A\C \beta$ with $\Sigma := \Cov(X)$. Also noting that $\wh B(\X\wh B)^+\X$ is a projection matrix, the fourth term in (\ref{decomp_error}) represents the error of estimating the range space of $\Sigma^+A$, which is exactly zero if the range of $\wh B(\X\wh B)^+\X$ contains the range of $\Sigma^+ A$. In general, the bound in (\ref{bd_R_B_star}) corresponding to this term is 
	\[
	    \errw \beta^\T (A^\T A)^{-1}\beta + \left(1+ {\errw \over \wh \eta}\right)\wh \psi\cdot  \beta^\T(A^\T A)^{-1}\beta,
	\]
	where the first part is the error of estimating the range space of $P_{\wh B}\Sigma^+A$ while the second part is that of estimating the range space of $P_{\wh B}^{\perp}\Sigma^+A$, controlled by $\wh \psi$. 
	\end{itemize}
	
	\begin{remark}
	In light of the  above discussion, we make two important remarks. First, to maintain a fast rate of the risk bound in (\ref{bd_R_B_star}), we should retain enough signal in $\X P_{\wh B}$ relative to the noise $\errw$ such that $\wh \eta \gtrsim \errw$ with high probability. Second, if this is the case, the bound (\ref{bd_R_B_star}) simplifies to
	\begin{equation*}
			\R(\wh B) -\sigma^2
		 \lesssim \left[{\|\Gamma\|_{{\rm op}}\over \wh\eta}\wh r+ (K\wedge \wh r+\log n) \right] {\sigma^2\over n}+ \left(\errw+ \wh\psi \right)\beta^\T(A^\T A)^{-1}\beta. 
    \end{equation*} 
    As $\wh r = \rank(\X P_{\wh B})$ increases, meaning that the predictor can be interpreted as more complex, the variance term increases, while the term $\errw\beta^\T(A^\T A)^{-1}\beta$ is not affected.
    
    If $\wh \psi$ decreases as $\wh r$ increases (as seen with the PCR predictor studied in the next section), the term $\wh\psi\beta^\T(A^\T A)^{-1}\beta$, corresponding to the error of estimating the range space of $P_{\wh B}^{\perp} \Sigma^+ A$, gets smaller. 
    
    Therefore, the tradeoff of using a more complex predictor lies between the increasing variance and the decreasing error of estimating the range space of $P_{\wh B}^{\perp} \Sigma^+ A$, provided that enough signal is retained in $\X P_{\wh B}$. A more transparent tradeoff can be seen for the PCR predictor analyzed in the next section. More generally, for each of our examples, 
	we will see the mechanism by which $\wh r$, $\wh \eta$, and $\wh \psi$ are controlled. 
	\end{remark}

	\section{Analysis of Principal Component Regression under the factor regression model}\label{sec_pred_pcr}
	
	In this section we use the general result, Theorem \ref{thm_pred}, to derive risk bounds for the popular Principal Component Regression (PCR) method.
For any integer $1\le k\le \rank(\X)$, the PCR-predictor PCR-$k$ corresponds to taking $\wh B = \U_k$,   the $p\times k$ matrix with columns equal to the first $k$ right singular vectors of $\X$ corresponding to the non-increasing singular values $\sigma_1(\X)\ge \sigma_2(\X)\ge \cdots$. 
We start by giving  risk bounds for PCR-$k$ for any $k$ in the corollary below.  For simplicity, we write 
	\[
		\wh \lambda_k = 
		{1\over n}\sigma_k^2(\X)    
	\]
with the convention that $\wh \lambda_0 = \i$ and $\wh \lambda_k=0$ for all $k>\rank(\X)$.
All the proofs of this section can be found in Appendix \ref{app_proof_PCR}. 

	\begin{cor}\label{cor_PCR_k}
		For any $\theta=(K, A,\beta,\C,\Gamma,\sigma^2)$ with  $K \le C n /\log n$ and some positive constant $C = C(\gamma_z)$ such that $(X,Y)$ follows sG-FRM$(\theta)$, there exists some absolute constant $c>0$ such that, for any   $k$ (possibly random), 
		\begin{align}\label{bd_R_U_k}
		\PP_{\theta}\left\{\R(\U_{k})-\sigma^2
		 \lesssim    \wh B(k)\right\} \ge 1-cn^{-1}
		\end{align}
		where  $ \wh B(k) = \wh  B_1(k) + \wh B_2(k)$ and  
		\begin{align}\label{def_B_k_1}
			\wh  B_1(k) &:= \left[{\|\Gamma\|_{{\rm op}}\over \wh\lambda_k}k+ \left(1+ {\errw \over \wh\lambda_k}\right)(K\wedge k+\log n) \right] {\sigma^2\over n}\\
			\label{def_B_k_2}
		\wh 	 B_2(k) &:= \left( {\|\Gamma\|_{\op}\over \wh\lambda_k } \errw+ \errw+\wh\lambda_{k+1} \right)\beta^\T(A^\T A)^{-1}\beta.
		\end{align}
	\end{cor}
  Corollary \ref{cor_PCR_k} follows immediately from the identities $\sigma_{k}^2(\X P_{\U_k}) = \sigma_{k}^2(\X)$ and $\sigma_1^2(\X P_{\U_k}^{\perp}) = \sigma_{k+1}^2(\X)$, and an application of  Theorem \ref{thm_pred}  with  
   \[ 
	    \wh r = k, \qquad  \wh \eta  = \wh \lambda_k,\qquad  \wh \psi = \wh \lambda_{k+1}\qquad \text{almost surely.}
	\] 
		The bound $\wh B(k)$ in (\ref{bd_R_U_k}) depends on $\wh \lambda_k$ and $\wh \lambda_{k+1}$, which may be further controlled by $\lambda_k(A\C A^\top) - \errw$ and $ \lambda_{k+1}(A\C A^\top) + \errw $, respectively,
		in order to make the bound more informative (see, for example, the proof of Remark \ref{rem_K} in Appendix \ref{app_proof_PCR}).
		Nevertheless,  (\ref{bd_R_U_k}) illustrates the effect of $k$ and hints at the choice $k=\wh s$ with
	\begin{equation}\label{def_s_hat}
		\wh s = \max \left\{k\ge 0: \ 
		\wh \lambda_k \ge C_0 \errw\right\} .
	\end{equation}
	Here $\errw$ is defined in (\ref{def_ErrW}) and $C_0$ is some positive  constant.  The quantity $\wh s$ corresponds to what is known as the \textit{elbow method}, and is a ubiquitous approach for selecting the number of top principal components of the data matrix $\X$. The quality of $\wh s$ as an estimator of the effective rank of $\Sigma=\Cov(X)$ has been analyzed in \cite{bunea2015}, but its role
	in PCR has received little attention. 
	By definition, $\wh \lambda_{\wh s+1} < C_0 \errw \le \wh \lambda_{\wh s}$, which implies
	\begin{equation*}
	 \wh B(\wh s)   \lesssim  (
    \wh	s + \log n)  {\sigma^2\over n}+  \errw  \beta^\T(A^\T A)^{-1}\beta, \qquad \text{ almost surely}.
    	\end{equation*}
    Furthermore,  Weyl's inequality implies $\wh\lambda_{K+1} \le \sigma_1^2(\W)/n$ and, in conjunction with (\ref{bd_W_op}),  and by choosing $C_0>1$, we obtain $\wh s\le K$ with high probability.
We summarize this discussion in the following result pertaining to prediction via the first $\wh s$ principal components selected  via the elbow method.
	 
	 \begin{cor}\label{cor_PCR_delta_w}
    For any $\theta=(K,A,\beta,\C,\Gamma,\sigma^2)$ with $K \le  C n/\log n$ such that $(X,Y)$ follows sG-FRM$(\theta)$,
    we have for $\wh s$ defined in (\ref{def_s_hat}) for any $C_0>1$,
	\begin{align}\label{bd_R_U_shat}
	    \PP_\theta\left\{
	    \R(\U_{\wh s})-\sigma^2
		\lesssim  \left(K +\log n\right) {\sigma^2\over n}+ \errw \beta^\T(A^\T A)^{-1}\beta\right\} \ge 1- O(n^{-1}).
	\end{align}
 
	\end{cor}
	
	\begin{remark}
	\label{rem_K}
	$\ $
	\begin{enumerate}
	   
	   \item We refer to the method analyzed in Corollary \ref{cor_PCR_delta_w} as the {\it theoretical} elbow method, as it involves the theoretically optimal threshold level $\delta_W$. The next section analyzes the performance of a {\it data-adaptive} elbow method.  
	    \item For any $\theta$, we show in Appendix \ref{app_proof_PCR} that,  if $\lambda_K(A\C A^\T) \ge  C\errw$ for some sufficiently large constant $C>0$, then $\wh \lambda_{K} \ge C_0\errw$ holds for some $C_0 > 1$ with high probability. The event
	    $\{ \wh \lambda_{K} \ge C_0\errw\} $ 
	    implies  $\{\wh s \ge K\}$ which, in conjunction with the high probability event $\{\wh s \le K\}$, guarantees $\wh s = K$ with high probability.
	    Corollary \ref{cor_PCR_delta_w} thus covers the risk of PCR-$K$, that is, the risk of the PCR predictor corresponding to the true  $K$ of  this $\theta$.
	\end{enumerate}

    \end{remark}
    

	 \subsection{Selection  of the number of retained principal components via penalized least squares   }\label{pcr_newmethod}
	A practical issue of PCR-$\wh s$ is that the selection of $\wh s$ according to  (\ref{def_s_hat}) relies on a theoretical order $\errw$ in  (\ref{def_ErrW}), which   depends on the unknown quantities $\|\Gamma\|_{\op}$ and $\tr(\sw)$. To overcome this difficulty, we 
	provide an alternative, data dependent procedure, which shares the risk bound derived for  PCR-$\wh s$.\\

	Our  procedure of selecting the number of retained principal components is adopted from    \cite{rank19},   originally proposed for selecting the rank of the coefficient of a  multivariate response regression model $\Y = \X B + \W$. 
	The factor model $\X = \Z A^\T + \W$ is a particular case 
	with $\X = \bI_{n\times p}$ and $B = \Z A^\T$, and, following   \cite{rank19}, we define
	\begin{equation}\label{est_K}
	       \wt s := \argmin_{0\le k\le \bar K} \wh v_k^2,\quad \textrm{ with } \quad \wh v_k^2:= {\|\X - \X_{(k)}\|_F^2 \over np - \mu_n k}, \quad \textrm{ and  }\quad  \bar K:= \left\lfloor {\kappa \over 1+\kappa}{np \over \mu_n} \right\rfloor \wedge n \wedge p,
	\end{equation}
 for a given sequence $\mu_n > 0$. 
 Here $\kappa >1$ is some absolute constant introduced to avoid division by zero. We write $\X_{(k)}$ as the best rank $k$ approximation of $\X$. More specifically, let the SVD of $\X$ as $\X = \sum_{j}\sigma_j u_j v_j^\T$ with non-increasing $\sigma_j$ and we have $\X_{(k)} = \sum_{j=1}^k \sigma_j u_j v_j^\T$.

 The denominator  of the ratio defining $\wh v_k^2$ can be viewed as a penalty on the numerator, with tuning sequence $\mu_n$. 
 From \citet[Equation 2.7]{rank19}, the minimizer $\wt s$ conveniently has a closed form
 \[ \wt s= \sum_k 1\{ \wh \lambda_k\ge \mu_n \wh v_k^2\},\]
  counting the number of singular values of $\X$ above a {\em variable} threshold. This is in contrast to the  elbow method in (\ref{def_s_hat}),  which counts the number of singular values of $\X$ above the  {\em fixed} threshold $\mu=C_0\errw$, as
 \[ \wh s=  \sum_k 1\{ \wh \lambda_k \ge \mu \}.\]

 We note that when $\sw=0$, $\| \X-\X_{(k)}\|_F=\|\Z A^\T - (\Z A^\T )_{(k)}\|_F=0$ for any $k\ge K$. Hence there are multiple minima (zeroes in this case) in $\wh v_k^2$, and if we adopt the convention to choose the first index $k$ with 
 $\| \X-\X_{(k)}\|_F=0$, we find $\wt s=K$, almost surely. The risk of PCR-$K$ has already been discussed in Remark \ref{rem_K} above.\\
 
	The theoretical guarantees proved in  \cite{rank19} are based on the assumption that  $\W$ has i.i.d. entries with zero mean and bounded fourth moments. Proposition \ref{prop_K} extends this to models in which the rows of  $\W$ are allowed to have dependent entries, when they follow a sub-Gaussian distribution. We show that the choice  $\mu_n = c_0(n +  p)$, for some absolute numerical constant $c_0$,  leads to desirable results.  The induced size of $\bar K$, for this $\mu_n$,  is of order $n\wedge p$. We found the choice $c_0=0.25$ worked well for all our simulations, as presented in Section \ref{sec_sims}. 

    Let $r_e(\sw) = \tr(\sw)/ \|\sw\|_{\op}$ denote the effective rank of $\sw$. The following proposition shows that $\wt s$ finds, adaptively, the {\em theoretical} elbow.

	\begin{prop}\label{prop_K}
	    Let $\wt s$ be defined in (\ref{est_K}) with $\mu_n = c_0 (n + p)$ for some absolute constant $c_0>0$. For any $\theta=(K,A,\beta,\C,\Gamma,\sigma^2)$ such that $(X,Y)$ follows sG-FRM$(\theta)$, $\log p \le cn$, $K\le \bar K$ and 
	    \begin{align}\label{r_e}
	        r_e(\Sigma_W) \ge c'(n\wedge p)
	    \end{align} for some positive constants $c=c(\w)$ and $c'=c'(\w)$,  we have
        \begin{align}\label{bd_R_U_s}
    	    \PP_\theta\left\{ \wt s \le K,\quad  \wh\lambda_{\wt s} \gtrsim \errw,\quad    \wh \lambda_{\wt s+1} \lesssim \errw\right\} \ge 1- O(1/n).
    	\end{align}
	\end{prop}
	Condition $K\le \bar K$ holds, for instance, if $K\le c''(n\wedge p)$ with $c'' \le \kappa/(2c_0(1+\kappa))$.
	We explain the connection between  restriction (\ref{r_e}) 
and the proposed  choice of $\mu_n$. 
Using elementary algebra, \citet[Theorem 6 and Proposition 7]{rank19} proves
    the  deterministic result
    \begin{align}\label{ineq:BW}
        \left\{  \frac{2\sigma_1^2 (\W)}{\| \W\|_F^2/ (np)} \le \mu_n\right\} ~  \subseteq  ~ \left\{\wt s \le K\right\},
    \end{align} 
    which shows that if $\mu_n$ is appropriately large, then the selected $\wt s$ is less than or equal to dimension $K$ of the factor regression model generating the data. On the other hand, by concentration inequalities of $\|\W\|_F^2/n$ and $\sigma_1^2(\W)/n$ around $\tr(\sw)$ and $\errw$, respectively (see the proof of Proposition \ref{prop_K} in Appendix \ref{app_proof_PCR}), the  bound
        \begin{align}\label{mu_n}
        \frac{2\sigma_1^2 (\W)}{\| \W\|_F^2/ (np)} \lesssim np \frac{\errw}{\tr(\Gamma)} = p+ {np\over r_e(\Gamma)} 
        \end{align}
  holds with probability larger than $1 - O(1/n)$. Thus, in view of (\ref{ineq:BW}) and (\ref{mu_n}), the event $\{\wt s \leq K\}$ holds with high probability as soon as $\mu_n >   p+ {np/ r_e(\Gamma)}$. Under (\ref{r_e}), we arrive 
  at the choice $\mu_n = c_0( n+p)$ and, in turn, $\bar K= O(n \wedge p)$.

    We note that (\ref{r_e}) holds, for instance, in the commonly considered setting
        \begin{equation}\label{cond_bounded}
		0< c'\le \lambda_p(\sw) \le \lambda_1(\sw)\le C' <\i,
		\end{equation}
while being more general. One can alternatively consider other error structures, for instance,  with $r_e(\Gamma)=O(1)$, in which case the above reasoning leads to the choice $\mu_n \gtrsim np$. However,  this would limit the range of $K$, up to $\bar K=O(1)$ in (\ref{est_K}), while our interest is in factor regression models with dimensions allowed to grow with $n$.\\

	Proposition \ref{prop_K} in conjunction with Corollary \ref{cor_PCR_k} immediately leads to the following  risk bound of PCR-$\wt s$. It coincides with the bound for  PCR-$\wh s$ in display (\ref{bd_R_U_shat}) of Corollary \ref{cor_PCR_delta_w}.
	
    \begin{cor}\label{cor_PCR_s_tilde}
     Let $\wt s$ be defined in (\ref{est_K}) with $\mu_n = c_0 (n + p)$ for some absolute constant $c_0>0$. For any $\theta=(K,A,\beta,\C,\Gamma,\sigma^2)$ with $K\le Cn/\log n$ such that $(X,Y)$ follows sG-FRM($\theta$), $\log p \le cn$, $K\le \bar K$ and (\ref{r_e}) holds,
     for some positive constants $c=c(\gamma_w)$ and $c'=c'(\gamma_w)$,  we have
        \begin{align}\label{bd_R_U_s_tilde}
    	    \PP_\theta\left\{\R(\U_{\wt s}) - \sigma^2 \lesssim  (K + \log n){\sep^2 \over n}+\errw \beta^\T (A^\T A )^{-1}\beta \right\} \ge 1- O(n^{-1}).
    	\end{align}
    \end{cor}

    \subsection{Existing results on PCR}\label{sec_existing}
   
    
    
 		
 			 Due to the popularity and simplicity of PCR, its prediction properties under the factor regression model have been studied for nearly two decades. Most existing theoretical results, discussed below,  are asymptotic in $n$ and $p$ and, to the best of our knowledge,  have been established for a model of known dimension $K$,  or when  $K$ is identifiable under additional restrictions on the parameter space, and can be consistently estimated. 
 			 
 		   The fact that   PCR prediction, under the factor regression model with known or identifiable $K$,  has asymptotically vanishing excess risk  only when both $p$ and $n$ grow  to $\infty$ is a well known result. 
 		    This can already be seen from our derivation  (\ref{risk_lb}) above, which shows that  a necessary condition for prediction with vanishing excess risk, under factor regression models with well conditioned $\Sigma_W$,  is $\og \beta^\top(A^\top A)^{-1}\beta \rightarrow 0 $, which can be met when $p \rightarrow \infty$, as explained below.
 		   
 		   This phenomenon was first quantified in  \cite{SW2002_JASA}, where it is shown that 
 		$$
 		\wh Y^*_{\U_K} - Z_{*}^\T \beta = o_p(1) \ \ \text{as} \ \  n, p \to \i.
 		$$
 		This result is the most closely related to ours, and we discuss it in detail below. We also mention that  several later works, for instance  \cite{Bai-factor-model-03} 	and \cite{fan2013large}, provided explicit convergence rates and inferential theory for the {\it in-sample} prediction error $\wh \Y - \Z \beta$, whereas in this work we study out-of-sample performance. For completeness, we comment on these related, but not directly comparable, results in Appendix \ref{app_literature}.


 	In addition to being asymptotic in nature, the results in  \cite{SW2002_JASA}, and also those regarding the  in-sample prediction accuracy, are established under the following set of conditions: 
 	 $K =O(1)$, $\|\beta\|^2=O(1)$, $\|\Gamma\|_{{\rm op}} = O(1)$, as $p\to\i$, and 
 	\begin{equation}\label{cond_PCR_ident}
 	    {1\over p}A^\T A \to \bI_K, \textrm{ as } p\to \i, \quad \C\textrm{ is a diagonal matrix with distinct diagonal entries}.
 	\end{equation}
 	These conditions serve as identifiability conditions for  $\theta =(K,\beta,A,\C,\Gamma,\sigma^2)$ \citep{SW2002_JASA}. 
 	Condition (\ref{cond_PCR_ident}) further implies that, for some constants $0<c\le C<\i$,
 		\begin{equation}\label{cond_PCR}
 		p \lesssim \lambda_K(AA^\T) \le \lambda_1(AA^\T) \lesssim p,\quad c \le \lambda_K(\C) \le \lambda_1(\C) \le C.
 		\end{equation}
 
 	In contrast, our Corollaries \ref{cor_PCR_k}, \ref{cor_PCR_delta_w} and \ref{cor_PCR_s_tilde} are non-asymptotic statements, which hold for any finite $K$, $n$ and $p$, where $K$ is allowed to depend on $n$, with $K\log n \lesssim n$. Consequently, $\|\beta\|_2^2$ and $\lambda_1(\sz)$ are  also allowed to grow with $n$. Furthermore, our  conditions on the signal $\lambda_K(A \C A^\T)$ are much weaker than (\ref{cond_PCR}) to derive the risk bound of PCR-$K$. To see this, and for a transparent comparison, suppose $\og \lesssim 1$ and $\lambda_K(\C)\ge c$. Then from Remark \ref{rem_K} we only require a condition much weaker than $\lambda_K(A A^\T) \gtrsim p$ of \citep{SW2002_JASA} given in (\ref{cond_PCR}) above, namely 
    $$ 
        \lambda_K(A A^\T) \gtrsim 1 + {p\over n}.
    $$
    
    Finally,  the results in \cite{SW2002_JASA} are established for the unique $\theta$ under additional restrictions of the parameter space discussed above, 
    whereas our results are established for any $\theta$ with $K\log n\lesssim n$ such that $(X,Y)$ satisfying sG-FRM$(\theta)$, without requiring $\theta$ to be identifiable. In particular, our results hold for any identifiable $\theta$ that further satisfies (\ref{cond_PCR_ident}).
    
    We conclude our comparison by giving the bound implied by our Corollary \ref{cor_PCR_delta_w}, should the more stringent conditions (\ref{cond_PCR}) be met. Since (\ref{cond_PCR}) implies that $\wh s = K$ with high probability from Remark \ref{rem_K}, 
    Corollary \ref{cor_PCR_delta_w} immediately yields, with probability $1-O(n^{-1})$,   	
     \begin{align*}
     	\R(\U_{K}) -\sigma^2 \lesssim {\log n\over n}\sigma^2+ {\|\Gamma\|_{{\rm op}}\over p} + { \|\Gamma\|_{{\rm op}} \over n},
     		\end{align*}
     and thus, as in  \cite{SW2002_JASA}, 
    $$ \R(\U_{K}) -\sigma^2 = o_p(1) $$
     when $p, n \rightarrow \infty$ and  $\|\Gamma\|_{{\rm op}} = O(1)$.

	\section{Analysis of alternative  prediction methods}\label{sec_pred_GLS_ER}

	 In this section we illustrate the usage of the main Theorem \ref{thm_pred} to  derive  risk bounds under a factor regression model for two other prediction methods: Generalized Least Squares 
 	\citep{bunea2020interpolation}, as an example of another model agnostic predictor construction, and model-tailored prediction, in an instance of an identifiable factor regression model provided  by  the \textit{Essential Regression} framework introduced in \cite{ER}.   	All proofs for this section are contained in Appendix \ref{app_proof_GLS_ER}.

    \subsection{Prediction risks of minimum norm interpolating predictors under factor regression models}\label{sec_pred_GLS}
	In the recent paper \cite{bunea2020interpolation}, risk bounds were established under the factor regression model for the Generalized Least Squares (GLS) predictor, which corresponds to taking $\wh B = \bI_p$:
	\begin{equation}\label{def_GLS}
	    \wh Y^*_{\bI_p} = X_*^\T \X^+\Y.
	\end{equation}
	We recover as these results in Corollary \ref{cor_gls_low} and Corollary \ref{cor_GLS} below, as further illustration of the application of our main theorem. Since $P_{\bI_p} = \bI_p$ and $P_{\bI_p}^\bot =0$, the application of Theorem \ref{thm_pred} with $\wh \psi =0$ amounts to obtaining a lower bound on the smallest non-zero singular value of $\X$ to bound $\wh \eta$. 
	
	
	We consider the low ($p < n$)- and high ($p > n$)-dimensional settings separately. In the former case, GLS reduces to the ordinary least squares (OLS) method. The following corollary states the prediction risk of the OLS under the factor regression model. The proof uses a standard random matrix theory result \citep[see][Theorem 5.39]{vershynin_2012} to show
	$\sigma_p^2(\X) \gtrsim \lp(\sw) n$, which implies
	$\wh \eta \gtrsim \lp(\sw)$.  Recall that  $\kappa(\sw) \coloneqq \lambda_1(\sw)/\lp(\sw)$. 
	
	\begin{cor}[GLS: low-dimensional setting]\label{cor_gls_low}
	    	Suppose $p\log n\le c_0n$ for an absolute constant $c_0\in (0,1)$. For any $\theta = (K,A,\beta,\C,\Gamma, \sigma^2)$ with $K\le Cn/\log n$ and $\lambda_p(\Gamma) > c$ such that $(X,Y)\sim \textrm{sG-FRM}(\theta)$, one has
		\begin{align*}
		\PP_{\theta} \left\{ \R(\bI_p)-\sigma^2  ~\lesssim \left(
		{p + \log n\over n}\sigma^2  +\|\Gamma\|_{\op}~ \beta^\T (A^\T A)^{-1}\beta
		\right)\kappa(\Gamma) \right\} \ge 1-O(n^{-1}).
		\end{align*}
	\end{cor} 

	When $p$ is much larger than $n$, the GLS becomes the minimum $\ell_2$ norm interpolator \citep{bunea2020interpolation}, one method studied in the recent wave of literature on the generalization of overparameterized models with zero or near-zero training error \citep{montanari2019generalization, bunea2020interpolation, muthukumar2019harmless,muthukumar2020classification,hastie2019surprises, feldman2019does, Belkin15849, belkin2019models,belkin2018overfitting,belkin2018understand, 
 belkin2018does, bartlett2019,liang2019}. 
 Theorem \ref{thm_pred} can also be applied to recover a slightly modified form of the prediction risk bound from \cite{bunea2020interpolation} in this case, which we state in the following corollary. Recall that $r_e(\Gamma) = \tr(\Gamma)/\|\Gamma\|_{{\rm op}}$ is the effective rank of $\Gamma$.

	\begin{cor}[GLS: high-dimensional setting. Interpolating predictors.]\label{cor_GLS}
	    For any 
	    $\theta = (K,A,\beta,\C,\Gamma, \sigma^2)$ with $K\le Cn/\log n$ such that $(X,Y)\sim \textrm{sG-FRM}(\theta)$, suppose $\wt W$ defined in Definition \ref{frm} has independent entries and $r_e(\Gamma) > C’n$ for some sufficiently large constant $C‘>0$. Then there exists $c>0$ such that
		\begin{align*}
		\PP_\theta \left\{\R(\bI_p) -\sigma^2  \lesssim {K + \log n\over n}\sigma^2 +
		{n \over r_e(\Gamma)}\sigma^2+  {r_e(\Gamma) \over n} \|\Gamma\|_{{\rm op}} \ \beta^\T (A^\T A)^{-1}\beta\right\} \ge 1-c/n.
		\end{align*}
	\end{cor}
	
	By Proposition 6 of \cite{bunea2020interpolation}, we have $\sigma_n^2(\X)\gtrsim \tr(\sw)$ with high probability when $r_e(\sw)\gtrsim n$. Corollary \ref{cor_GLS} thus follows from 
	Theorem \ref{thm_pred} with $\wh\psi = 0$ and $\wh\eta \gtrsim \tr(\sw)/n$ in the high-dimensional setting. 
	A simplified version of the risk bound in Corollary \ref{cor_GLS}, together with a comparison with PCR-$k$ prediction, is presented  in Section \ref{sec_compare}. \\

	\subsection{Prediction under Essential Regression}\label{sec_pred_er}
	Both Principal Component Regression and Generalized Least Squares are model-agnostic methods, in that they do not use explicit estimates of the model parameters $\theta = (K,A,\beta, \C,\Gamma, \sigma^2)$ to perform prediction. In contrast,  further assumptions can be placed on the factor model to make $\theta$ identifiable, in which case a direct estimate of $A$ can be meaningfully constructed and used for prediction. The Essential Regression (ER) framework introduced in \cite{ER} provides an approach to do this.
	
	Essential Regression is a particular factor regression model under which the latent factor $Z$ becomes interpretable under additional model assumptions. Specifically, under model (\ref{main_model}), one further assumes the following model specifications.
	\begin{ass}\label{ass_model}\mbox{}
		\begin{enumerate}
			\setlength\itemsep{0mm}
			\item[(A0)] $\|A_{j\sbt}\|_1 \le 1$ for all $j\in [p]$.
			\item[(A1)] For every $k\in [K]$, there exists at least two $j\ne \ell \in [p]$, such that $|A_{j\sbt }| = |A_{\ell \sbt}| = e_k$.
			\item [(A2)] 
			There exists a constant $\nu>0 $ such that 
			$$
			\min_{1\le a<b\le K} \left(\, [\C]_{aa}\wedge [\C]_{bb} - |[\C]_{ab}| \, \right)> \nu.
			$$  
			\item[(A3)] The covariance $\sw$ of $W$ is diagonal with bounded diagonal entries.
		\end{enumerate}
	\end{ass}
	\noindent
    The indices $i\in [p]$ satisfying $A_{i\sbt} = e_k$ are called \textit{pure variables} and collected in the set $I$. We use $J = [p]\setminus I$ to denote all the variables that are \textit{non}-pure. 
    
	Within the Essential Regression framework, the matrix $A$ becomes identifiable up to a signed permutation \citep{LOVE}. In fact, $\theta = (K, A, \beta, \C, \Gamma, \sigma^2)$ can be further shown to be identifiable \citep{ER}. 
	
	We explain how to construct predictors of $Y$ tailored to a factor model, and elaborate on the predictor tailored to Essential Regression. 
Under  any  factor model (\ref{main_model}), the best predictor of $Y$ from $Z$ is $Z^\T \beta$. However, since $Z$ is not observable, this expression does not lend itself to sample level prediction.  A practically usable expression for a predictor under the factor regression model can be obtained by the following reasoning. 
Using the Moore-Penrose inverse 
$A^+\coloneqq (A^\T A)^{-1}A^\T $ of the matrix $A$, 
we observe that model (\ref{main_model}) implies 
	\begin{equation*}
	\bar X \coloneqq A^+ X = Z + A^+ W.
	\end{equation*}
 The best linear predictor (BLP) of $Z$ from $\bar X$ is given by
	\begin{equation}\label{def_Theta_bar}
	\wt Z = \Cov(Z, \bar X)[\Cov(\bar X)]^{-1}\bar X = \C\left(\C + A^+\Gamma A^{+\T}\right)^{-1}A^+ X.
	\end{equation}
	The simple observation that 
	\[ 
	 \arg\min_{\alpha} \EE[ (Y - Z^\T\alpha)^2]  = \beta = \arg\min_{\alpha}
	\EE[ (Y - \wt Z^\T\alpha)^2 ]
	\] 
	justifies predicting  $Y$ by $\wt Y = \wt Z^\T\beta$. 
	Inserting the identity  $\beta = \C^{-1}A^+\Cov(X, Y)$ simplifies $\wt Y$ to 
	\begin{align}\nonumber
	\wt Y_{A} &~=~ X^\T A^{+\T}\left(\C + A^+\Gamma  A^{+\T}\right)^{-1}\C\beta \\\label{pred_A}
	&~=~ X^\T A\left[\Cov(A^\T X)\right]^{-1}\Cov(A^\T X, Y), \nonumber 
	\end{align}	
motivating prediction based on a new data point $X_*$ by 
\[Y^{*}_{\wh A} = X_{*}^\T \wh A\left(\wh A^\T\X^\T\X\wh A\, \right)^+\wh A^\T\X^\T \Y, \] 
which has the general form (\ref{def_pred_B_intro}) with  $\wh B = \wh A$,  with  $\wh A$ being an estimator of $A$ tailored to the ER model, developed in \cite{LOVE}. We summarize the construction of $\wh A$ in Appendix \ref{app_love} for completeness. \\

To analyze the prediction risk of $Y^{*}_{\wh A}$ we will also need the following assumption on  the covariance matrix $\C$, which   plays the same role as the Gram matrix in  classical linear regression with random design.
	\begin{ass}\label{ass_C}
		Assume $c\le \lambda_{K}(\C) \le \lambda_{1}(\C) \le C$ for some constants $c$ and $C$ bounded away from $0$ and $\i$. 
	\end{ass}

	

    The prediction risk of $\wh Y^*_{\wh A}$  can be obtained via an application of Theorem \ref{thm_pred}, with the choice $\wh B = \wh A$.  Since $A$ is identifiable under the Essential Regression framework, the estimator $\wh A$ can be compared directly with $A$ and, as shown in \cite{LOVE}, 
    \begin{equation}\label{eqn_a_est}
        \|\wh A - A\|_{\op}^2\le \|A_J\|_0\log(n\vee p)/n
    \end{equation}
    with high probability. The rows of the $p\times |J|$ submatrix $A_J$ of $A$   correspond  to all the index set $J$ of   non-pure variables. The estimation bound (\ref{eqn_a_est}) can be leveraged to obtain a small improvement in the risk  bound by slightly adjusting the proof of Theorem \ref{thm_pred}. Using this approach, we obtain the following result by establishing, with high probability, that 
    \begin{align*}
        & \wh r = K,\\
        &\wh \eta \gtrsim \lambda_K(A\C A^\T),\\
        & \wh \psi \lesssim \|A_J\|_0{\log\pn \over n} + \|\Gamma\|_{\op} := \psi_n(A_J).
    \end{align*}
     
	\begin{thm}[Prediction in Essential Regression] \label{thm_pred_A}
	    Suppose $(X,Y)\sim \textrm{sG-FRM}(\theta)$ with $\theta =  
	    (K,A,\beta,\C,\Gamma, \sigma^2)$ satisfying Assumptions \ref{ass_model} \& \ref{ass_C}, $K\le Cn/\log n$
		and
		$$
		\lambda_K(A\C A^{\top}) \ge c \cdot  \psi_n(A_J)
		$$
		for some sufficiently small constant $c>0$. Then, with probability at least  $1-O(n^{-1})$,
		\begin{align}\label{rate_pred_A}
		\R(\wh A)-\sigma^2 &\lesssim  {K + \log n\over n}\sigma^2+
		 \psi_n(A_J) \beta^\T (A^\T A)^{-1}\beta.
		\end{align}
	\end{thm}		
	
\begin{remark}
	$\ $
    \begin{enumerate}
        \item  We note that the bound (\ref{rate_pred_A}) depends on $\|A_J\|_0$, which in turn depends on the \textit{number} of non-pure variables, and the \textit{sparsity} of the rows of $A$ corresponding to these non-pure variables. The rate indicates that prediction based on $\wh A$ will perform best when the number of pure variables is large, and any non-pure variable $X_i$, the $i$th component of $X$, only depends on   a small number of latent variables. We give, in the following section,   a simplified form of this bound, and compare this prediction scheme  with the other methods discussed in this work. 
        \item  The  identifiable factor model $X = AZ + W$, with $A$ satisfying Assumption \ref{ass_model}, has been used in \cite{LOVE} to construct overlapping clusters of the components on $X$. The latent factors can be viewed as random cluster centers, while a sparse matrix $A$ gives the cluster membership. From this perspective, and in light of the discussion leading up to the predictor construction, one  can view 	$\R(\wh A)$ as the risk of predicting $Y$ from predicted cluster centers, on the basis of data that exhibits a latent cluster structure with overlap. 
    \end{enumerate}
\end{remark}

	\subsection{Comparison of simplified prediction risks}\label{sec_compare}
	
	In this section we offer a comparison of the prediction risk of the predictors analyzed above. 
	For a transparent comparison, we compare them under an identifiable factor regression model. To this end, 
	we consider  the Essential Regression framework as a data generating mechanism under which we compare  PCR-$k$, with known $k = K$, the GLS predictor ($\wh B =\bI_p$), and the Essential Regression predictor ($\wh B = \wh A$), based on Corollary \ref{cor_PCR_delta_w}, Remark \ref{rem_K}, Corollary \ref{cor_GLS} and Theorem \ref{thm_pred_A}, respectively.  The notation $a_n \lessapprox b_n$ stands for $a_n =O(b_n)$ up to a multiplicative logarithmic factor in $n$ or $p$. 
	
	For ease of comparison, we consider the simplified setting in which $\lambda_K(A^\T A) \gtrsim p / K$,\footnote{This is met for instance when all $X$'s are pure variables and the numbers of pure variables for all groups are balanced in the sense that $|I_k| \asymp |I| / K$. Another instance such that $\lambda_K(A^\T A)\gtrsim p / K$ holds with high probability is that $|I_k| \asymp |I| / K$ and the rows of $A_J$ are i.i.d. realizations of a sub-Gaussian random vector whose second moment has operator norm bounded by $1/K$. The factor $1/K$ takes $(A0)$ in Assumption \ref{ass_model} into account.} $\|\beta\|_2 \le R_{\beta}$ and $r_e(\Gamma) \asymp p$, and focus on the high-dimensional regime where $p > Cn$ for a large enough constant $C>0$. We have 
	\begin{equation}\label{three_risks}
	\begin{aligned}
		\R(\U_K)-\sigma^2 ~&\lessapprox ~ {K\over n}\sigma^2 + {K\over p}\|\sw\|_\op R^2_{\beta} + {K\over n}\|\sw\|_\op R^2_{\beta} \\
		\R(\wh A)-\sigma^2  ~ &\lessapprox ~ {K\over n}\sigma^2+ {K\over p}\og R^2_{\beta}+ {K\|A_J\|_0 \over np}\og R^2_{\beta}\\
		\R(\bI_p)-\sigma^2   ~&\lessapprox ~  {K\over n}\sigma^2 +
		{n \over p}\sigma^2+  {K \over n} \|\sw\|_\op R^2_{\beta}
	\end{aligned}
	\end{equation}
	 Since the Essential Regression predictor is an instance of  model based prediction, we comment on when the two model agnostic predictors are competitive, under this particular model specification.

	We begin with a comparison between $\R(\U_K)$ and $\R(\wh A)$, and note that the difference in their respective errors bounds  depends on the sparsity of $A_J$. The risk bound on $\R(\U_K)$ is valid for any $\theta$ such that $(X,Y)\sim \textrm{sG-FRM}(\theta)$, and is in particular valid for $\theta$ satisfying the additional Essential Regression constraints. Our results show that while PCR-$K$ prediction is certainly a valid choice under this particular model set-up, it could be outperformed by the model tailored predictor.
	If each row of $A_J$ is sparse such that $\|A_J\|_0\asymp |J|$, then $\R(\wh A)$ has a faster rate. This advantage becomes considerable if $|J| = o(p)$, that is, in the presence of a growing number of pure variables.  
	However, if $A_J$ is not sparse such that $\|A_J\|_0 \asymp |J| K$, and $|J| \asymp p$, then $\R(\wh A)$ has  a slower rate than $\R(\U_K)$. Nevertheless, from a practical perspective, conditions on the sparsity of $A$ ($\|A_J\|_0 \asymp |J|$) simply mean that not all $p$ variables in the vector $X$ contribute to explaining a particular $Z_k$, for each $k$, which is the main premise of Essential Regression. Furthermore, in this risk bound comparison, $\R(\wh A)$  corresponds to $\wh A \in \RR^{p\times \wh K }$, for an appropriate, fully data dependent,  estimator $\wh K$ of the identifiable dimension $K$. In order to employ a fully data driven PCR prediction, corresponding to an estimated $K$, we would also need the delicate step of estimating it described in Section \ref{sec_pred_pcr} above. The risk bound above will then hold under conditions discussed in Remark \ref{rem_K}.
	
	Finally, the much simpler GLS interpolating predictor has a bound that compares favorably to the other agnostic predictor, PCR-$K$,  only  when $n / p$ is small enough, for instance, $p > n^2 / K$. This extra term $\sigma^2n/p$ in the bound for $\R(\bI_p)$ compared to the bound for PCR-$K$, is due to the additional variance induced by the usage the full data matrix $\X$, as opposed to the first $K$ principal components, which may already capture the majority of the signal.

	\section{Predictor selection  via data splitting}\label{sec_data_split}
	
	Whenever a factor regression model can be assumed to generate a given data set, but it is unclear what further model specifications are in place, one can, in principle, construct several predictors, some model agnostic and some tailored to prior beliefs. In this section we address the problem of choosing among a set of candidate predictors for a given data set  that is assumed to be generated by a factor regression model.  Suppose we have $M$ linear predictors with respective coefficients $\wh \alpha_1,\ldots,\wh \alpha_M$ that we want to choose from. For ease of presentation, in this section assume $n$ is divisible by $2$. Let $D_1$ be a subset of $[n]$ with $|D_1| = n/2$, and let $D_2 = [n]\setminus D_1$. Define 
	\begin{equation}\label{eqn:m hat}
	    \wh m \coloneqq \arg\min_{m\in [M]} 
	    \sum_{i\in D_2}(Y_i - X_i^\top \wh\alpha_m)^2,
	\end{equation}
	where for each $m\in [M]$, $\wh\alpha_m$ is trained on the data set $\{(X_i,Y_i):i\in D_1\}$ and is thus independent of $\{(X_i,Y_i):i\in D_2\}$. We then use $\wh\alpha \coloneqq \wh\alpha_{\wh m}$ as our predictor, for which we establish the following oracle inequality, which is an adaptation of Theorem 2.1 from \cite{wegkamp2003}  to factor regression models and unbounded linear predictors. Moreover, we provide  a  high-probability statement, as opposed to a bound on the expected risk as in \cite{wegkamp2003}. The proof is deferred to Appendix \ref{proof:split}.
	\begin{thm}\label{thm:split}
	Let $\wh\alpha\coloneqq \wh\alpha_{\wh m}$, where $\wh m$ is defined in (\ref{eqn:m hat}). Then for any $\theta = (K,A,\beta,\C,\Gamma, \sigma^2)$ such that $(X,Y)\sim \textrm{sG-FRM}(\theta)$, there exist absolute constants $c,c'>0$ and a constant $c_0 = c_0(\gamma_w,\gamma_z,\gamma_\eps)>0$ such that when $n > c\log(M)$ and for any $a>0$,
	\begin{align}\label{eqn_split}\nonumber
	   \PP_\theta\bigg\{ \R(\wh \alpha) -\sigma^2\le (1+a)^2 &\min_{m\in [M]} \{\R(\wh \alpha_m)-\sigma^2\}\\
	   &+ C(a)\left(\sigma^2\vee \max_{m\in [M]} \{\R(\wh \alpha_m)-\sigma^2\} \right )\frac{\log(nM)}{n}\bigg\}\ge 1-c'n^{-1},
	\end{align}
	where $C(a)=c_0 (1+a)^3/a$.
	\end{thm}

In the bound above, the worst excess risk $\max_m \{\R(\wh\alpha_m) - \sigma^2\}$ appears in the remainder term, which may appear unusual. Most model-selection oracle inequalities either are formulated as a bound on the empirical risk, or assume that the predictors are uniformly bounded, or both, and as a result do not contain a term of this form. The bound we give is for the prediction risk on new data, and for unbounded loss and predictors, since $\sup_{\alpha}(X^\top \alpha - y)^2 = \infty$. For the bound to be useful, it thus must be the case that none of the $M$ predictors has risk that grows too fast. In particular, if the risks of all $M$ predictors are bounded above in high probability, then the second term in (\ref{eqn_split}) will be $O(\log n/n)$ and thus typically subdominant. 

As an illustration, we can use this data-splitting procedure with $M=3$ and the three prediction methods discussed in Section \ref{sec_compare}. If the three excess risks in (\ref{three_risks}) are all $O(1)$,  which is met under the conditions discussed in detail in Section \ref{sec_compare}, then the bound (\ref{eqn_split})  becomes
\[\R(\wh \alpha) - \sigma^2 \lesssim (1+a)^2\min\bigg(\R(\U_K)-\sigma^2,\ \R(\wh A)-\sigma^2 , \  \R(\bI_p)-\sigma^2\bigg) + C(a)\sigma^2{\log n \over n}.\]
We further confirm the ability of the data-splitting approach to adapt to the best-case risk via simulations in Section \ref{sec_sims} below.


On a practical note, we remark that the splitting procedure can be repeated several times with random splits to obtain estimates $\wh \alpha^{(1)},\ldots,\wh\alpha^{(N)}$ that can be used to construct the average $N^{-1}\sum_{i=1}^N\wh \alpha^{(i)}$. This aggregate coefficient vector satisfies the same risk bound (\ref{eqn_split}) by convexity of the loss, while this approach in practice could alleviate some of the bias induced by the choice of split for the data.\\

	\section{Simulations}\label{sec_sims}
	
	In this section, we complement and support our theoretical findings with simulations, focusing on the prediction performance of candidate predictors under both the  generic factor regression model and  the Essential Regression framework.
	
	\paragraph{Candidate predictors:} We consider the following list of predictors:
	
	\begin{itemize}[leftmargin = 8mm]
		\setlength\itemsep{0mm}
		\item PCR-$\wt s$ with $\wt s$ obtained from (\ref{est_K}) with $\mu_n = 0.25(n+p)$;
		\item {PCR-$K$}: the principal component regression (PCR) predictor using the true $K$;
		\item {PCR-ratio}: PCR with $k$ selected via the criterion proposed in \cite{lam2012,Ahn-2013}; \footnote{We have also implemented the selection criterion  suggested by  \cite{Bai-Ng-K}, but it had inferior performance, and is for this reason not included in our comparison here.}
		\item{GLS}: the Generalized Least Squares predictor defined in (\ref{def_GLS});
		\item {ER-A}: the Essential Regression predictor with $\wh B = \wh A$ in (\ref{def_pred_B_intro});
		\item {Lasso}: implemented in {\textsf{glmnet}} with the tuning parameter chosen via cross-validation;
		\item {Ridge}: implemented in {\textsf{glmnet}} with the tuning parameter chosen via cross-validation;
		\item {MS}: the selected predictor from (\ref{eqn:m hat}) in Section \ref{sec_data_split}.
	\end{itemize}

	Both Lasso and Ridge are included for comparison. The Lasso is developed for predicting $Y$ from $X$ when we expect that the best predictor of $Y$ is well approximated by a sparse linear combination of the components of $X$. Under our model specifications, the best linear predictor of $Y$ from $X$ is given by 
	\[	X^{\top}  \alpha^* = X^{\top}     [ \Cov(X) ]^{-1}\Cov(X,Y) 
	= X^{\top}     \Gamma^{-1} A\left[ \C^{-1} + A^{\top}      \Gamma^{-1}A \right]^{-1}\beta,
	\]
	where the last step follows from the factor model (\ref{main_model}) and an application of the Woodbury matrix identity. Although $\alpha^*$ is not sparse in general, we observe that $\|\alpha^*\|_2^2 \le \beta^{\top}     [\C^{-1} + A^{\top}      \Gamma^{-1} A]^{-1}\beta$. Hence its $\ell_2$-norm may be small if  $\|\Gamma\|_{\op}\beta^\T(A^\T A)^{-1}\beta$ is small. Our simulation design allows for these possibilities.  
	
	\paragraph{Data generating mechanism:} We first describe how we generate $\C$, $\Gamma$, and $\beta$. To generate $\C$, we set $\textrm{diag}(\C)$ to a $K$-length sequence from 2.5 to 3 with equal increments. The off-diagonal elements of $\C$ are then chosen as 
	$[\C]_{ij} = (-1)^{(i+j)}([\C]_{ii} \wedge [\C]_{jj})(0.3)^{|i-j|}$ for all $i\ne j \in [K]$. Finally, $\Gamma$ is chosen as a diagonal matrix with diagonal elements sampled from $\textrm{Unif}(1, 3)$, and $\beta$ is generated with entries sampled from Unif$(0,3)$.
	
	Generating $A$ depends on the modeling assumption. Under the factor regression model, we sample each entry of $A$ independently from $N(0, 1/\sqrt{K})$. Under the Essential Regression setting, recall that $A$ can be partitioned into $A_I$ and $A_J$ which satisfy Assumption \ref{ass_model}. To generate $A_I$, we set $|I_k| = m$ for each $k \in [K]$ and choose $A_I = \bI_K \otimes {\bm 1}_m$, where $\otimes$ denotes the kronecker product. Each row $A_{j\sbt}$ of $A_J$ is generated by first randomly selecting its support with cardinality $s_j$ drawn from $\{2,3\ldots,\lfloor K/2 \rfloor\}$ and then by sampling its non-zero entries from $\textrm{Unif}(0,1/s_j)$ with random signs. In the end, we rescale $A_J$ such that the $\ell_1$ norm of each row is no greater than 1.

	Finally, we generate the $n\times K$ matrix $\Z$ and the $n\times p$ noise matrix $\W$ whose rows are i.i.d.\ from $N_K(0, \C)$ and $N_p(0, \Gamma)$, respectively. We then set $\X = \Z A^{\top}      + \W$ and $\Y = \Z\beta + \Eps$ where the $n$ components of $\Eps$ are i.i.d.~$N(0, 1)$. 
	
	For each setting, we generating 100 repetitions of $(\X, \Y)$ and record their corresponding results. The performance metric is based on the new data prediction risk. To calculate it, we independently generate a new dataset $( \X_{new}, \Y_{new})$ containing $n$ i.i.d.\ samples drawn according to our data generating mechanism. The prediction risk of the predictor $\wh \Y_{new}$ is calculated as $\|\wh \Y_{new} - \Z_{new}\beta\|^2/n$.

	\subsection{Prediction under the factor regression model}
	
	We compare the performance of PCR-$\wt s$, PCR-$K$, PCR-ratio, GLS, Lasso, Ridge and MS by varying $p$, $K$ and the signal-to-noise ratio (SNR) $\xi$ defined in (\ref{def_snr}), one at a time. The MS predictor is based on (\ref{eqn:m hat}) over all the aforementioned methods. 
	
	We first set $n = 300$, $K = 5$ and vary $p$ from $\{100, 300, 700, 1500, 3000, 5000\}$, then choose $n = 300$, $p = 500$ and vary $K$ from $\{3, 5, 10, 15, 20\}$. The prediction risks of different predictors for these two settings are shown in Figure \ref{fig_FR_pK}. Since both PCR-$\wt s$ and PCR-ratio consistently select the true $K$, we only present the result for PCR-$K$. 
	
		\begin{figure}[ht]
		\centering
		\begin{tabular}{ccc}
			\hspace{-4mm}
			\includegraphics[width=.4\textwidth]{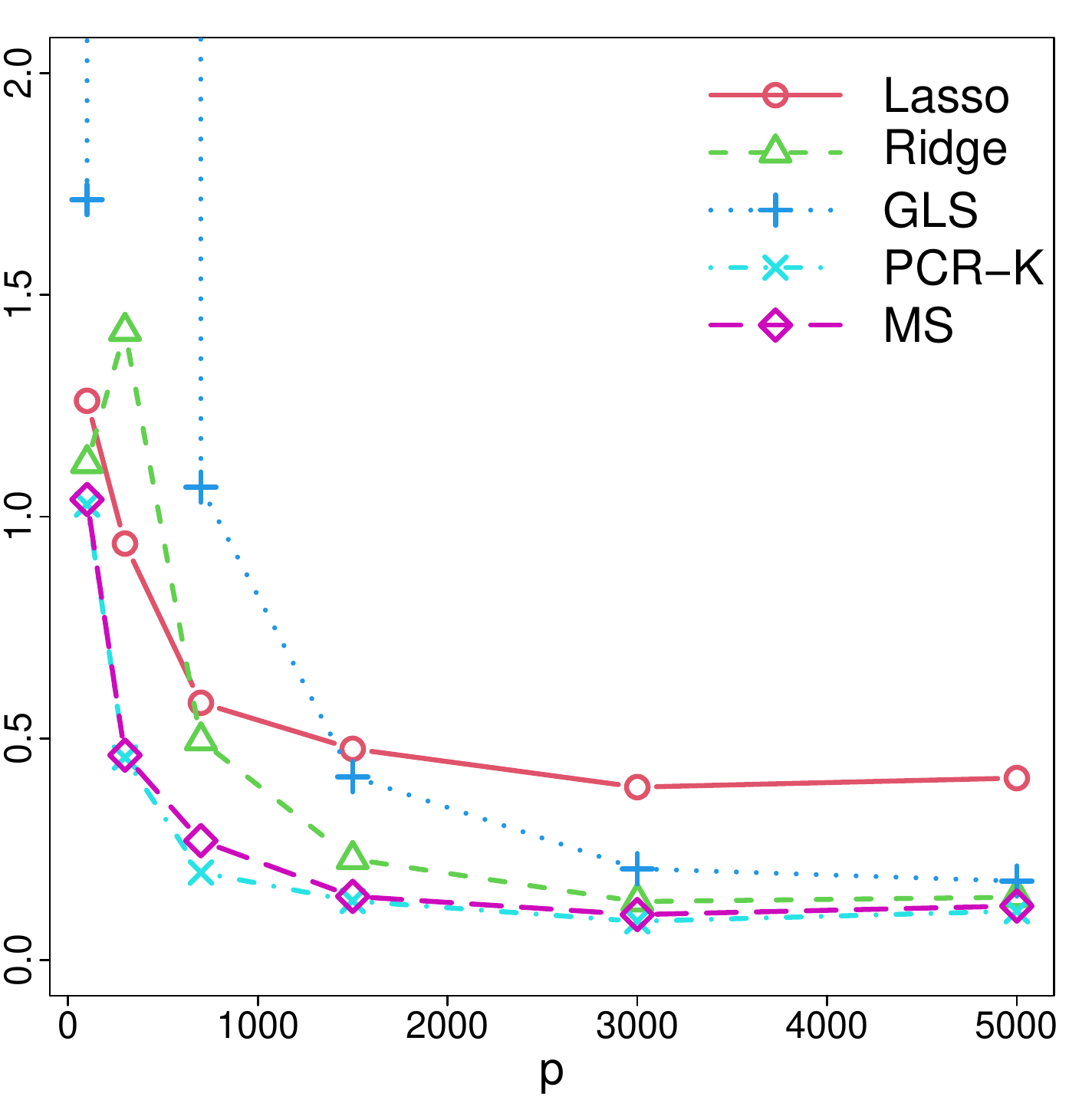}& \includegraphics[width=.4\textwidth]{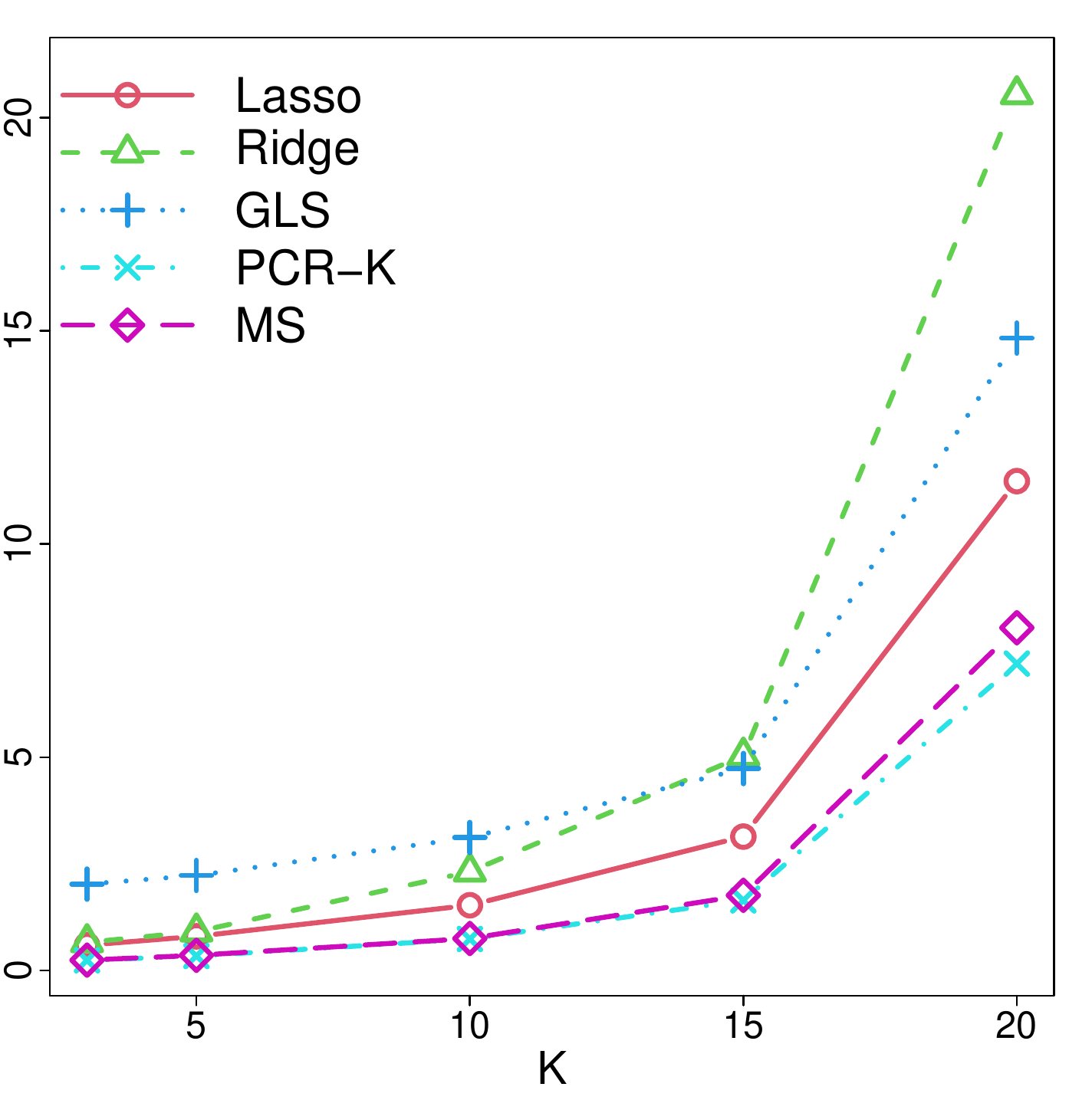}
		\end{tabular}
		\caption{Prediction risks of different predictors under the factor regression model as $p$ and $K$ vary separately}
		\label{fig_FR_pK}
	\end{figure}
	
	\paragraph{Results:}
	
	Overall, it is clear that the MS predictor selects the best predictor in almost all settings, corroborating Theorem \ref{thm:split}. Meanwhile, PCR-$K$ has the best performance in all settings as it is tailored to the factor regression model. 
	
	From the first panel, all methods perform better as $p$ increases (with exceptions given to GLS and Ridge when $p\approx n=300$). This contradicts the classical understanding that having more features increases the degrees of freedom of the model, hence inducing larger variance. By contrast, in our setting, increasing the number of features provides information that can be used to predict $A$.  This can be seen from the minimal excess risk in Lemma \ref{lem_pop_X} by noting that $\lambda_K(A^\T A)$ increases as $p$ increases.
	This phenomenon has been observed in the classical factor (regression) model, see, for instance,  \cite{SW2002_JASA, Bai-factor-model-03, Bai-Ng-forecast,Bai-Ng-CI, fan2013large} and the references therein.
	
	Perhaps more interestingly, when $p$ is much larger than $n$, GLS and Ridge have performance similar to PCR-$K$. This demonstrates our conclusions in Section \ref{sec_compare} that GLS and PCR-$K$ are comparable when $p \gg n$. We also note from our simulation that Ridge tends to select near-zero regularization parameter when $p \gg n$, whence Ridge essentially reduces to GLS \citep{hastie2019surprises}. In contrast to GLS and Ridge, the performance of Lasso stops improving after $p>2500$. 
	When $p$ is moderately large (say $p < 1000$), GLS and Ridge have larger errors than PCR-$K$ and Lasso. In particular, if 
	$p$ is close to $n$, the error of GLS diverges, a phenomenon observed in \cite{hastie2019surprises}, for example, under the linear model.
	
	From the second panel, the  prediction error for all methods deteriorates as $K$ increases. This indicates that prediction becomes more difficult for large $K$, supporting our results in Sections \ref{sec_pred_pcr} and \ref{sec_pred_GLS_ER}. We also note that the performance of Ridge deteriorates faster than the other methods when $K$ grows.\\
	
	\medskip

	To further demonstrate how different predictors behave as the signal-to-noise ratio (SNR) changes, we multiply $A$ by a scalar $\alpha$ chosen within $\{0.1, 0.13, 0.16,\cdots,0.37, 0.40\}$. We set $n = 300$, $p = 500$ and $K =5$. For each $\alpha$, we calculate the SNR and plot the prediction risks of each predictor in Figure \ref{fig_FR_snr}.
	
	\begin{figure}[ht]
	    \centering
	    \includegraphics[width=.4\textwidth]{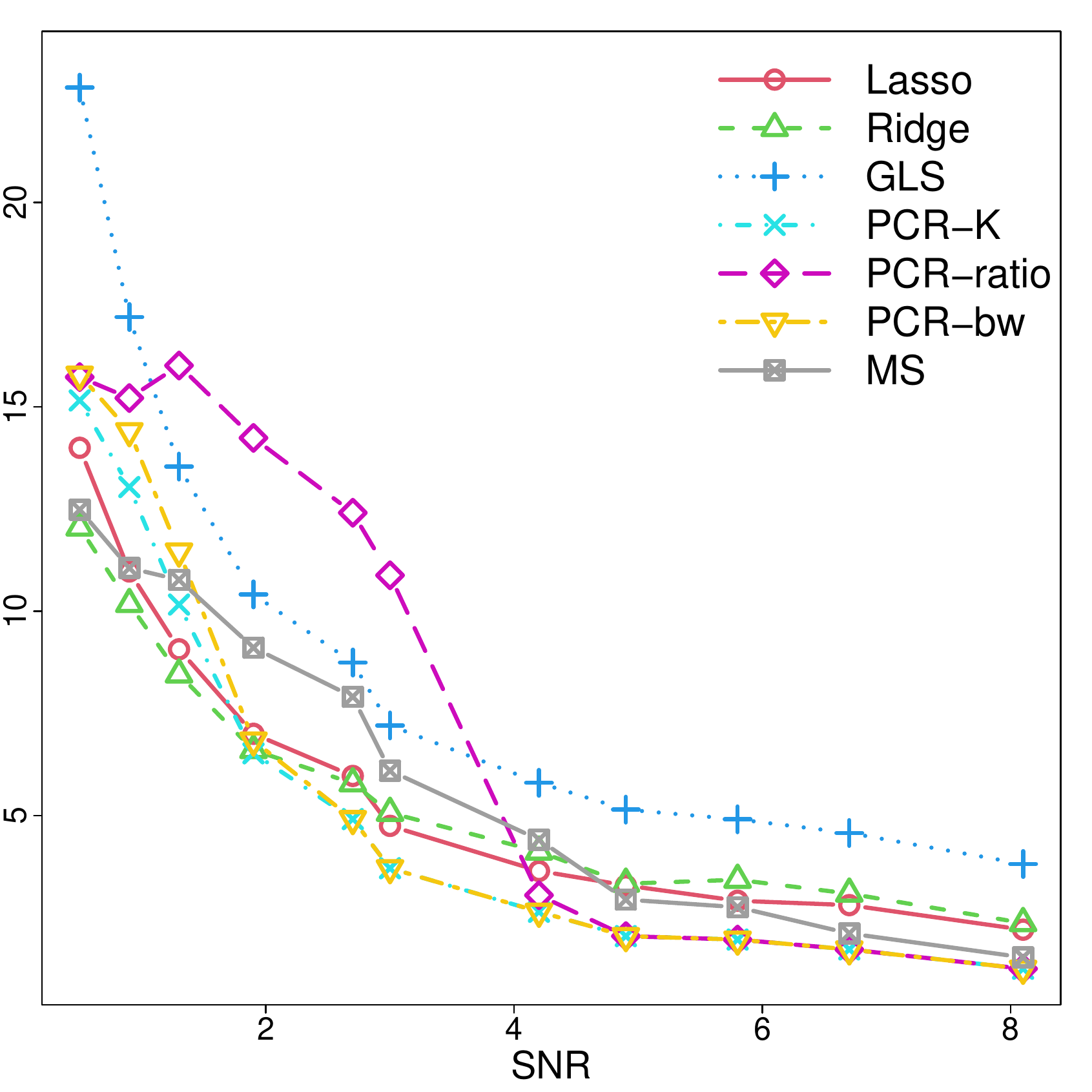}
	    \caption{Prediction risks of different predictors under the factor regression model as SNR varies}
	    \label{fig_FR_snr}
	\end{figure}
	
    \paragraph{Results:} 
    As expected, all methods perform worse as the SNR decreases. MS has consistently selected the (near) best predictor. When the SNR is small (less than 2), Ridge has the best performance. As soon as the SNR exceeds $2$,  PCR-K and PCR-$\wt s$ start to outperform the other methods. 
    In terms of selecting $K$, when the SNR is larger than $2$, PCR-$\wt s$ starts estimating $K$ consistently whereas PCR-ratio fails until the SNR is greater than $4$. Both PCR-$\wt s$ and PCR-ratio tend to under-estimate $K$ in the presence of a small SNR. However, PCR-$\wt s$ selects $\wt s$ closer to $K$ than PCR-ratio, leading to better performance. Moreover, the loss due to using $\wt s<K$ by PCR-$\wt s$ is not significant, in line with Corollary \ref{cor_PCR_s_tilde} and Remark \ref{rem_K}.

	\subsection{Prediction under the Essential Regression model}
	We compare all the predictors when data is generated from an  Essential Regression model. To vary $p$ and $K$ individually, we first set $n=300$, $K =5$, $m = 5$ and choose $p$ from $\{100, 300, 500, 700, 900\}$, then fix $n = 300$, $p=500$, $m = 5$ and vary $K$ in $\{3, 5, 10, 15, 20\}$. The prediction risks of different predictors are shown in Figure \ref{fig_ER_pK}. PCR-$\wt s$ and PCR-ratio are not included as they have almost the same performance as PCR-$K$. As it was demonstrated under the factor regression setting that GLS is outperformed by the other predictors when $p$ is not large enough, we also excluded its performance from the plot. 

	\paragraph{Summary:} We observe the same phenomenon as before, that is: (1) all predictors benefit from large $p$; (2) as $K$ increases, the performance of all predictors deteriorate. Furthermore, the model-based ER predictor has similar performance as the model-free PCR predictor when $K$ is small. The advantage of ER over PCR enlarges as $K$ grows. This is aligned with our theoretical findings in Section \ref{sec_compare} that ER benefits from the sparsity of $A_J$, because our data generating mechanism ensures that the larger $K$ is, the sparser $A_J$ becomes.\\

	\begin{figure}[ht]
		\centering
		\begin{tabular}{ccc}
			\includegraphics[width=.4\textwidth]{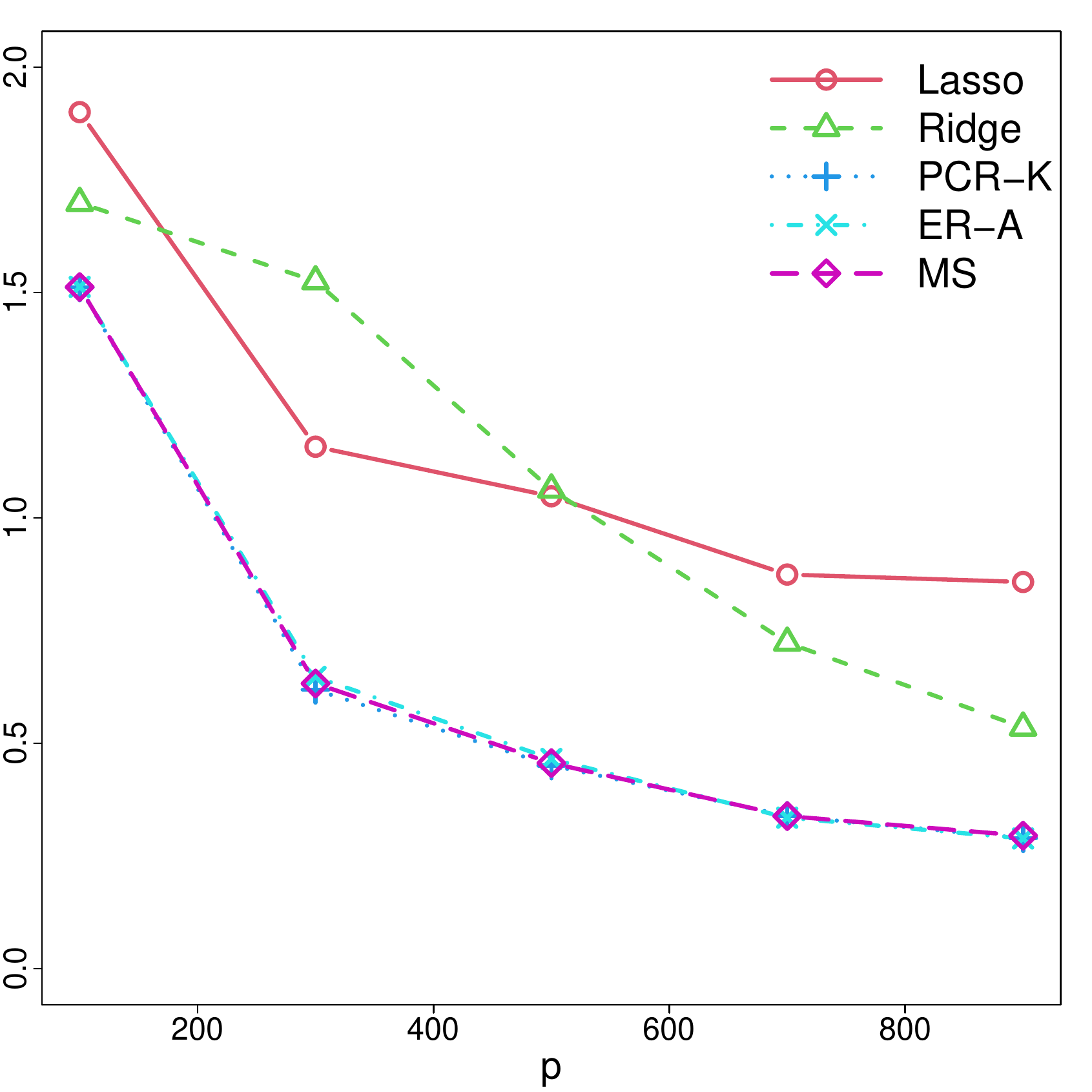}& \includegraphics[width=.4\textwidth]{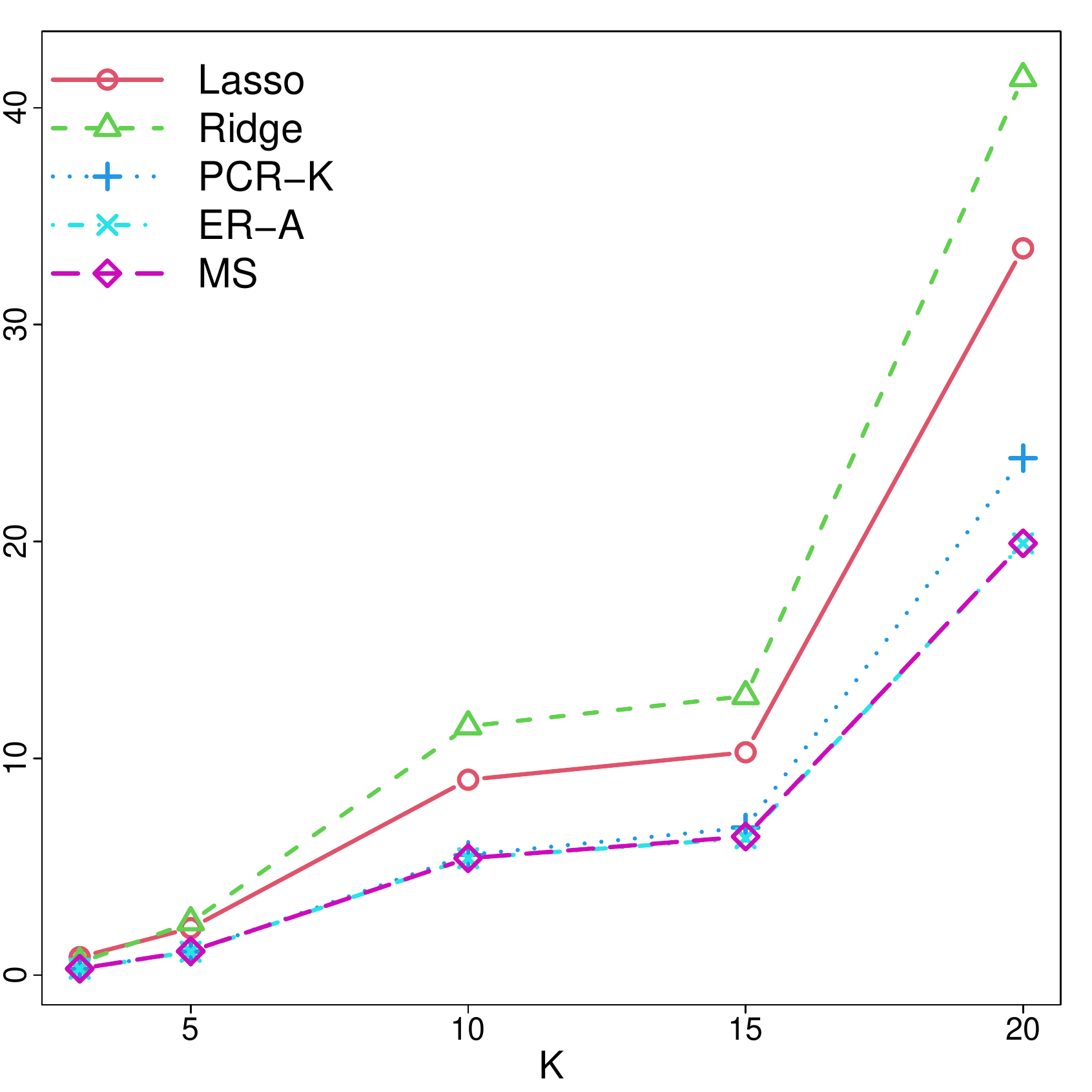}
		\end{tabular}
		\caption{Prediction risks of different predictors under the Essential Regression model as $p$ and $K$ vary separately}
		\label{fig_ER_pK}
	\end{figure}
	
	\section*{Acknowledgements}
	  Bunea and Wegkamp are supported in part by NSF grants DMS-1712709 and DMS-2015195.\\
	
	\bibliographystyle{plainnat}
    \bibliography{ref}

\begin{thebibliography}{33}
\providecommand{\natexlab}[1]{#1}
\providecommand{\url}[1]{\texttt{#1}}
\expandafter\ifx\csname urlstyle\endcsname\relax
  \providecommand{\doi}[1]{doi: #1}\else
  \providecommand{\doi}{doi: \begingroup \urlstyle{rm}\Url}\fi

\bibitem[Ahn and Horenstein(2013)]{Ahn-2013}
Seung~C. Ahn and Alex~R. Horenstein.
\newblock Eigenvalue ratio test for the number of factors.
\newblock \emph{Econometrica}, 81\penalty0 (3):\penalty0 1203--1227, 2013.

\bibitem[Bai(2003)]{Bai-factor-model-03}
Jushan Bai.
\newblock Inferential theory for factor models of large dimensions.
\newblock \emph{Econometrica}, 71\penalty0 (1):\penalty0 135--171, 2003.

\bibitem[Bai and Ng(2002)]{Bai-Ng-K}
Jushan Bai and Serena Ng.
\newblock Determining the number of factors in approximate factor models.
\newblock \emph{Econometrica}, 70\penalty0 (1):\penalty0 191--221, 2002.

\bibitem[Bai and Ng(2006)]{Bai-Ng-CI}
Jushan Bai and Serena Ng.
\newblock Confidence intervals for diffusion index forecasts and inference for
  factor-augmented regressions.
\newblock \emph{Econometrica}, 74\penalty0 (4):\penalty0 1133--1150, 2006.
\newblock \doi{10.1111/j.1468-0262.2006.00696.x}.

\bibitem[Bai and Ng(2008)]{Bai-Ng-forecast}
Jushan Bai and Serena Ng.
\newblock Forecasting economic time series using targeted predictors.
\newblock \emph{Journal of Econometrics}, 146\penalty0 (2):\penalty0 304 --
  317, 2008.
\newblock Honoring the research contributions of Charles R. Nelson.

\bibitem[Bair et~al.(2006)Bair, Hastie, Paul, and Tibshirani]{Bair_JASA}
Eric Bair, Trevor Hastie, Debashis Paul, and Robert Tibshirani.
\newblock Prediction by supervised principal components.
\newblock \emph{Journal of the American Statistical Association}, 101\penalty0
  (473):\penalty0 119--137, 2006.

\bibitem[Bartlett et~al.(2019)Bartlett, Long, Lugosi, and
  Tsigler]{bartlett2019}
Peter~L. Bartlett, Philip~M. Long, G{\'a}bor Lugosi, and Alexander Tsigler.
\newblock Benign overfitting in linear regression.
\newblock \emph{In arXiv:1906.11300}, 2019.

\bibitem[Belkin et~al.(2018{\natexlab{a}})Belkin, Hsu, and
  Mitra]{belkin2018overfitting}
Mikhail Belkin, Daniel Hsu, and Partha Mitra.
\newblock Overfitting or perfect fitting? risk bounds for classification and
  regression rules that interpolate.
\newblock \emph{In arXiv:1806.05161}, 2018{\natexlab{a}}.

\bibitem[Belkin et~al.(2018{\natexlab{b}})Belkin, Ma, and
  Mandal]{belkin2018understand}
Mikhail Belkin, Siyuan Ma, and Soumik Mandal.
\newblock To understand deep learning we need to understand kernel learning.
\newblock \emph{In arXiv:1802.01396}, 2018{\natexlab{b}}.

\bibitem[Belkin et~al.(2018{\natexlab{c}})Belkin, Rakhlin, and
  Tsybakov]{belkin2018does}
Mikhail Belkin, Alexander Rakhlin, and Alexandre~B. Tsybakov.
\newblock Does data interpolation contradict statistical optimality?
\newblock \emph{In arXiv:1806.09471}, 2018{\natexlab{c}}.

\bibitem[Belkin et~al.(2019{\natexlab{a}})Belkin, Hsu, Ma, and
  Mandal]{Belkin15849}
Mikhail Belkin, Daniel Hsu, Siyuan Ma, and Soumik Mandal.
\newblock Reconciling modern machine-learning practice and the classical
  bias{\textendash}variance trade-off.
\newblock \emph{Proceedings of the National Academy of Sciences}, 116\penalty0
  (32):\penalty0 15849--15854, 2019{\natexlab{a}}.
\newblock \doi{10.1073/pnas.1903070116}.

\bibitem[Belkin et~al.(2019{\natexlab{b}})Belkin, Hsu, and
  Xu]{belkin2019models}
Mikhail Belkin, Daniel Hsu, and Ji~Xu.
\newblock Two models of double descent for weak features.
\newblock \emph{In arXiv:1903.07571}, 2019{\natexlab{b}}.

\bibitem[Bing and Wegkamp(2019)]{rank19}
Xin Bing and Marten~H. Wegkamp.
\newblock Adaptive estimation of the rank of the coefficient matrix in
  high-dimensional multivariate response regression models.
\newblock \emph{Ann. Statist.}, 47\penalty0 (6):\penalty0 3157--3184, 12 2019.
\newblock \doi{10.1214/18-AOS1774}.
\newblock URL \url{https://doi.org/10.1214/18-AOS1774}.

\bibitem[Bing et~al.(2019)Bing, Bunea, and Wegkamp]{ER}
Xin Bing, Florentina Bunea, and Marten Wegkamp.
\newblock Inference in interpretable latent factor regression models.
\newblock \emph{In arXiv:1905.12696}, 2019.

\bibitem[Bing et~al.(2020)Bing, Bunea, Yang, and Wegkamp]{LOVE}
Xin Bing, Florentina Bunea, Ning Yang, and Marten Wegkamp.
\newblock Adaptive estimation in structured factor models with applications to
  overlapping clustering.
\newblock \emph{To appear in the Annals of Statistics}, 2020.

\bibitem[Bunea and Xiao(2015)]{bunea2015}
Florentina Bunea and Luo Xiao.
\newblock On the sample covariance matrix estimator of reduced effective rank
  population matrices, with applications to fpca.
\newblock \emph{Bernoulli}, 21\penalty0 (2):\penalty0 1200--1230, 05 2015.
\newblock \doi{10.3150/14-BEJ602}.
\newblock URL \url{https://doi.org/10.3150/14-BEJ602}.

\bibitem[Bunea et~al.(2020)Bunea, Strimas-Mackey, and
  Wegkamp]{bunea2020interpolation}
Florentina Bunea, Seth Strimas-Mackey, and Marten Wegkamp.
\newblock Interpolation under latent factor regression models.
\newblock \emph{In arXiv:2002.02525}, 2020.

\bibitem[Fan et~al.(2013)Fan, Liao, and Mincheva]{fan2013large}
Jianqing Fan, Yuan Liao, and Martina Mincheva.
\newblock Large covariance estimation by thresholding principal orthogonal
  complements.
\newblock \emph{Journal of the Royal Statistical Society: Series B (Statistical
  Methodology)}, 75\penalty0 (4):\penalty0 603--680, 2013.

\bibitem[Fan et~al.(2017)Fan, Xue, and Yao]{fan2017}
Jianqing Fan, Lingzhou Xue, and Jiawei Yao.
\newblock Sufficient forecasting using factor models.
\newblock \emph{Journal of Econometrics}, 201\penalty0 (2):\penalty0 292 --
  306, 2017.

\bibitem[Feldman(2019)]{feldman2019does}
Vitaly Feldman.
\newblock Does learning require memorization? \uppercase{A} short tale about a
  long tail.
\newblock \emph{arXiv:1906.05271}, 2019.

\bibitem[Hastie et~al.(2019)Hastie, Montanari, Rosset, and
  Tibshirani]{hastie2019surprises}
Trevor Hastie, Andrea Montanari, Saharon Rosset, and Ryan~J. Tibshirani.
\newblock Surprises in high-dimensional ridgeless least squares interpolation.
\newblock \emph{In arXiv:1903.08560}, 2019.

\bibitem[Hotelling(1957)]{Hotelling}
Harold Hotelling.
\newblock The relations of the newer multivariate statistical methods to factor
  analysis.
\newblock \emph{British Journal of Statistical Psychology}, 10\penalty0
  (2):\penalty0 69--79, 1957.

\bibitem[Hsu et~al.(2014)Hsu, Kakade, and Zhang]{Hsu2014}
Daniel Hsu, Sham~M. Kakade, and Tong Zhang.
\newblock Random design analysis of ridge regression.
\newblock \emph{Found. Comput. Math.}, 14\penalty0 (3):\penalty0 569--600, June
  2014.
\newblock ISSN 1615-3375.
\newblock \doi{10.1007/s10208-014-9192-1}.

\bibitem[Kelly and Pruitt(2015)]{Kelly-2015}
Bryan Kelly and Seth Pruitt.
\newblock The three-pass regression filter: A new approach to forecasting using
  many predictors.
\newblock \emph{Journal of Econometrics}, 186\penalty0 (2):\penalty0 294 --
  316, 2015.
\newblock ISSN 0304-4076.
\newblock High Dimensional Problems in Econometrics.

\bibitem[Kendall(1957)]{Kendall}
Maurice~G. Kendall.
\newblock \emph{A course in multivariate analysis}.
\newblock Hafner Pub. Co., 1957.

\bibitem[Lam and Yao(2012)]{lam2012}
Clifford Lam and Qiwei Yao.
\newblock Factor modeling for high-dimensional time series: Inference for the
  number of factors.
\newblock \emph{Ann. Statist.}, 40\penalty0 (2):\penalty0 694--726, 04 2012.

\bibitem[Liang and Rakhlin(2018)]{liang2019}
Tengyuan Liang and Alexander Rakhlin.
\newblock Just interpolate: Kernel "ridgeless" regression can generalize.
\newblock \emph{In arXiv:1808.00387}, 2018.

\bibitem[Montanari et~al.(2019)Montanari, Ruan, Sohn, and
  Yan]{montanari2019generalization}
Andrea Montanari, Feng Ruan, Youngtak Sohn, and Jun Yan.
\newblock The generalization error of max-margin linear classifiers:
  High-dimensional asymptotics in the overparametrized regime.
\newblock \emph{In arXiv:1911.01544}, 2019.

\bibitem[Muthukumar et~al.(2019)Muthukumar, Vodrahalli, Subramanian, and
  Sahai]{muthukumar2019harmless}
Vidya Muthukumar, Kailas Vodrahalli, Vignesh Subramanian, and Anant Sahai.
\newblock Harmless interpolation of noisy data in regression.
\newblock \emph{In arXiv:1903.09139}, 2019.

\bibitem[Muthukumar et~al.(2020)Muthukumar, Narang, Subramanian, Belkin, Hsu,
  and Sahai]{muthukumar2020classification}
Vidya Muthukumar, Adhyyan Narang, Vignesh Subramanian, Mikhail Belkin, Daniel
  Hsu, and Anant Sahai.
\newblock Classification vs regression in overparameterized regimes: Does the
  loss function matter?
\newblock \emph{In arXiv:2005.08054}, 2020.

\bibitem[Stock and Watson(2002)]{SW2002_JASA}
James~H. Stock and Mark~W. Watson.
\newblock Forecasting using principal components from a large number of
  predictors.
\newblock \emph{Journal of the American Statistical Association}, 97\penalty0
  (460):\penalty0 1167--1179, 2002.
\newblock ISSN 01621459.

\bibitem[Vershynin(2012)]{vershynin_2012}
Roman Vershynin.
\newblock \emph{Introduction to the non-asymptotic analysis of random
  matrices}, pages 210 -- 268.
\newblock Cambridge University Press, 2012.

\bibitem[Wegkamp(2003)]{wegkamp2003}
Marten Wegkamp.
\newblock Model selection in nonparametric regression.
\newblock \emph{Ann. Statist.}, 31\penalty0 (1):\penalty0 252--273, 2003.

\end{thebibliography}

		\newpage
	
	\appendix
	
	\section*{Appendix}

    We provide section-by-section proofs for the main results in Appendices \ref{app_proof_pred} -- \ref{proof:split}.  Auxiliary lemmas are collected in Appendix \ref{sec_proof_aux}. Appendix \ref{app_love} contains the procedure of estimating $A$ under the Essential Regression framework while comparison with more existing literature on factor models is stated in Appendix \ref{app_literature}.

	\section{Main proofs} \label{sec:main proofs} 
    We start by giving an elementary lemma that proves $Y^*_{\wh B} = Y^*_{P_{\wh B}}$ for any $\wh B \in \RR^{p\times q}$. Recall that, for any matrix $M$, $M^+$ denotes its Moore-Penrose inverse and $P_M$ denotes the projection onto the column space of $M$.
	\begin{lemma}\label{lem_invariant}
	    Let $\wh B \in \RR^{p\times q}$ be any matrix. Then 
	    \[
	        \wh B(\X\wh B)^+ = P_{\wh B}(\X P_{\wh B})^+.
	    \]
	\end{lemma}
	\begin{proof}
	     Write the SVD of $\wh B$ as $\wh B = UDV^\T$ where $U\in \RR^{p\times r_0}$ and $V\in \RR^{q\times r_0}$ are orthonormal matrices with $r_0 = \rank(\wh B)$. We then have 
	    \begin{align*}
	        \wh B(\X\wh B)^+ &= 
	        \wh B \left( \wh B^{\T} \X^\T \X \wh B \right)^{+} \wh B^\T \X^\T\\
	        &=
	        UDV^\T \left(VDU^\T\X^\T \X UDV^\T\right)^{+}VDU^\T\X^\T\\ 
	        &\overset{{(i)}}{=} U(U^\T \X^\T\X U)^+U^\T \X^\T\\
	        & \overset{{(ii)}}{=} UU^\T(UU^\T \X^\T\X UU^\T)^+UU^\T \X^\T.
	    \end{align*}
	    The result then follows by noting that $P_{\wh B} = UU^\T$. Step {(i)} uses the fact that 
    	\[
    	\left(VDU^\T\X^\T \X UDV^\T\right)^{+} = VD^{-1}\left(U^\T\X^\T \X U\right)^+D^{-1}V^\T 
    	\]
    	which can be verified by the definition of Moore-Penrose inverse. Indeed, let $M = U^\T\X^\T \X U$, $N = VD M DV^\T$ and $\wt N = VD^{-1} M^+ D^{-1}V^\T$. We need to verify 
    	\[
    	        N\wt N N = N,\qquad \wt N N \wt N = \wt N.
    	\]
    	Straightforwardly, 
    	\[
    	    N\wt N N = VD M M^+M DV^\T = VD M DV^\T = N
    	\]
    	and similar arguments hold for $\wt N N \wt N = \wt N$. Step (ii) uses step (i) with $D=\bI_{r_0}$ and $V=U$
	\end{proof}

	\subsection{Proofs for Section \ref{sec_pred_fr}}\label{app_proof_pred}
 	\subsubsection*{Proof of Lemma \ref{lem_pop_X}}
	Let $\sx = \Cov(X)$, $\sxy = \Cov(X,Y)$. Since $\sw$ is invertible, $\lp(\sx) = \lp(\sza +\sw)\ge \lp(\sw)>0$ so $\sx$ is invertible. Thus, letting $\alpha^* = \sx^{-1}\sxy$,
	\begin{equation}\label{rstar}
	    \R^*-\sigma^2  = \EE[(X^\top\alpha^* -Z^\top\beta )^2].
	\end{equation}
	Using this expression, and the factor model structure $X=AZ+W$, $Y=Z^\T \beta +\eps$, the proof of Lemma 4 in \cite{bunea2020interpolation} uses the Woodbury matrix identity to simplify (\ref{rstar}), arriving at
	\[\R^* -\sigma^2 = \beta^\top (\sz^{-1} + A^\top\sw^{-1} A)^{-1}\beta.\]
	Letting $H= \sz^{1/2}A^\top \sw^{-1}A\sz^{1/2}$, we then have
	\begin{align*}
	    \R^{*}-\sigma^2  &= \beta^{\T} \sz^{1/2}( \bI_K + H)^{-1}\sz^{1/2}\beta\\
	    &= \beta^\T \sz^{1/2}H^{-1/2}(\bI_K + H^{-1})^{-1}H^{-1/2}\sz^{1/2}\beta.
	\end{align*}
	To obtain the upper bound on $\R^*$ we use
	\[\R^*-\sigma^2 = \beta^\T \sz^{1/2}H^{-1/2}(\bI_K + H^{-1})^{-1}H^{-1/2}\sz^{1/2}\beta \le \frac{\beta^\top \sz^{1/2}H^{-1}\sz^{1/2}\beta}{1+ \lk(H^{-1})}\le \beta^\top (A^\top \sw^{-1} A)^{-1}\beta,\]
	where we used $\sz^{1/2} H^{-1}\sz^{1/2} = (A^\top \sx^{-1} A)^{-1}$ in the last step.
	
	To find the lower bound we first observe that
	\[\R^*-\sigma^2 = \beta^\T \sz^{1/2}H^{-1/2}(\bI_K + H^{-1})^{-1}H^{-1/2}\sz^{1/2}\beta \ge\frac{\beta^\top \sz^{1/2}H^{-1}\sz^{1/2}\beta}{1+ \|H^{-1}\|_{\op}}=  \frac{\beta^\top (A^\top \sx^{-1} A)^{-1}\beta}{1+ \lk^{-1}(H)}.\]
	Furthermore,
	\[\lk(H) = \lk(\sz^{1/2}A^\top \sw^{-1}A\sz^{1/2})\ge \lk(\sza)/\og= \xi,\]
	so using this in the previous display,
	\[\R^*-\sigma^2  \ge  \frac{\beta^\top (A^\top \sx^{-1} A)^{-1}\beta}{1+ \xi^{-1}} = \frac{\xi}{1+\xi} \cdot \beta^\top (A^\top \sx^{-1} A)^{-1}\beta,\]
	as claimed. \qed
	\subsection*{Proof of Theorem \ref{thm_pred}}\label{sec_proof_thm_pred}

	Define $\wh \alpha_{\wh B} =  \wh B\left(\wh B^\T\X^\T\X\wh B\, \right)^+\wh B^\T\X^\T \Y$ and recall that 
	$
	\wh Y^*_{\wh B} = X_*^\T\wh \alpha_{\wh B}
	$
	from (\ref{def_pred_B_intro}). Pick any $\theta$ with $K\le (Cn/\log n) \wedge p$ such that $(X,Y)$ follows FRM($\theta$) where $C = C(\gamma_z)$ is some positive constant.
	By $X_* = AZ_*+W_*$ and $Y_* = Z_*^\top\beta + \eps_*$, and the independence of $Z_*$, $\eps_*$, and $W_*$, one has 
	\begin{align}\label{eq_decomp}\nonumber
	\R(\wh B) - \sigma^2 &=
	\EE_{(Z_*,W_*)}\left[
	\left( \wh Y^*_{\wh B} - Z_*^\T\beta \right)^2
	\right]  \\
	& = \EE_{Z_*}\left[
	\left(Z_*^\T A^\T \wh \alpha_{\wh B} - Z_*^\T\beta \right)^2
	\right]+ \EE_{W_*}\left[
	\left(W_*^\T \wh \alpha_{\wh B}\right)^2
	\right]\\\nonumber
	&= \left\|
	\C^{1/2}\left(A^\T \wh \alpha_{\wh B}-\beta \right)
	\right\|^2 + \left\|
	\Gamma^{1/2}\wh \alpha_{\wh B}
	\right\|^2\\
	&\le \left\|
	\C^{1/2}\left(A^\T \wh \alpha_{\wh B}-\beta \right)
	\right\|^2 + \|\Gamma\|_{\op} \left\|
	\wh \alpha_{\wh B}
	\right\|^2.
	\end{align}
	We define an event $\E^*$ in (\ref{def_event_star}) below, on which we bound the risk. Invoking Lemmas \ref{lem_t1}, \ref{lem_t2}  and using  $\beta^\T A^+\Gamma A^{+\T}\beta \le \beta^{\top}(A^\T A)^{-1}\beta \|\Gamma\|_{\op}$, we find that the stated bound holds on the event $\E^*$. Then, by Lemma \ref{lem_event}, $\PP(\E^*)\ge 1 - cn^{-1}$, which completes the proof. \qed

	\bigskip
	
	We state and prove three lemmas which are used in the proof of Theorem \ref{thm_pred}.
	Recall that 
	\[
	\wh r = \rank(\X P_{\wh B}),
	    \qquad 	\wh \psi = {1\over n}\sigma_1^2\left(\X P_{\wh B}^{\perp}\right), \qquad 	\wh \eta = {1\over n}\sigma_{\wh r}^2\left(\X P_{\wh B} \right).
	 \] 
	\begin{lemma}\label{lem_event}
	For any $\theta$ with $K\le  (Cn/\log n) \wedge p$ and some positive constant $C = C(\gamma_z)$ such that $(X,Y)$ follows FRM($\theta$), we have $\PP(\E^*)\ge 1-cn^{-1}$ for some absolute constant $c>0$, where we define the event
	\begin{equation}\label{def_event_star}
	\E^*\coloneqq  \E_{\Z} \cap \E_{\W}\cap \E_{\W}'\cap \E_M\cap\E_{M'}\cap \E_{\Z \beta}.
	\end{equation}
	Here, for some constants $c(\gamma_z)$ and $c'(\gamma_w)$ depending on $\gamma_z$ and $\gamma_w$, respectively, 
	\begin{align*}
	&\E_{\Z}\coloneqq\left\{
	\lambda_K\left(\Omega^{1/2}	{1\over n}\Z^{\top}     \Z\Omega^{1/2}\right) \ge c(\gamma_z)
	\right\},\\
	& \E_{\Z \beta} := \left\{
	{1\over n}\left\|P_{\X\wh B}^{\perp}\Z\beta\right\|^2\le 8\gamma_w^2\beta^\T A^+\Gamma A^{+\T}\beta + 2 \wh \psi \beta^\T (A^\T A)^{-1}\beta
	\right\},\\
	&\E_{\W}\coloneqq\left\{
	{1\over n}\left\|\W^\T\W\right\|_{\op} \le \errw
	\right\},\\
	&\E'_{\W} \coloneqq \left\{
	{1\over n}\left\|\W A^{+\T}\beta \right\|^2 \le   4\gamma_w^2 \beta^\T A^+\Gamma A^{+\T}\beta 
	\right\},\\
	&\E_{M}\coloneqq \left\{
		\Eps^\top M\Eps \le 2\gamma_{\eps}^2\sigma^2\Big[ 2\|M\|_{\op} \log n+ \tr(M)\Big]
		\right\},\\
	&\E_{M'}\coloneqq \left\{
		\Eps^\top M'\Eps \le 2\gamma_{\eps}^2\sigma^2\Big[ 2\|M'\|_{\op} \log n+ \tr(M')\Big]
		\right\},
	\end{align*}
	with $\Omega \coloneqq \C^{-1}$, $\errw$ defined in (\ref{def_ErrW}), and
	\begin{align*}
	   &M\coloneqq (\X\wh B)^{+\T} \wh B^\T \wh B (\X \wh B)^+,\\
	&M'\coloneqq (\X \wh B)^{+\T} \wh B^\T A\C A^\T \wh B  (\X \wh B)^+.
	\end{align*}
	\end{lemma}
	\begin{proof}
	
By an application of Theorem 5.39 of \cite{vershynin_2012} and $K\log n\le C(\gamma_z)n$, we find $\PP\{\E^c_{\Z}\}\lesssim n^{-c'K}$. From Lemma \ref{lem_op_norm} with $\G = \W\sw^{-1/2}$, $H = \sw$, and $\gamma = \gamma_w$, we find $\PP\{\E^c_{\W}\}\le e^{-n}$.

	We note that $\W A^{+\T}\beta$ has independent $\gamma_w\sqrt{\beta^\T A^+\Gamma A^{+\T}\beta}$ sub-Gaussian entries, so $\W A^{+\T}\beta$ is a $\gamma_w\sqrt{\beta^\T A^+\Gamma A^{+\T}\beta}$ sub-Gaussian random vector. 
	Applying Lemma \ref{lem_quad} with $\xi = \W A^{+\T}\beta$, $H = \bI_n$, $\gamma_{\xi}^2 = \gamma_w^2\beta^\T A^+\Gamma A^{+\T}\beta $ and choosing $t = \log n$ yield
	\begin{equation}\label{e_w}
	   \PP\{(\E_{\W}')^{c}\} = \PP\left\{
	{1\over n}\left\|\W A^{+\T}\beta \right\|^2 >  4\gamma_w^2 \beta^\T A^+\Gamma A^{+\T}\beta 
	\right\} \le n^{-1}.
	\end{equation}
	We prove $ \E_{\W}'\cap \E_{\Z\beta} = \E_{\W'}$ in Lemma \ref{lem_comp_Zbeta}. 
	By the independence of $\Eps$ and both $\X$ and $\wh B$, the matrix $M$ is independent of $\Eps$. Thus, by an application of Lemma \ref{lem_quad} with $\xi = \Eps$, $H = M$, $\gamma_{\xi} = \sigma\gamma_{\eps}$ and  $t = \log n$  gives $\PP\{\E_M^c|M\}\le n^{-1}$. Taking the expectation over $M$ then gives $\PP\{\E_M^c\}\le n^{-1}$. The same argument with $H=M'$ gives $\PP\{\E_{M'}^{c}\}\le n^{-1}$.  
	
	Combining results, we find
	\begin{align*}
	\PP\{\E^{*c}\} &\le \PP\left\{\E_{\Z}^c\right \}+\PP\left\{\E^c_{\W}\right\} + \PP\left\{
	(\E'_{\W})^c\right\} +\PP\{\E_M^c\}+\PP\{\E_{M'}^c\}\lesssim n^{-1}.
	\end{align*}
\end{proof}

	\begin{lemma}\label{lem_t1}
		Under conditions of Theorem \ref{thm_pred}, on the event $\E^*$ defined in (\ref{def_event_star}),
		\begin{align}\label{bd_theta}
		\|\wh \alpha_{\wh B}\|^2 \lesssim_{\theta} {(\wh r +\log n)\sigma^2\over n\wh \eta}+  \beta^\T(A^\T A)^{-1}\beta + \wh \eta^{-1} \left(\wh \psi  \beta^{\top}      (A^\T A)^{-1}\beta +  \beta^\T A^+\Gamma A^{+\T}\beta\right).
		\end{align}
	\end{lemma} 
	\begin{proof}
		Starting with the identity
		\begin{eqnarray}\label{identity}
		\wh \alpha_{\wh B} = \wh B (\X \wh B)^+\Y = \wh B (\X \wh B)^+(\Z\beta + \Eps),
		\end{eqnarray}
		with $(\X \wh B)^+ \coloneqq (\wh B \X^\T \X \wh B)^{+}\wh B^\T \X^\T$, we have 
		\[
		\|\wh \alpha_{\wh B}\|^2 \le 2\left\|
		\wh B (\X \wh B)^+ \Eps
		\right\|^2 + 2 \left\|
		\wh B (\X \wh B)^+\Z\beta
		\right\|^2.
		\]
		To bound the first term, notice that 
		\begin{align*}
		    		    \left\|
		\wh B (\X \wh B)^+ \Eps
		\right\|^2 &= \Eps^\T (\X\wh B)^{+\T} \wh B^\T \wh B (\X \wh B)^+\Eps \\
		&= \Eps^\T M \Eps\\
		&\le 2\gamma_{\eps}^2\sigma^2\Big[ 2\|M\|_{\op} \log n+ \tr(M)\Big],
		\end{align*}
        where the last step holds on $\E^*$ (in particular, on $\E_M\subset \E^*$).
		Observe that, on $\E^*$,
		\begin{align*}
		\tr(M)  &= \tr\left((\X\wh B)^{+\T} \wh B^\T \wh B (\X \wh B)^+\right)\\
		& \le \rank(\X \wh B) \cdot \|M\|_{\op}\\
		& = \wh r \|M\|_{\op}.
		\end{align*}
		Write the SVD of $\wh B$ as $\wh B =  U DV^\T$ where $U\in \RR^{p\times r_0}$ and $V \in \RR^{q\times  r_0}$ are orthogonal matrices with $r_0 = \rank(\wh B)$. Recalling that 
		$(\X\wh B)^+ = (\wh B^\T \X^\T \X \wh B)^+ \wh B\X^\T$, the following holds, on the event $\E^*$,
		\begin{align}\label{bd_bxb}\nonumber
		\|M\|_{\op}  & = \left\| (\X\wh B)^{+\T} \wh B^\T \wh B (\X \wh B)^+\right\|_{\op}\\\nonumber
		&\overset{(i)}{=} \left\| \wh B (\X \wh B)^+(\X\wh B)^{+\T} \wh B^\T \right\|_{\op}\\\nonumber
		& = \left\|  \wh B (\wh B^\T \X^\T \X \wh B)^+ \wh B\X^\T \X \wh B (\wh B^\T \X^\T \X \wh B)^+  \wh B^\T\right\|_{\op}\\\nonumber
		&= \left\|\wh B\left(\wh B^\T \X^\T \X \wh B\right)^+ \wh B^\T \right\|_{{\rm op}}\\\nonumber
		&=  \left\|U\left(U^\T\X^\T \X U\right)^+ U^\T \right\|_{{\rm op}}\\\nonumber
		&\overset{(ii)}{\le} \sigma_{\wh r}^{-2}\left(\X U\right)\\
		&\overset{(iii)}{=}  (n\wh \eta)^{-1}
		\end{align}
		where we used $\|FF^\T\|_{\op} = \|F^\T F\|_{\op}$ for any matrix $F$ in $(i)$, $\rank(\X U) = \rank(\X P_{\wh B}) = \wh r$ in $(ii)$  and 
		\[
		\sigma^2_{\wh r}(\X U) = \lambda_{\wh r}(\X U U^\T \X) = \lambda_{\wh r}(\X P_{\wh B}^2 \X) = \sigma_{\wh r} (\X P_{\wh B})
		\]
		in $(iii)$.
		This concludes, on the event $\E^*$,
		\begin{equation}\label{bd_bxbeps}
		\left\|
		\wh B (\X \wh B)^+ \Eps
		\right\|^2 \le {2\gamma_{\eps}^2\sigma^2 \over n\wh \eta}(\wh r +  2\log n).
		\end{equation}
		
		On the other hand, by $A^\T A^{+\T} = \bI_K$ and $\X = \Z A^\T + \W$, observe that
		\begin{align}\label{eq_term}\nonumber
		\wh B (\X \wh B)^+\Z&= \wh B (\X \wh B)^+\Z A^\T A^{+\T}\\\nonumber
		& = \wh B (\X \wh B)^+(\X - \W) A^{+\T}\\
		&= \wh B (\X \wh B)^+\X P_{\wh B}A^{+\T} +\wh B (\X \wh B)^+\X P^{\perp}_{\wh B}A^{+\T} - \wh B (\X \wh B)^+\W A^{+\T}.
		\end{align}
		By $P_{\wh B} = \wh B \wh B^+$ and the inequality $(a+b+c)^2\le 3a^2+3b^2+3c^2$, 
		\begin{align}\label{eq_bd_bxbzbeta}
		\left\|
		\wh B (\X \wh B)^+\Z\beta
		\right\|^2 &\le 3	\left\| \wh B (\X \wh B)^+\X \wh B \wh B^{+}A^{+\T}  \beta
		\right\|^2+3\left\|\wh B (\X \wh B)^+\X P^{\perp}_{\wh B}A^{+\T} \beta\right\|^2\\\nonumber
		&\quad  + 3\left\|\wh B (\X \wh B)^+\W A^{+\T}\beta\right\|^2\\\nonumber
		&\le 3	\left\| \wh B (\X \wh B)^+\X \wh B \wh B^{+}\right\|_{{\rm op}}^2\left\|A^{+\T}  \beta
		\right\|^2+3\left\|\wh B (\X \wh B)^+\right\|_{{\rm op}}^2\left\|\X P^{\perp}_{\wh B}\right\|_{\op}^2 \left\|A^{+\T} \beta\right\|^2\\\nonumber
		&\quad  + 3\left\|\wh B (\X \wh B)^+\right\|_{{\rm op}}^2\left\|\W A^{+\T}\beta\right\|^2.
		\end{align}
		Recalling $\wh B = UDV^\T$, on the event $\E^*$, the following observation
		\begin{align*}
		       \left\| \wh B (\X \wh B)^+\X \wh B \wh B^{+}\right\|_{{\rm op}} &= \left\|
		U\left( U^\T \X^\T \X U\right)^+  U^\T \X^\T \X  UU^\T 
		\right\|_{{\rm op}}\\
		&\le  \left\|\left( U^\T \X^\T \X U\right)^+  U^\T \X^\T \X  U 
		\right\|_{{\rm op}}\le 1,
		\end{align*}
		together with (\ref{bd_bxb}), concludes 
		\begin{align}\label{bd_bxbz}
		\left\|
		\wh B (\X \wh B)^+\Z\beta
		\right\|^2& \le 3 \beta^\T(A^\T A)^{-1}\beta + 3 \wh \eta^{-1} \left(\wh \psi  \beta^\T (A^\T A)^{-1}\beta+  4\gamma_{w}^2\beta^\T A^+\Gamma A^{+\T}\beta\right).
		\end{align}
		Collecting (\ref{bd_bxbeps}) -- (\ref{bd_bxbz}) concludes the proof. 
	\end{proof}
	
	\bigskip
	
	\begin{lemma}\label{lem_t2}
	    Under conditions of Theorem \ref{thm_pred}, on the event $\E^*$ defined in (\ref{def_event_star}),
		\begin{align*}
		\left\|
		\C^{1/2}\left(A^\T \wh \alpha_{\wh B}-\beta \right)
		\right\|^2 &\lesssim_{\theta}  \left(1+ {\errw \over \wh \eta}\right)\left( {K\wedge \wh r+\log n\over n}\sigma^2 +\beta^\T A^+\Gamma A^{+\T}\beta\right)\\
		&\quad + \left[\left(1+ {\errw \over \wh \eta}\right)\wh \psi + \errw\right]\beta^\T(A^\T A)^{-1}\beta.
		\end{align*}
	\end{lemma}
	\begin{proof}
		Use identity (\ref{identity}) and the inequality $(x+y)^2 \le 2x^2+2y^2$ to find
		\begin{align}\label{eq_Atheta}\nonumber
		&\left\|
		\C^{1/2}\left(A^\T \wh \alpha_{\wh B}-\beta \right)
		\right\|^2 \\
		& \le 2\left\|
		\C^{1/2}[A^\T \wh B (\X \wh B)^+\Z -\bI_K]\beta 
		\right\|^2 +2	\left\|
		\C^{1/2}A^\T \wh B(\X \wh B)^+\Eps
		\right\|^2 .
		\end{align}
		For the first term, since $\Z\in \RR^{n\times K}$ has $\rank(\Z) = K$ on the event $\E^*$, we have 
		\begin{align}\label{eq_ABXB}\nonumber
		A^\T \wh B (\X \wh B)^+ -\Z^+ &= \Z^+ \Z A^\T\wh B (\X \wh B)^+ -\Z^+
		&(\textrm{by $\Z^+\Z = \bI_K$ on $\E^*$})\\\nonumber
		&= \Z^+ (\X-\W)\wh B (\X \wh B)^+ -\Z^+\\
		&= - \Z^+P_{\X \wh B}^{\perp} - \Z^+\W\wh B(\X \wh B)^+,
		\end{align}
		which yields 
		\begin{align}\label{eq_bd_lem_t2}\nonumber
		&\left\|
		\C^{1/2}[A^\T \wh B (\X \wh B)^+\Z -\bI_K]\beta 
		\right\|^2 \\ \nonumber
		&\le 2\left\|
		\C^{1/2}\Z^+ P_{\X \wh B}^{\perp}\Z \beta 
		\right\|^2 +2\left\|
		\C^{1/2}\Z^+\W \wh B(\X \wh B)^+\Z \beta 
		\right\|^2 \\
		&\lesssim {1\over n}\left\|P_{\X \wh B}^{\perp}\Z \beta 
		\right\|^2+{1\over n}\left\|\W \wh B(\X \wh B)^+\Z \beta 
		\right\|^2 \\\nonumber
		&\lesssim {1\over n}\left\|P_{\X \wh B}^{\perp}\Z \beta 
		\right\|^2 +\errw\cdot \left\| \wh B(\X \wh B)^+\Z \beta 
		\right\|^2 .
		\end{align}
		We used $\|\C^{1/2}\Z^+\|_{{\rm op}} = \sigma_K^{-1}(\Z \Omega^{-1/2})\lesssim 1/\sqrt{n}$ on $\E^*$ in the third line. The event $\E_{\Z\beta}$ and (\ref{bd_bxbz}) conclude 
		\begin{align}\label{bd_abxbz_I}
		&\left\|
		\C^{1/2}[A^\T \wh B (\X \wh B)^+\Z -\bI_K]\beta 
		\right\|^2 \\ \nonumber
		&\qquad \lesssim    \left(1+ {\errw \over \wh \eta}\right)\left(\beta^\T A^+\Gamma A^{+\T}\beta+ \wh \psi  \beta^{\top}      (A^\T A)^{-1}\beta\right) + \errw\beta^\T(A^\T A)^{-1}\beta.
		\end{align}

		For the second term in (\ref{eq_Atheta}), we use that on $\E^*$ (in particular, $\E_{M'}\subset \E^*$),
		\[\left\|
		\C^{1/2}A^\T \wh B(\X \wh B)^+\Eps
		\right\|^2 \le 2\gamma_{\eps}^2\sigma^2\left[
		2\|M'\|_{\op} \log n + \tr(M')
		\right]\]
		Since $\rank(\C) = K$ and $\rank(\X \wh P_{\wh B}) = \wh r$, we have 
		\[
		\tr(M') \le (K\wedge \wh r\ )\|M'\|_{\op}.
		\]
		Moreover,
		\begin{align*}
		\|M'\|_{\op} = \left\|
		\C^{1/2}A^\T \wh B(\X \wh B)^+\right\|^2_{{\rm op}}
		&\le  2\left\|
		\C^{1/2}\Z^+P_{\X \wh B}\right\|_{{\rm op}}^2+ 2 \left\|
		\C^{1/2}\Z^+\W\wh B(\X \wh B)^+\right\|_{{\rm op}}^2 \\
		&\lesssim {1 \over n}+ \errw\cdot \left\|\wh B(\X \wh B)^+\right\|_{{\rm op}}^2 
		\end{align*} 
		by using (\ref{eq_ABXB}) in the first line and $\E^*$ in the second line. Invoking (\ref{bd_bxb}) concludes that, on $\E^*$,
		\begin{align}\label{bd_ABXBeps}
		\left\|
		\C^{1/2}A^\T \wh B(\X \wh B)^+\Eps
		\right\|^2  &\lesssim {(K\wedge \wh r+\log n) \sigma^2 \over n} \left(
		1 + {\errw \over \wh \eta}
		\right).
		\end{align}
		Plugging (\ref{bd_abxbz_I}) and (\ref{bd_ABXBeps}) into  (\ref{eq_Atheta}) completes the proof. 
	\end{proof}

	\medskip

	\begin{lemma}\label{lem_comp_Zbeta}
	Under conditions of Theorem \ref{thm_pred}, on the event  $\E_{\W}'$ from (\ref{def_event_star}),  
	 \begin{align}\label{bd_P_comp_Zbeta}
 	    {1\over n}\left\|P_{\X\wh B}^{\perp}\Z\beta\right\|^2\le 8\gamma_w^2 \beta^\T A^+\Gamma A^{+\T}\beta + 2 \wh \psi \beta^\T (A^\T A)^{-1}\beta.
 	\end{align}
	\end{lemma}
	\begin{proof}
	By $\X = \Z A^\T + \W$, one has 
 	\begin{eqnarray*}
 		P_{\X\wh B}^{\perp}\Z\beta &=&	P_{\X\wh B}^{\perp} \left ( \X A^{+\T} - \W A^{+\T}\right )\beta  \\
 		& =&  -P_{\X\wh B}^{\perp}\W A^{+\T}\beta + 	P_{\X\wh B}^{\perp}\X A^{+\T}\beta \\
 		&=& -  P_{\X\wh B}^{\perp}\W A^{+\T}\beta +
 		P_{\X\wh B}^{\perp}\X\left(A^{+\T} - \wh B G\right)\beta
 	\end{eqnarray*}
 	for any matrix $G \in \RR^{q\times K}$. 
 	Choose 
 	$$
 	G = \wh B^+ A^{+\T} = \min_{G'} \left\|A^{+\T} - \wh BG'\right\|_F
 	$$ 
 	to obtain
 	\begin{align*}
 	P_{\X\wh B}^{\perp}\Z\beta &= P_{\X\wh B}^{\perp}\W A^{+\T}\beta +
 	P_{\X\wh B}^{\perp}\X P_{\wh B}^{\perp}A^{+\T} \beta.
 	\end{align*}
 	Then by the basic inequality $(a+b)^2\le 2a^2+2b^2$, 
 	\begin{align}\label{eq_bd_Zbeta}
 	\left\|P_{\X\wh B}^{\perp}\Z\beta\right\|^2&\le 2
 	\left\|P_{\X\wh B}^{\perp}\W A^{+\T}\beta \right\|^2+ 2
 	\left\|P_{\X\wh B}^{\perp}\X P_{\wh B}^{\perp} A^{+\T} \beta\right\|^2\\\nonumber
 	&\le 2
 	\left\|P_{\X\wh B}^{\perp}\right\|_{{\rm op}}^2 
 	\left\|\W A^{+\T}\beta \right\|^2+ 2
 	\left\|\X P_{\wh B}^{\perp}\right\|_{\op}^2\left\|A^{+\T} \beta\right\|^2\\\nonumber
 	&\le 2
 	\left\|\W A^{+\T}\beta \right\|^2+ 2 n \wh \psi \beta^\T (A^\T A)^{-1}\beta
 	\end{align}
 	where we invoked the definition of $\wh \psi$ in the last line. Invoke $\E_{\W}'$ from (\ref{def_event_star}) to finish the proof.
	\end{proof}
	
	\subsection{Proofs for Section \ref{sec_pred_pcr}}\label{app_proof_PCR}
	\subsubsection*{Proof of Corollary \ref{cor_PCR_k}}
	{The corollary is an application of Theorem \ref{thm_pred} with $\wh B=\U_k$.
	Given any realization of $(\X,\Y)$ and (possibly random) $k \in \{0,1, \ldots, \rank(\X)\}$, 
	we may write  the SVD of $\X$ as
	\begin{align*}
	\X = \V \D \U^\T &~= ~  \sum_{1\le j\le k} \D_{jj} \V_{\sbt j} \U_{\sbt j}^\T  + \sum_{j>k} \D_{jj} \V_{\sbt j} \U_{\sbt j}^\T\\
	& ~ \coloneqq ~ \V_k \D_k \U_k^\T + \V_{(-k)}\D_{(-k)}\U_{(-k)}^\T.
	\end{align*}
	The diagonal matrix $\D$ contains the non-increasing singular values and $\U_k$ contains the corresponding  $k$ right-singular vectors.
	Consequently, 
	\begin{align*}
	&\rank(\X \U_{k}) = \rank(\V_k \D_{k})=k,\\
	&\sigma_1^2\left(\X P_{\U_{k}}^\perp\right)  = \left\|\X \U_{(-k\ )} \U_{(-k)}^\T \right\|_{\op}^2 = \left\|\V_{(-k\ )}\D_{(-k)}\U_{(-k\ )}^\T\right\|_{{\rm op}}^2= \sigma_{k+1}^2\left(\X\right) = n\wh \lambda_{k+1},\\
	 &  \sigma_1^2\left(\X P_{\U_{k}}\right)  = \sigma_1^2\left(\V_k\D_k\U_k^\T\right) =\sigma_k^2(\X)= n \wh \lambda_k.
		\end{align*}
	Invoke Theorem \ref{thm_pred}
	with $\wh B=\U_k$, $\wh r=k$,
	$\wh \psi = \wh\lambda_{k+1}$ and $\wh \eta = \wh\lambda_k$
		to conclude the proof.\qed
		}
%
%
%

	\subsubsection*{Proof of Corollary \ref{cor_PCR_delta_w} \& Remark \ref{rem_K}}
	We first prove Corollary \ref{cor_PCR_delta_w}. 
	From Corollary \ref{cor_PCR_k}, it suffices to show $\PP_\theta\{\wh s \le K\}\ge 1-c/n$, which is guaranteed by proving
	\[
	    \PP_\theta \left\{{1\over n}\sigma^2_{K+1}(\X) < C_0\errw\right\} \ge 1-c/n.
	\]
	By Weyl's inequality,
	\[
	    \sigma_{K+1}(\X) \le \sigma_{K+1}(\Z A^\T) + \sigma_{1}(\W) = \sigma_{1}(\W).
	\]
	The result then follows by (\ref{bd_W_op}) and $C_0 > 1$. \qed 

    To prove Remark \ref{rem_K}, we will show
	\[
	    \PP\left\{\wh \lambda_K \gtrsim \lambda_k(A\C A^\T) - \errw\right\} \ge 1-n^{-c}.
	\]
	Note that Weyl's inequality yields 
	\[
    	\sigma_k(\X)  \ge \sigma_{k}(\Z A^\T) - \sigma_1(\W) \ge 
	    \sigma_K(\Z \C^{-1/2})\sigma_{k}(\C^{1/2} A^\T) - \sigma_1(\W).
	\]
	We obtain the desired result by invoking $\E_{\Z}$ from Lemma \ref{lem_event} and (\ref{bd_W_op}).
	\qed

	 
	 \subsubsection*{Proof of Proposition \ref{prop_K}}
	 
	 		    We work on the event 
	\[
	\E_{\W}'' := \left\{
	\sigma_1^2(\W) \le  n\delta_W
	\right\} \cap \left\{
	c_1\ \tr(\Gamma) \le {1\over n}\|\W\|_F^2 \le C_1\ \tr(\Gamma)
	\right\}
	\]
	with $\errw$ defined in (\ref{def_ErrW})  and some constants $C_1\ge c_1>0$, depending on $\gamma_w$.
	 We have on the event   $\E_{\W}''$,
	\begin{align}\label{bd_mu_n}\nonumber
	2\sigma_1^2(\W)  {np \over \|\W\|_F^2 } &\le 2n\errw  {np \over \|\W\|_F^2 }\\\nonumber
	&\le {2 \errw \over c_1} {np \over \tr(\Gamma)} & \textrm{by }\E_{\W}''\\ \nonumber
	&= {2c\over c_1}\left( {np\over r_e(\Gamma)} + p\right) & \textrm{by }(\ref{def_ErrW})\\\nonumber
	&\le {2c\over c_1}\left( {n\vee p\over c'} + p\right) & \textrm{by }r_e(\sw) \ge c'(n\wedge p)\\\nonumber
	&\le c_0(n+p) = \mu_n
	\end{align}
	by choosing any $c_0 \ge 2c(1+1/c')/c_1$.
	From Theorem 6 and Proposition 7 of  \cite{rank19}  with $P = \bI_n$, $E = \W$ and $m = p$, we deduce 
	\[
	\wt s \le K
	\]
	on the event $\E_{\W}''$. \\
	
To prove the lower bound $\sigma^2_{\wt s}(\X) \gtrsim n\delta_W$, we notice that, on the event $\E_{\W}''$,
	\begin{equation}\label{lbd_sigma_X}
	\sigma_{\wt s}^2(\X) \ge \mu_n {\|\X - \X_{(\wt s)}\|_F^2 \over np - \mu_n \wt s} \ge \mu_n {\|\X - \X_{(K)}\|_F^2 \over {np}}.
	\end{equation}
 	The first inequality uses (2.7)
	 in \cite{rank19}, while the second inequality uses 
	  $K\le \bar K$. 
	Further invoking (3.8) in Proposition 7 of \cite{rank19} 
	yields 
	\[
	{\|\X - \X_{(K)}\|_F^2
		\over np-\mu_n K}\ge {\|\W\|_F^2 \over np}.
	\]
  Next, on the event $\E_{\W}''$, choosing $c_0 \ge 2c(1+1/c')/c_1$ in $\mu_n = c_0(n+p)$, we find
\begin{align*}
\mu_n{\|\W\|_F^2 \over np} &\ge \mu_n c_1  {\tr(\sw) \over p}\\
&\ge 2c\left(1 + {1\over c'}\right)  {n+p \over p}\tr(\sw)\\
&\ge 2c\left(\tr(\sw) + {1\over c'} {n+p\over p} r_e(\sw)\|\sw\|_{\op}\right)\\
&\ge 2c\left(\tr(\sw) +(n\wedge p) {n+p\over p} \|\sw\|_{\op}\right) &\textrm{ by  $r_e(\sw) \ge c'(n\wedge p)$}\\
&\ge  2c\left(\tr(\sw) + n\|\sw\|_{\op}\right)\\
&= 2 n\delta_W.
\end{align*}
Hence, combining all three previous displays, we derive
\begin{align*}
\sigma_{\wt s}^2(\X) &\ge	\mu_n {\|\X - \X_{(K)}\|_F^2 \over np}\\
&~ \ge~ \mu_n {\|\W\|_F^2 \over np} {np-\mu_n K \over np}\\
	&~ \ge ~   n\delta_W {np-\mu_n K \over np}\\
	&~ \ge ~ {1 \over 1+ \kappa}n\errw &\textrm{by $K\le \bar K$ and  (\ref{est_K})}.
	\end{align*}
Next, we prove  	$\sigma_{\wt s+1}^2(\X) \lesssim \delta_W$.
By (2.7) in \cite{rank19} once again, we have  
	\begin{align*}
	\sigma_{\wt s+1}^2(\X) &\le \mu_n {\|\X - \X_{(\wt s+1)}\|_F^2 \over np - \mu_n (\wt s+1)}.
	\end{align*}
	From (2.3) in Proposition 1 of \cite{rank19}, this inequality is equivalent to 
	\[
	\sigma_{\wt s+1}^2(\X) \le \mu_n {\|\X - \X_{(\wt s)}\|_F^2 \over np - \mu_n \wt s }.
	\]
	Since  $\wt s \le K$ on $\E_{\W}''$, we have 
	\begin{align*}
	\sigma_{\wt s+1}^2(\X) &~ \le ~ \mu_n {\|\X - \X_{(K)}\|_F^2 \over np - \mu_n K}\\
	&~ {\le} ~ \mu_n {np \over np - \mu_n K} {\|\W\|_F^2 \over np}
	&\textrm{by  (3.8) of Proposition 7 in \cite{rank19}}\\
	& ~ \le ~ (1+\kappa)\mu_n  {\|\W\|_F^2 \over np} &\textrm{by (\ref{est_K})}\\
	& ~ \le ~ (1+\kappa)c_0C_1(n+p) {\tr(\Gamma) \over p} & \textrm{by $\E_{\W}''$ and $\mu_n = c_0(n+p)$}\\
	& ~ \le ~ { (1+\kappa)c_0C_1 \over c}n\errw & \textrm{by } \tr(\Gamma) \le p\|\Gamma\|_{\op}.
	\end{align*}
	It remains to prove $1- \PP(\E_{\W}'') \lesssim 1/n$. First note that 
	\[
	{1\over n}\|\W\|_F^2 = \sum_{j=1}^p {1\over n}\W_{\sbt j}^\T \W_{\sbt j}.
	\]
	By invoking Lemma \ref{lem_bernstein} for fixed $j\in [p]$ and some absolute constant $c$,   the inequality \[
	\left|{1\over n}\W_{\sbt j}^\T \W_{\sbt j} - [\Gamma]_{jj}\right| \le c \gamma_w^2 [\Gamma]_{jj} \sqrt{\log p \over n}
	\] holds
	with probability at least  $1-2(p\vee n)^{-2}.$ Apply the union bound over $1\le j\le p$,    invoke $\log p \le C n$ for sufficiently large $C$, and  conclude 
	\[
	\PP\left\{c(\gamma_w)\ \tr(\Gamma) \le {1\over n}\|\W\|_F^2 \le C(\gamma_w)\ \tr(\Gamma)\right\}\ge 1-2(p\vee n)^{-1}.
	\]
Finally, Lemma \ref{lem_op_norm} shows that $\PP\{\sigma_1^2(\W) \le n\delta_W\}\ge 1- e^{-n}$, taking $c$ in $\delta_W$ large enough.   \qed

	\bigskip

	\subsection{Proofs for Section \ref{sec_pred_GLS_ER}}\label{app_proof_GLS_ER}
	
	\subsubsection*{Proof of Corollary \ref{cor_gls_low}}
	By Theorem 5.39 of \cite{vershynin_2012}, $\sigma_p^2(\X\sx^{-1/2})\gtrsim n$ with probability at least $1-cn^{-1}$, where we use that $\X\sx^{-1/2}$ has independent sub-Gaussian rows with sub-Gaussian constant bounded by an absolute constant, which is implied by the sub-Gaussianity of $Z$ and $W$, and that $p\log n\lesssim n$. Thus, with the same probability,
	\[\sigma_p^2(\X)\ge \lp(\sx)\sigma_p^2(\X\sx^{-1/2})\ge \lp(\sw)\sigma_p^2(\X\sx^{-1/2})\gtrsim \lp(\sw)n.\]
	Corollary \ref{cor_gls_low} then follows from Theorem \ref{thm_pred} with $\wh \psi=0$, $\wh \eta \gtrsim \lp(\sw)$, and $\wh r\le p$.
	\qed 
	
	\subsubsection*{Proof of Corollary \ref{cor_GLS}}\label{sec_proof_cor_GLS_high}
			Under conditions of Corollary \ref{cor_GLS},  \cite{bunea2020interpolation} proves that 
			\[
			\PP\left\{
			\sigma_n^2(\X) \gtrsim \tr(\Sigma_W)
			\right\} \ge 1-cn^{-1}.
			\]
		We thus have $r=n$, $\wh \psi=0$, and $\wh \eta \gtrsim \tr(\sw)/n$. Further noting that
			$$
		 \errw = \|\Gamma\|_{{\rm op}}\left(
			1 + {r_e(\Gamma)  \over n}
			\right) \asymp   {\tr(\Gamma) \over n},
			$$
			such that $\errw / \wh \eta \asymp 1$, we conclude
			\begin{align*}
			\R^*(\bI_p)	-\sigma^2&\lesssim  {K + \log n\over n}\sigma^2 + 
			{n \over r_e(\Gamma)}\sigma^2 +   {\tr(\Gamma) \over n} \beta^\T(A^\T A)^{-1}\beta \\
			&\lesssim  {K + \log n\over n}\sigma^2 + 
			{n \over r_e(\Gamma)}	\sigma^2+  {r_e(\Gamma) \over n}\|\Gamma\|_{{\rm op}}\ \beta^\T(A^\T A)^{-1}\beta .
			\end{align*}
		\qed

	\subsubsection*{Proof of Theorem \ref{thm_pred_A}}\label{sec_proof_pred_A}
      Instead of directly applying Theorem \ref{thm_pred}, we slightly modify the proofs of Theorem \ref{thm_pred} to obtain a sharp result for $\R(\wh A)$. 
      
      From the proof of Theorem \ref{thm_pred}, display (\ref{eq_decomp}) gives
      \[
       \R(\wh A) 	-\sigma^2\le \left\|
	    \C^{1/2}\left(A^\T \wh \alpha_{\wh A}-\beta \right)
	        \right\|^2 + \|\Gamma\|_{\op} \left\|
	        \wh \alpha_{\wh A}
	        \right\|^2.
      \]
      We then point out the modifications of the proof of Lemmas \ref{lem_t1} and \ref{lem_t2}. Recall $\wh A \in \RR^{p\times \wh K}$. We work on the event $\E^*$ defined in the proof of Theorem \ref{thm_pred} intersected with the event that $\wh K = K$ and 
      $$
      \|\wh A - A\|_{\rm op}^2 \le \|\wh A- A\|_F^2 \lesssim \|A_J\|_0{\log\pn \over n}.
      $$
      The last two events holds with probability at least  $1-c\pn^{-1}$ for some constant $c > 0$ \citep{LOVE}. 
      In display (\ref{eq_bd_bxbzbeta}) of Lemma \ref{lem_t1} for bounding $\|
	        \wh \alpha_{\wh A}
	        \|^2$, we use 
	  \begin{align*}
		\left\|
		\wh B (\X \wh B)^+\Z\beta
		\right\|^2 &\le 3	\left\| \wh B (\X \wh B)^+\X \wh B \wh B^{+}A^{+\T}  \beta
		\right\|^2+3\left\|\wh B (\X \wh B)^+\X P^{\perp}_{\wh B}A^{+\T} \beta\right\|^2\\\nonumber
		&\quad  + 3\left\|\wh B (\X \wh B)^+\W A^{+\T}\beta\right\|^2\\\nonumber
		&\le 3	\left\| \wh B (\X \wh B)^+\X \wh B \wh B^{+}\right\|_{{\rm op}}^2\left\|A^{+\T}  \beta
		\right\|^2+3\left\|\wh B (\X \wh B)^+\right\|_{{\rm op}}^2\left\|\X P^{\perp}_{\wh B} A^{+\T} \beta\right\|^2\\\nonumber
		&\quad  + 3\left\|\wh B (\X \wh B)^+\right\|_{{\rm op}}^2\left\|\W A^{+\T}\beta\right\|^2.
	\end{align*}
    We change the way to bound the second term on the right hand side. Specifically, set $\wh B = \wh A$ and use $(a+b)^2 \le 2a^2 + 2b^2$ twice to obtain
    \begin{align*}
        \left\|\X P^{\perp}_{\wh A} A^{+\T} \beta\right\|^2 & \le  2\left\|\Z A P^{\perp}_{\wh A} A^{+\T} \beta\right\|^2 + 2\left\|\W P^{\perp}_{\wh A} A^{+\T} \beta\right\|^2\\
        &\le 2\left\|\Z\Omega^{1/2}\right\|_{\rm op}^2 \left\|\C^{1/2}(A-\wh A)^\T P^{\perp}_{\wh A} A^{+\T} \beta\right\|^2 & (\textrm{by }\wh A^\T \wh P_{\wh A}^{\perp} = 0)\\
        &\quad + 4\left\|\W A^{+\T} \beta\right\|^2 + 4\left\|\W P_{\wh A} A^{+\T} \beta\right\|^2     & (\textrm{by } P^{\perp}_{\wh A}  = \bI_p - P_{\wh A}).
    \end{align*} 
    By $\E_{\Z}$, $\E_{\W}'$ and Lemma \ref{lem_WPA}, after a bit algebra, we conclude 
    \begin{align}\label{bd_XAbeta}\nonumber
        {1\over n}\left\|\X P^{\perp}_{\wh A} A^{+\T} \beta\right\|^2 &\lesssim  \left(\|A_J\|_0{\log\pn \over n} + \delta_{W,J}\right)\beta^T(A^\T A)^{-1}\beta+ \beta^\T A^+\Gamma A^{+\T} \beta\\
        &\lesssim 
        \left(\|A_J\|_0{\log\pn \over n} + \|\Gamma\|_{\op}\right)\beta^T(A^\T A)^{-1}\beta+ \beta^\T A^+\Gamma A^{+\T} \beta.
    \end{align}
    with probability at least  $1-cn^{-1}$. In the last step, we used the fact that $\|\Gamma\|_{\op}$ is bounded and $\|A_J\|_{\ell_0/\ell_2} \le \|A_J\|_0$. Together with the proofs of Lemma \ref{lem_t1}, one can deduce that
    \[
        \|\wh \alpha_{\wh A}\|^2 \lesssim {(K +\log n)\sigma^2\over n\wh \eta}+  \beta^\T(A^\T A)^{-1}\beta + \wh \eta^{-1} \left(\wh \psi  \beta^{\top}      (A^\T A)^{-1}\beta +  \beta^\T A^+\Gamma A^{+\T}\beta\right).
    \]
    where 
    \[
        \wh \psi \lesssim \|\Gamma\|_{\op} + \|A_J\|_0{\log\pn \over n}.
    \]
    To bound $\|\C^{1/2}(A^\T \wh \alpha_{\wh A} - \beta)\|^2$, we modify two places in the proof of Lemma \ref{lem_t2}. Display (\ref{eq_bd_lem_t2}) is bounded by 
    \begin{align*}
		\left\|
		\C^{1/2}[A^\T \wh A (\X \wh A)^+\Z -\bI_K]\beta 
		\right\|^2
		&\lesssim {1\over n}\left\|P_{\X \wh A}^{\perp}\Z \beta 
		\right\|^2+{1\over n}\left\|\W \wh A(\X \wh A)^+\Z \beta 
		\right\|^2\\
		&\lesssim {1\over n}\left\|P_{\X \wh A}^{\perp}\Z \beta 
		\right\|^2+{1\over n}\left\|\W P_{\wh A}\right\|_{\rm op}^2\left\|\wh A(\X \wh A)^+\Z \beta 
		\right\|^2
	\end{align*}
	where we will invoke Lemma \ref{lem_WPA}. For the first term of the right hand side, by (\ref{eq_bd_Zbeta}), we have 
	\begin{align*}
 	\left\|P_{\X\wh B}^{\perp}\Z\beta\right\|^2&\le 2
 	\left\|P_{\X\wh B}^{\perp}\W A^{+\T}\beta \right\|^2+ 2
 	\left\|P_{\X\wh B}^{\perp}\X P_{\wh B}^{\perp} A^{+\T} \beta\right\|^2\\\nonumber
 	&\le 2 
 	\left\|\W A^{+\T}\beta \right\|^2+ 2
 	\left\|\X P_{\wh B}^{\perp}A^{+\T} \beta\right\|^2
 	\end{align*}
 	which can be further bounded by using  (\ref{bd_XAbeta}) and invoking the event $\E_{\W}'$. Collecting all these ingredients, we conclude 
 	\begin{align*}
 	    \left\|
	    \C^{1/2}\left(A^\T \wh \alpha_{\wh A}-\beta \right)
	        \right\|^2 &\lesssim \left(1+ {\delta_{W,J} \over \wh \eta}\right)\left( {K+\log n\over n}\sigma^2 +\beta^\T A^+\Gamma A^{+\T}\beta\right)\\
		&\quad + \left[\left(1+ {\delta_{W,J} \over \wh \eta}\right)\wh \psi + \delta_{W,J}\right]\beta^\T(A^\T A)^{-1}\beta.
 	\end{align*}
 	It then remains to lower bound $\wh \eta$ by bounding $\sigma_K(\X P_{\wh A})$ from below. By Weyl's inequality, $\rank(\wh A) = K$, we have 
	\begin{align*}
	\sigma_K\left(\X P_{\wh A}A(A^\T A)^{-1/2}  \right)  &\ge \sigma_K\left(\X A(A^\T A)^{-1/2} \right) - \left\|
	\X P_{\wh A}^{\perp} A(A^\T A)^{-1/2}
	\right\|_{{\rm op}}\\
	&\ge \sigma_K\left(\X AN^{-1/2}N^{1/2}(A^\T A)^{-1/2} \right) - \left\|
	\X P_{\wh A}^{\perp} A(A^\T A)^{-1/2}
	\right\|_{{\rm op}}\\
	&\ge  \sigma_K\left(\X AN^{-1/2}\right)\sigma_K\left(N^{1/2}(A^\T A)^{-1/2}\right) - \left\|
	\X P_{\wh A}^{\perp} A(A^\T A)^{-1/2}
	\right\|_{{\rm op}}.
	\end{align*}
	by writing $N = A^\T \Sigma A$.  
	To lower bound $\sigma_K\left(\X AN^{-1/2} \right)$, using Weyl's inequality again and invoking Lemma \ref{lem_op_norm_diff} yield
	\begin{align*}
	&\lambda_K\left(
	N^{-1/2}A^\T 	{1\over n}\X^\T \X A N^{-1/2}
	\right)\\
	& \gtrsim  \lambda_K\left( N^{-1/2}A^\T \Sigma A N^{-1/2}
	\right)- \left\|
	N^{-1/2}A^\T \left({1\over n}\X^\T \X - \Sigma \right)AN^{-1/2}
	\right\|_{{\rm op}}\\
	&\gtrsim 1 - \sqrt{K\log n\over n} - {K\log n \over n} \gtrsim 1
	\end{align*}
	with probability at least  $1-cn^{-C}$. On the other hand,  by $\X = \Z A^\T + \W$, 
	\begin{align*}
	    \left\|
	\X P_{\wh A}^{\perp} A(A^\T A)^{-1/2}
	\right\|_{{\rm op}} &\le  \left\|
	\Z A^\T  P_{\wh A}^{\perp} A(A^\T A)^{-1/2}
	\right\|_{{\rm op}} +\left\|
	\W  P_{\wh A}^{\perp} A(A^\T A)^{-1/2}
	\right\|_{{\rm op}}\\
	&\le \left\|
	\Z (A-\wh A)^\T
	\right\|_{{\rm op}} +\left\|
	\W A(A^\T A)^{-1/2}
	\right\|_{{\rm op}} + \left\|
	\W  P_{\wh A}A(A^\T A)^{-1/2}
	\right\|_{{\rm op}}\\
	&\le \left\|
	\Z \Omega^{1/2}\right\|_{\rm op} \sigma_1(\C) \left\|(A-\wh A)^\T
	\right\|_{{\rm op}} +\left\|
	\W A(A^\T A)^{-1/2}
	\right\|_{{\rm op}} + \left\|
	\W  P_{\wh A}
	\right\|_{{\rm op}}.
    \end{align*}
    By $\E_{\Z}$ and Lemmas \ref{lem_WPA} and \ref{lem_op_norm}, we have 
    \[
        {1\over n}\left\|
	\X P_{\wh A}^{\perp} A(A^\T A)^{-1/2}
	\right\|_{{\rm op}}  \lesssim \delta_{W,J} + {\|A_J\|_0 \log\pn \over n}\lesssim \|\Gamma\|_\op +   {\|A_J\|_0 \log\pn \over n}
    \]
    with probability at least  $1-cn^{-1}$.
	Provided that 
	\[
	    \lambda_K(A\C A^\T) \ge C \left(\|\Gamma\|_\op +   {\|A_J\|_0 \log\pn \over n}\right)
	\]
	for sufficiently small constant $C>0$, we then conclude that
	\[
	\sigma_K^2\left(\X P_{\wh A} A(A^\T A)^{-1/2} \right) \gtrsim {n}\lambda_K(A\C A^\T) 
	\]
	from noting 
	$
	\sigma_K^2\left(N^{1/2}(A^\T A)^{-1/2}\right)= \lambda_K(A\C A^\T).
	$ 
	This concludes $\wh \eta \gtrsim \lambda_K(A\C A^\T) $. The result then follows by collecting terms. \qed 
	
	\bigskip

	The following lemma provides upper bounds for the operator norm of $\W P_{\wh A}$. 
	Recall that $\|A_J\|_{\ell_0/\ell_2} = \sum_{j\in J}1_{\{\|A_{j\sbt}\|_2 \ne 0\}}$.
	\begin{lemma}\label{lem_WPA}
	Under conditions of Theorem \ref{thm_pred_A}, with probability at least  $1-c\pn^{-1}$, one has 
	    \[
	        {1\over n} \left\|\W P_{\wh A}\right\|_{\rm op}^2 \lesssim \|\Gamma\|_{\op}\left(1 + {\|A_J\|_{\ell_0/\ell_2} \over n}\right) := \delta_{W,J}.
	    \]
	\end{lemma}
	\begin{proof}
	    We work on the event $\wh K = K$ and $\wh A_{I} = A_I$ which holds with probability at least  $1-c\pn^{-c'}$ \citep{LOVE}. Then 
	    \begin{align*}
	        \left\|\W P_{\wh A}\right\|_{\rm op} = \left\|\W \wh A \wh A^+\right\|_{\rm op} &\le \left\|\W_{\sbt I}A_I \wh A^+\right\|_{\rm op} +  \left\|\W_{\sbt J}\wh A_J \wh A^+\right\|_{\rm op}\\
	        &\le \left\|\W_{\sbt I}A_I(A_I^\T A_I)^{-1/2}\right\|_{\rm op} \left\|(A_I^\T A_I)^{1/2} \wh A^+\right\|_{\rm op} + \|\W_{\sbt J}\|_{\rm op}\left\|\wh A_J \wh A^+\right\|_{\rm op}.
	    \end{align*}
	    Since 
	    \[
	        \left\|(A_I^\T A_I)^{1/2} \wh A^+\right\|_{\rm op}^2 = \left\|(A_I^\T A_I)^{1/2}(\wh A^\T \wh A)^{-1}(A_I^\T A_I)^{1/2}\right\|_{\rm op}\le 1
	   \]
	   by noting $\wh A^\T \wh A = A_I^\T A_I + \wh A_J^\T \wh A_J$,
	   and similar arguments yield
	   \begin{align*}
	    \left\|\wh A_J \wh A^+\right\|_{\rm op}^2 &  =  \left\|\wh A_J(\wh A^\T \wh A)^{-1}\wh A_J^\T \right\|_{\rm op} = \left\|(\wh A^\T \wh A)^{-1/2}\wh A_J^\T\wh A_J (\wh A^\T \wh A)^{-1/2}\right\|_{\rm op}\le 1,
	   \end{align*}
	   invoking Lemma \ref{lem_op_norm} to bound $\|\W_{\sbt I}A_I(A_I^\T A_I)^{-1/2}\|_{\rm op}$ and $\|W_{\sbt J}\|_{\rm op}$ gives
	   \[
	        {1\over n}\left\|\W_{\sbt I}A_I(A_I^\T A_I)^{-1/2}\right\|_{\rm op}^2 \lesssim \left\|\Psi_{II}\right\|_{\rm op}+ {\tr(\Psi_{II}) \over n}, 
	   \]
	   $$
	    {1\over n}\|\W_{\sbt J}\|_{\rm op}^2 \lesssim \left\|
	        [\Gamma]_{JJ}
	    \right\|_{\rm op} + {\tr( [\Gamma]_{JJ}) \over n} \le \delta_{W,J},
	   $$
	   with probability at least  $1-2e^{-n}$,
	   where 
	   \[
	        \Psi_{II} = (A_I^\T A_I)^{-1/2}A_I^\T [\Gamma]_{II} A_I(A_I^\T A_I)^{-1/2}.
	   \]
	   The result then follows by using $\|\Psi_{II}\|_{\rm op} \le \|[\Gamma]_{II}\|_\op$, $\tr(\Psi_{II}) \le K \|\Psi_{II}\|_{\rm op} \le  K \|[\Gamma]_{II}\|_{\op}$ and $K\log n \lesssim n$.
	\end{proof}

	\bigskip

	\subsection{Proof of Theorem \ref{thm:split} in Section \ref{sec_data_split}}\label{proof:split}
    For any $\alpha \in \RR^p$, let
	 \[\wh \R(\alpha) = \frac{2}{n}\sum_{i\in D_1}[Y_i - X_i^{\top}      \alpha ]^2\]
	 so that for all $m\in [M]$, by the definition of $\wh m$, $\wh S(\wh \alpha)  \le \wh S(\wh \alpha_m)$.
	 Also let 
	  \[\wh S(\alpha) = \frac{2}{n}\sum_{i\in D_1}[Z_i^{\top}     \beta - X_i^{\top}      \alpha ]^2.\]
	 Finally, for any fixed or random $\alpha$ define
	 \[S(\alpha) = \EE_{(Z_*,X_*)}(Z_*^{\top}     \beta - X_*^{\top}     \alpha)^2,\hspace{1cm}\R(\alpha) = S(\alpha) + \sigma^2,\]
	 where the expectation is over $(Z_*,X_*)$ that are independent of $\alpha$.
	 
	 We have
	 \begin{align*}
	     S(\wh \alpha) &= \R(\wh \alpha) - \sigma^2\\
	     &= (1+a)[\wh \R(\wh\alpha) - \frac{2}{n}\sum_{i\in D_1}^n\eps_i^2]  + [\R(\wh\alpha) - (1+a)\wh \R(\wh\alpha) - (\sigma^2 - (1+a)\frac{2}{n}\sum_{i\in D_1}\eps_i^2)].
	 \end{align*}
	  Using $\wh \R(\wh \alpha)  \le \wh \R(\wh \alpha_m)$ in the first term of the above, we have for any $m\in [M]$,
	 \begin{align}
	     S(\wh \alpha) &\le (1+a)[\wh \R(\wh\alpha_m) - \frac{2}{n}\sum_{i\in D_1}^n\eps_i^2] \nonumber\\
	     &\hspace{5mm} + \max_m[\R(\wh\alpha_m) - (1+a)\wh \R(\wh\alpha_m)  -(\sigma^2 - (1+a)\frac{2}{n}\sum_{i\in D_1}\eps_i^2)] \nonumber\\
	     &= (1+a)[\wh \R(\wh\alpha_m) - \frac{2}{n}\sum_{i\in D_1}^n\eps_i^2]\nonumber\\
	     &\hspace{5mm} + \max_m[S(\wh\alpha_m) -(1+a)\wh S(\wh\alpha_m) +  2(1+a)\frac{2}{n}\sum_{i\in D_1}\eps_i(X_i^{\top}     \wh\alpha_m - Z_i^{\top}     \beta)]\nonumber\\
	     &\le (1+a)[\wh \R(\wh\alpha_m) - \frac{2}{n}\sum_{i\in D_1}^n\eps_i^2]+ \max_m[S(\wh\alpha_m) -(1+\frac{a}{2})\wh S(\wh\alpha_m) ]\nonumber\\
	     &\hspace{5mm}+ \max_m[ 2(1+a)\frac{2}{n}\sum_{i\in D_1}\eps_i(X_i^{\top}     \wh\alpha_m - Z_i^{\top}     \beta) - \frac{a}{2}\wh S(\wh\alpha_m) ].\label{eqn_s_1}
	 \end{align}
	 The first term in the above can be further re-written as
	 \begin{align*}
        \wh \R(\wh\alpha_m) - \frac{2}{n}\sum_{i\in D_1}^n\eps_i^2 &= (1+a)S(\wh \alpha_m) + [\wh \R(\alpha_m) - (1+a)S(\wh\alpha_m) - \frac{2}{n}\sum_{i\in D_1}\eps_i^2]\\
        &= (1+a)S(\wh \alpha_m) + [\wh S(\wh\alpha_m) - (1+a)S(\wh\alpha_m)+\frac{4}{n}\sum_{i\in D_1}\eps_i(Z_i^\top\beta - X^\top_i\wh\alpha_m)]\\
        &\le (1+a)S(\wh \alpha_m) + \max_m[ (1+\frac{a}{2})\wh S(\wh\alpha_m) - (1+a)S(\wh\alpha_m)]\\
        &\hspace{5mm}+\max_m[\frac{4}{n}\sum_{i\in D_1}\eps_i(Z_i^\top\beta - X^\top_i\wh\alpha_m) - \frac{a}{2}\wh S(\wh\alpha_m)].
	 \end{align*}
	 Using this result in (\ref{eqn_s_1}), we find that for any $m\in [M]$,
	 \begin{align}
	     S(\wh\alpha) &\le  (1+a)^2 S(\wh \alpha_m) \nonumber\\
	     &\hspace{5mm} +  (1+a)\max_m[ (1+\frac{a}{2})\wh S(\wh\alpha_m) - (1+a)S(\wh\alpha_m)]\nonumber\\
	     &\hspace{5mm}+(1+a)\max_m[\frac{4}{n}\sum_{i\in D_1}\eps_i(Z_i^\top\beta - X^\top_i\wh\alpha_m) - \frac{a}{2}\wh S(\wh\alpha_m)]\nonumber\\
	     &\hspace{5mm}+ \max_m[S(\wh\alpha_m) -(1+\frac{a}{2})\wh S(\wh\alpha_m) ]\nonumber\\
	     &\hspace{5mm}+ \max_m[ 2(1+a)\frac{2}{n}\sum_{i\in D_1}\eps_i(X_i^{\top}     \wh\alpha_m - Z_i^{\top}     \beta) - \frac{a}{2}\wh S(\wh\alpha_m) ]\nonumber\\
	     &=:  (1+a)^2 S(\wh \alpha_m) + (1+a)T_1+(1+a)T_2  + T_3+T_4.\label{eqn_s_2}
	 \end{align}
	 Below we prove that
	 	 \begin{equation}\label{eqn:t1t3}
	   \PP_\theta\left((1+a)T_1 +T_3\le c_1 \frac{(2+a)^3}{a}\cdot\frac{\max_mS(\wh\alpha_m)\log(nM)}{n}\right)\ge 1- c_1'n^{-1},
	 \end{equation}
	 and
    \begin{equation}\label{eqn:t2t4}
    \PP_\theta\left\{(1+a)T_{2} + T_{4} \le c_2 \frac{(1+a)^3}{a}\sigma^2\frac{\log (nM)}{n} \right\} \ge 1- c_2'n^{-1},
	 \end{equation}
	 where $c_1$ and $c_2$ depend only on $\gamma_z,\gamma_w,\gamma_\eps$ from Definition \ref{frm}, and $c_1,c_2>0$ are absolute constants. The final result follows from taking a minimum over $m$ in (\ref{eqn_s_2}) and combining (\ref{eqn:t1t3}) and (\ref{eqn:t2t4}) with a union bound.
	 
	 \subsubsection*{Bounding $T_1$ and $T_3$}
    Since $\wh\alpha_1,\ldots,\wh\alpha_2$ are independent of $\{X_i:i\in D_1\}$, we will prove (\ref{eqn:t1t3}) for the case when $\wh\alpha_1,\ldots,\wh\alpha_2$ are non-random without loss of generality.
    
	We first consider $T_3$. For all $t,b >0$, the following holds:
	 \begin{equation}\label{t3id}
	     S - \wh S \le \sqrt{t}\sqrt{S}\hspace{3mm}\Rightarrow\hspace{3mm} S\le (1+b)\wh S + t\frac{1+b}{b},
	 \end{equation}
	 where we write $S=S(\wh\alpha_m)$ and $\wh S =\wh S(\wh\alpha_m)$.
	 To prove this, suppose the left hand side holds true and consider the cases $\sqrt{S}\le \frac{1+b}{b}\sqrt{t}$, which implies $S\le\wh S+ t\frac{1+b}{b}$, and $\sqrt{S}> \frac{1+b}{b}\sqrt{t}$, which implies $S\le \wh S + \frac{b}{1+b}S$ and thus $S\le (1+b)\wh S$. Thus,
	 \begin{align}
	     \PP_\theta\left(T_3> t\frac{1+a/2}{a/2}\right) &\le M \max_m\PP_\theta\left( S(\wh\alpha_m) - (1+\frac{a}{2})\wh S(\wh\alpha_m) > t\frac{1+a/2}{a/2}\right)\nonumber\\
	     &\le M\max_m\PP_\theta\left(\frac{S(\wh\alpha_m) - \wh S(\wh\alpha_m)}{\sqrt{S(\wh\alpha_m)}}> \sqrt{t}\right) &(\text{by } (\ref{t3id}))\nonumber\\
	     &\le M\max_m\PP_\theta\left(\bigg |\frac{2}{n}\sum_{i\in D_1}[\EE[g_i(m)]  - g_i(m)]\bigg |> \sqrt{t}\right),\label{t3 bern}
	 \end{align}
	 where we let $g_i(m) \coloneqq (Z_i^\top\beta - X_i^\top\wh\alpha_m)^2/\sqrt{S(\wh\alpha_m)}$ in the last step. Recalling that for any random variable $U$, $\|U^2\|_{\psi_1} = \|U\|_{\psi_2}^2$, and using the assumption that $\wh\alpha_m$ is a fixed vector, we find
	 \begin{align*}
	     &\|(Z_i^\top\beta - X_i^\top\wh\alpha_m)^2\|_{\psi_1}\\
	     &= \|Z_i^\top\beta - X_i^\top\wh\alpha_m\|^2_{\psi_2}\\
	     &\le \|Z_i^\top\beta - Z_i^\top A^\top \wh\alpha_m\|^2_{\psi_2} + \|W_i^\top \wh\alpha_m\|^2_{\psi_2} &(\text{since } X_i = AZ_i+W_i)\\
	     &= \|(\sz^{-1/2}Z_i)^\top(\sz^{1/2}[\beta - A^\top\wh\alpha_m])\|_{\psi_2}^2 +\|(\sw^{-1/2} W)^\top(\sw^{1/2}\wh\alpha_m)\|_{\psi_2}^2\\
	     &=\|\sz^{1/2}(\beta - A^\top\wh\alpha_m)\|^2\|(\sz^{-1/2}Z_i)^\top u)\|_{\psi_2}^2&(\text{with } \|u\|=\|v\|=1)\\
	     &\hspace{10mm}+\|\sw^{1/2}\wh\alpha_m\|^2\|(\sw^{-1/2} W)^\top v\|_{\psi_2}^2 \\
	     &\le c_1 \|\sz^{1/2}(\beta - A^\top\wh\alpha_m)\|^2+c_1\|\sw^{1/2}\wh\alpha_m\|^2 &(\text{by Definition }(\ref{frm}))\\
	     &= c_1S(\wh\alpha_m),
	 \end{align*}
	 where $c_1 = c_1(\gamma_z,\gamma_w)$. Thus, 
	 \[\|\EE g_i(m) - g_i(m) \|_{\psi_1}\lesssim \|g_i(m)\|_{\psi_1}\le c_1  \sqrt{S(\wh\alpha_m)},\] 
	 so by Bernstein's inequality \citep{vershynin_2012},
	 \begin{equation}\label{eqn:bern app}
	     \PP_\theta\left(\bigg|\frac{2}{n}\sum_{i\in D_1}[\EE[g_i(m)]  - g_i(m)]\bigg|> \sqrt{t}\right) \le 2\exp\left(-n \left(\frac{t}{c_1S(\wh\alpha_m)} \wedge \sqrt{\frac{t}{c_1S(\wh\alpha_m)}}\right)\right).
	 \end{equation}
	 Choosing $t = c_1\max_m S(\wh\alpha_m)\log(nM)/n$, and combining with (\ref{t3 bern}), for  $\log(M)<cn$,
	 \begin{equation}\label{t3final}
	     \PP_\theta\left(T_3> \frac{1+a/2}{a/2}\cdot c_1\frac{\max_mS(\wh\alpha_m)\log(nM)}{n}\right) \le 2/n.
	 \end{equation}
	 
	 We next consider $T_1$. For $t,b >0$, we have
	 \begin{equation*}
	     \wh S - S \le \sqrt{t}\sqrt{S}\hspace{3mm}\Rightarrow\hspace{3mm} \wh S\le \left(1+\frac{b}{1+b}\right) S+ t\frac{1+b}{b}.
	 \end{equation*}
    To prove this, suppose the left hand side holds and consider the cases $\sqrt{S}\le \frac{1+b}{b}\sqrt{t}$, which implies $\wh S \le S+\frac{1+b}{b}t$, and $\sqrt{S}> \frac{1+b}{b}\sqrt{t}$, which implies $\wh S\le [1+b/(1+b)]S$. Multiplying the right hand inequality by $(1+b)$, and choosing $b=a/2$, we find
    \begin{equation}\label{eqn:t1id}
        \left(1+\frac{a}{2}\right)\wh S- (1+a)S > t\frac{(1+a/2)^2}{a/2} \hspace{3mm}\Rightarrow\hspace{3mm}  \wh S - S > \sqrt{t}\sqrt{S}
    \end{equation}
    Recalling
    \[T_1= \max_m[ (1+\frac{a}{2})\wh S(\wh\alpha_m) - (1+a)S(\wh\alpha_m)],\]
    an application of (\ref{eqn:t1id}) gives
    \begin{align*}
        \PP_\theta\left(T_1> t\frac{(1+a/2)^2}{a/2}\right)&\le M\max_m\PP_\theta(\wh S(\wh\alpha_m) - S(\wh\alpha_m) > \sqrt{t}\sqrt{S})\\
        &\le M\max_m\PP_\theta\left(\bigg |\frac{2}{n}\sum_{i\in D_1}[\EE[g_i(m)]  - g_i(m)]\bigg |> \sqrt{t}\right)
    \end{align*}
	 Choosing $t = c_1\max_m S(\wh\alpha_m)\log(nM)/n$ and applying (\ref{eqn:bern app}) with $\log(M) < cn$, we conclude
	 \begin{equation}\label{t1final}
	     \PP_\theta\left(T_1 >  \frac{(1+a/2)^2}{a/2}\cdot c_1\frac{\max_m S(\wh\alpha_m)\log(nM)}{n}\right) \le 2/n.
	 \end{equation}
	 Combining (\ref{t3final}) and (\ref{t1final}) with a union bound and some algebra proves (\ref{eqn:t1t3}).

	 \subsubsection*{Bounding $T_2$ and $T_4$}

    For each $i\in D_1$, define $h_i(m) = (Z_i^\top\beta - X_i^\top\wh \alpha_m)/[\wh S(\wh \alpha_m)]^{1/2}$. Using the inequality $2|xy| \le x^2/c + cy^2$ for $c>0$, we have that
	\begin{align*}
	    \frac{4}{n}\sum_{i\in D_1}\eps_i(Z_i^\top\beta - X_i^\top\wh \alpha_m) - \frac{a}{2} \wh S(\wh\alpha_m) &=2[\wh S(\wh \alpha_m)]^{1/2} \frac{2}{n}\sum_{i\in D_1} \eps_i h_i(m) -  \frac{a}{2} \wh S(\wh \alpha_m) \\
	    &\le 2[\wh S(\wh \alpha_m)]^{1/2} \bigg |\frac{2}{n}\sum_{i\in D_1} \eps_i h_i(m)\bigg | - \frac{a}{2} \wh S(\wh \alpha_m) \\
	    &\le \frac{2}{a}\bigg |\frac{2}{n}\sum_{i\in D_1} \eps_i h_i(m)\bigg|^2
	\end{align*}
	Similarly,
	\[2(1+a)\frac{2}{n}\sum_{i\in D_1}\eps_i(X_i^\top\wh\alpha_m - Z_i^\top\beta) - \frac{a}{2}\wh S(\wh\alpha_m) \le \frac{2(1+a)^2}{a}\bigg |\frac{2}{n}\sum_{i\in D_1} \eps_i h_i(m)\bigg|^2.\]
	Thus, 
	\[T_2+T_4 \lesssim \max_m\frac{(1+a)^2}{a} \bigg|\frac{2}{n}\sum_{i\in D_1} \eps_i h_i(m)\bigg|^2,\]
	so
	\[\PP_\theta \left(T_2+T_4 \ge t \frac{(1+a)^2}{a}\right) \le M\max_{m}\PP_\theta\left(\bigg|\frac{2}{n}\sum_{i\in D_2} \eps_i h_i(m)\bigg| \ge \sqrt{t} \right)\]
	Since $\{\eps_i\}_{i\in D_1}$ is independent of $(Z_i, X_i)_{i\in D_2}$, $\EE[\eps_i h_i(m)] = 0$ for all $i\in D_2$. Furthermore, $\|\eps_i\|_{\psi_2}\lesssim \sigma$ and $|h_i(m)|$ is bounded by $1$, so $\|\eps_i h_i(m)\|_{\psi_2} \le \sigma/c_2$, where $c_2=c_2(\gamma_\eps)$. Thus by Hoeffding's inequality \citep{vershynin_2012},
	\[\PP_\theta\left(\bigg|\frac{2}{n}\sum_{i\in D_2} \eps_i h_i(m)\bigg| \ge \sqrt{t}\right)\le 2\exp(-c_2tn/\sigma^2). \]
    Choosing $t = \sigma^2\log(nM)/(c_2n)$ completes the proof of (\ref{eqn:t2t4}).
	\qed
	
	\medskip

	\section{Auxiliary lemmas}\label{sec_proof_aux}
	
	The following lemma is used in our analysis. The tail inequality is for a quadratic form of sub-Gaussian random vectors. It is a slightly simplified version of Lemma 30 in \cite{Hsu2014}. 
	\begin{lemma}\label{lem_quad}
		Let $\xi\in \RR^d$ be a $\gamma_\xi$ sub-Gaussian random vector. For all symmetric positive semi-definite matrices $H$, and all $t\ge 0$, 
		\[
		\PP\left\{
		\xi^\T H\xi > \gamma_\xi^2\left(
		\sqrt{{\rm tr}(H)}+ \sqrt{2\|H\|_{\op}t}
		\right)^2
		\right\} \le e^{-t}.
		\] 
	\end{lemma}
	\begin{proof}
		From Lemma 8 in \cite{Hsu2014}, one has
		\[
		\PP\left\{
		\xi^\T H\xi > \gamma_\xi^2\left(
		\tr(H) + 2\sqrt{\tr(H^2)t}+2\|H\|_{\op}t
		\right)
		\right\} \le e^{-t},
		\] 
		for all $t\ge 0$. The result then follows from $\tr(H^2) \le \|H\|_{\op}\tr(H)$.
	\end{proof}

	The following lemma provides an upper bound on the operator norm of $\G H \G^\T$ where  $\G\in \RR^{n\times d}$ is a random matrix and its rows are independent sub-Gaussian random vectors. It differs from \citet[Theorem 10]{bunea2020interpolation} in the sense that 
	independence across columns of $\G$ is not required. 
	\begin{lemma}\label{lem_op_norm}
	    Let $\G$ be $n$ by $d$ matrix whose rows are independent $\gamma$ sub-Gaussian  random vectors with identity covariance matrix. Then for all symmetric positive semi-definite matrices $H$, 
		\[
		\PP\left\{{1\over n}\| \G H \G^\T \|_{{\rm op}} \le \gamma^2\left( \sqrt{{\rm tr}(H) \over n} + \sqrt{6\|H\|_{\op}}
		\right)^2\right\} \ge  1 -  e^{-n}
		\]
	\end{lemma}
	\begin{proof}
	    By definition and the property of the $1/2$-net $\N$,
	    \[
	        \|\G H \G^\T \|_{\op} = \sup_{u\in \S^{n-1}} u^\T \G H \G^\T  u \le 2 \sup_{u\in \N}   u^\T \G H \G^\T  u.
	    \]
	    For fixed $u\in \N$, since $\G^\T u$ is a $\gamma$ sub-Gaussian random vector, an application of Lemma \ref{lem_quad} with $\xi = \G^\T u$, $\gamma_{\xi} = \gamma$ and $H = H$ yields 
	    \[
	        \PP\left\{
	            u^\T \G H \G^\T u > \gamma^2\left(
	            \sqrt{\tr(H)} + \sqrt{2\|H\|_{\op}t}\right)^2
	        \right\} \le e^{-t}.
	    \]
	    Since $|\N| \le 5^n$, see \citet[Lemma 5.2]{vershynin_2012}, choosing $t = 3n$ and taking a union bound over $u\in \N$ completes the proof. 
	\end{proof}
	
	Another useful concentration inequality of the operator norm of the random matrices with i.i.d. sub-Gaussian rows is stated in the following lemma. This is an immediate result of \citet[Remark 5.40]{vershynin_2012}.
	
	\begin{lemma}\label{lem_op_norm_diff}
 		Let $\G$ be $n$ by $d$ matrix whose rows are i.i.d. $\gamma$ sub-Gaussian  random vectors with covariance matrix $\Sigma_Y$. Then for every $t\ge 0$, with probability at least  $1-2e^{-ct^2}$,
 		\[
 		\left\|	{1\over n}\G^\T \G - \Sigma_Y\right\|_{{\rm op}}\le \max\left\{\delta, \delta^2\right\} \left\|\Sigma_Y\right\|_{{\rm op}},
 		\]
 		with $\delta = C\sqrt{d/n}+ t/\sqrt n$ where $c = c(\gamma)$ and $C=C(\gamma)$ are positive constants depending on $\gamma$.
 	\end{lemma}

    The deviation inequalities of the inner product of two random vectors with independent sub-Gaussian elements are well-known; we state the one in \cite{ER} for completeness. 
	
	\begin{lemma}\cite[Lemma 10]{ER}\label{lem_bernstein}
		Let $\{X_t\}_{t=1}^n$ and $\{Y_t\}_{t=1}^n$ be any two sequences, each with zero mean independent $\gamma_x$ sub-Gaussian and $\gamma_y$ sub-Gaussian elements. Then, for some absolute constant $c>0$, we have 
		\[
		\PP\left\{{1\over n}\left|\sum_{t=1}^n\left(X_t Y_t - \EE[X_t Y_t]\right)\right| \le \gamma_x \gamma_y t \right\}\ge 1-2\exp\left\{-c\min\left( t^2,t \right)n\right\}.
		\]
		In particular, when $\log p\le n$, one has
		\[
		\PP\left\{{1\over n}\left|\sum_{t=1}^n\left(X_t Y_t - \EE[X_t Y_t]\right)\right| \le C\sqrt{\log\pn \over n} \right\}\ge 1-2\pn^{-c}
		\]
		where $c \ge 2$ and $C = C(\gamma_x,\gamma_y,c)$ are some positive constants.
	\end{lemma}

    \section{The LOVE algorithm}\label{app_love}
	For the reader's convenience, we give the specifics of estimating $\wh A$ in the Essential Regression model, as developed in \cite{LOVE}. The first step is estimation of the number of latent factors, $K$, and the partition of pure variables, $\I$, which is achieved by Algorithm \ref{alg1} below.
	
	{\begin{algorithm}[ht]
			\caption{Estimate the partition of the pure variables $\I$ by $\wh \I$}\label{alg1}
			\begin{algorithmic}[1]
				\Procedure {PureVar}{$\wh \Sigma$, $\delta$}
				\State $\wh \I \gets \emptyset$.
				\ForAll {$i\in [p]$} 
				\State $\wh I^{(i)} \gets \bigl\{l\in [p]\setminus\{i\}: \max_{j\in [p]\setminus\{i\}}|\wh \Sigma_{ij}| \le |\wh \Sigma_{il}|+2\delta\bigr\}$
				\State $Pure(i) \gets True$.
				\ForAll {$j \in \wh I^{(i)}$}
				\If {$\bigl||\wh \Sigma_{ij}|- \max_{k\in [p]\setminus\{j\}}|\wh\Sigma_{jk}|\bigr| > 2\delta$}   
				\State $Pure(i) \gets False$,
				\State \textbf{break}
				\EndIf	
				\EndFor
				\If {$Pure(i)$}
				\State $\wh I^{(i)} \gets \wh I^{(i)}\cup \{i\}$
				\State $\wh\I \gets$ \textsc{Merge($\wh I^{(i)},\ \wh \I$)}
				\EndIf
				\EndFor
				\State \Return $\wh \I$ and $\wh K$ as the number of sets in $\wh \I$
				\EndProcedure
				\Statex
				\Function {Merge}{$\wh I^{(i)}$, $\wh \I$}\label{alg2}
				\ForAll {$G \in \wh \I$}
				\Comment $\wh \I$ is a collection of sets
				\If {$G \cap \wh I^{(i)}\ne \emptyset$} 
				\State  $G\gets G\cap \wh I^{(i)}$
				\Comment Replace $G\in \wh \I$ by $G\cap \wh I^{(i)}$
				\State \Return $\wh \I$
				\EndIf
				\EndFor
				\State $\wh I^{(i)} \in \wh \I$
				\Comment add $\wh I^{(i)}$  in $\wh \I$
				\State \Return $\wh \I$
				\EndFunction
			\end{algorithmic}
	\end{algorithm}}
	Given estimates $\wh K$ and $\wh \I$ as outputs of Algorithm 1, we compute, for each $a\in[\wh K]$ and $b\in [\wh K]\setminus\{a\}$,
	\begin{equation}\label{Chat}
	\left[\whC\right]_{aa} = 
	\frac{1}{|\wh I_a|(|\wh I_a|-1)}\sum_{i, j\in \wh I_a, i\ne j}\rs|\wh \Sigma_{ij}|,
	\quad \left[\whC\right]_{ab} = 
	\frac{1}{|\wh I_a||\wh I_b|}\sum_{i\in \wh I_a, j\in \wh I_b}\rs \wh A_{ia}\wh A_{ib}\wh \Sigma_{ij},
	\end{equation}
	to form the estimator  $\whC$ of $\C$.\\ 
	
	The submatrix $\wh A_{\wh I}$ is then constructed as follows. For each $k\in [\wh K]$ and the estimated pure variable set $\wh I_k$,  
	\begin{align}\label{est_AI_a}
	&\textrm{Pick an element $i\in \wh I_k$ at random, and set $\wh A_{i\cdot }=e_k$;}\\\label{est_AI_b}
	&\textrm{For the remaining $j\in \wh I_k\setminus\{i\}$, set $\wh A_{j\cdot } = \textrm{sign}(\wh \Sigma_{ij})\cdot e_k$.}
	\end{align}
	Letting $\wh J = [p]\setminus \wh I$, to construct the remaining submatrix $\wh A_{\wh J}$, we use the Dantzig-type estimator $ \wh A_D$ proposed in \cite{LOVE} given by
	\begin{equation}\label{est_AJ}
	\wh A_{j\cdot} = \arg\min_{\beta^j}\left\{\|\beta^j\|_1:\  \left\|\whC \beta^j - (\wh A_{\wh I}^{\top}     \wh A_{\wh I})^{-1}\wh A_{\wh I}^{\top}     \wh\Sigma_{\wh Ij}\right\|_\i \le \mu\right\}
	\end{equation}
	for any $j\in \wh J$, with tuning parameter $\mu = O(\sqrt{\log\pn / n})$.
	The estimator $\wh A$ enjoys the optimal convergence rate of $\max_{j\in [p]}\|\wh A_{j\cdot} -A_{j\cdot}\|_q$ for any $1\le q\le \i$ \citep[Theorem 5]{LOVE}.\\

	\section{More existing literature on factor models}\label{app_literature}
    We discuss in this section some related work on factor models which might be used to establish results of the excess risk of PCR. 
    
 		By treating $X$ and $Y$ jointly from model \ref{main_model} as an augmented factor model
 		 \[
 		    \wt X := \begin{bmatrix}
 		        Y \\ X
 		    \end{bmatrix} = 
 		    \begin{bmatrix}
 		         \beta^\T \\A 
 		    \end{bmatrix} Z + \begin{bmatrix}  \eps \\ W \end{bmatrix},
 		 \]
 		 the fit $\wh \Y$ is constructed by regressing $\Y$ onto $\wt \X \wt \U_K$ where $\wt \U_K$ is the matrix of the first $K$ right singular vectors of $\wt \X = (\wt \X_{1\sbt}^\T, \ldots, \wt \X_{n\sbt}^\T)^\T$. \cite{Bai-factor-model-03} shows that
 		\begin{equation}\label{def_Vt}
 		V_t^{-1/2}\left(
 		\wh \Y_t - \Z_{t\sbt}^\T \beta
 		\right) \to N(0, 1),\qquad \textrm{for any $1\le t\le n$}
 		\end{equation}
 		for a variance term $V_t$. The uniform convergence rate of $\wh \Y_t - \Z_{t\sbt}^\T\beta$ over  $1\le t\le n$  is further derived in
 		\cite{fan2013large}. These element-wise results for  {\it in-sample} prediction could, in principle,  be extended to out-of-sample prediction, via additional arguments, but is not treated in the aforementioned works. 
 		
 		We now comment on the main differences between our Corollary \ref{cor_PCR_delta_w} and the aforementioned results. The existing results are all established under conditions including $K =O(1)$, $\|\beta\|_2^2=O(1)$,  $p\to\i$, and (\ref{cond_PCR}), 
 		The uniform consistency in \cite{fan2013large} additionally requires  $n = o(p^2)$. As a result, all previous results are asymptotic statements as $n,p\to \i$.
		
 		By contrast, our  Corollaries \ref{cor_PCR_k}, \ref{cor_PCR_delta_w} and \ref{cor_PCR_s_tilde} are non-asymptotic statements which hold for any finite $K$, $n$ and $p$. Moreover, they only requires the sub-Gaussian tail assumptions in Definition \ref{frm} and $K\log n \lesssim n$. 
 		As detailed in Section \ref{sec_existing}, our conditions on the signal $\lambda_K(A \C A^\T)$ are much weaker than (\ref{cond_PCR}) to derive the risk of PCR-$K$. 
 		
 		 Under condition (\ref{cond_PCR}), as assumed in the aforementioned literature,
 		the prediction risk in our Corollary \ref{cor_PCR_delta_w} reduces to 
 		\begin{align*}
 		\R(\U_K) -\sigma^2&= O_p\left({\sigma^2\over n}+ {\|\Gamma\|_{{\rm op}}\over p} + { \|\Gamma\|_{{\rm op}} \over n} \right).
 		\end{align*}
 		This rate coincides with that of $V_t$, introduced in (\ref{def_Vt}). Under conditions in \cite{fan2013large}, their results (see, for instance, Corollary 3.1) imply 
 		\[
 		\max_{1\le t\le n}\left|\wh \Y_t- \Z_{t\sbt}^\T\beta  \right |^2  = O_p\left(
 		\left(\log n\right)^{2/r_2}{\log p \over n} + {n^{1/2} \over p}
 		\right)
 		\]
 		for some constant $r_2>0$, which is slower than our rate.

\end{document}